\documentclass[11pt]{article}
\pdfoutput=1
\usepackage{etoolbox,verbatim}
\newtoggle{arxiv}
\toggletrue{arxiv}
\newtoggle{colt}
\togglefalse{colt}
\newtoggle{iclr}
\togglefalse{iclr}
\newtoggle{customthms}
\toggletrue{customthms}

\iftoggle{iclr}{
\usepackage{iclr2026_conference,times}
}{}

\usepackage{xcolor}
\usepackage{tcolorbox}
\tcbuselibrary{skins,breakable}

\tcbset{
  aibox/.style={
      width=\linewidth,
      top=6pt,
      bottom=0pt,
      left=1.5pt,
      right=1.5pt,
colback=blue!60!gray!5,
colframe=black,
colbacktitle=black,
      enhanced,
      center,
      attach boxed title to top left={yshift=-0.1in,xshift=0.15in},
      boxed title style={boxrule=0pt,colframe=white,},
    }
}
\newtcolorbox{AIbox}[2][]{aibox,title=#2,#1}

\definecolor{primalcolor}{HTML}{A60000}
\definecolor{contrarycolor}{HTML}{00A6A6}
\definecolor{darkcontrarycolor}{HTML}{004C4C}
\definecolor{lightblue}{HTML}{2970CC}
\definecolor{lightpurple}{HTML}{673147}
\definecolor{ForestGreen}{HTML}{FF5733}
\definecolor{myred}{HTML}{AA4A44}
\definecolor{hyppurple}{HTML}{800080}

\newcommand{\linkcolor}{darkcontrarycolor}
\newcommand{\urlcolor}{darkcontrarycolor}
\newcommand{\citecolor}{darkcontrarycolor}

\newcommand{\thmcolordark}{red!30!black}

\newcommand{\takeawaycolor}{red!40!black}
\newcommand{\takeawaybold}[1]{{\color{\takeawaycolor} {\textbf{#1}}}}

\usepackage[T1]{fontenc}
\usepackage{capt-of}
\usepackage[font=small,labelfont=bf]{caption}

\usepackage{natbib}
\usepackage{amssymb,upgreek,bm}
\usepackage{mathtools}
\usepackage[english]{babel}  
\usepackage{multirow,tcolorbox,tikz-cd,turnstile,xspace,xparse}
\usepackage{relsize,breakcites,appendix}
\usepackage{longtable,tablefootnote,booktabs,float,makecell}
\usepackage[normalem]{ulem}
\usepackage{tabu}
\usepackage[shortlabels]{enumitem} 
\usepackage{footnote}
\usepackage{boxedminipage}
\usepackage{multirow,nicefrac}
\usepackage{verbatim} %
\usepackage{wrapfig}

\usepackage[colorlinks=true,linkcolor=\linkcolor,urlcolor=\urlcolor,citecolor=\citecolor,breaklinks]{hyperref}

\usepackage{prettyref}

\usepackage[capitalise,nameinlink]{cleveref}

\iftoggle{colt}{
    \DeclareRobustCommand{\qed}{
        \usepackage{thmtools}
          \ifmmode \mathqed
          \else
            \leavevmode\unskip\penalty9999 \hbox{}\nobreak\hfill
            \quad\hbox{\qedsymbol}%
          \fi
    }
}
{
    \usepackage{amsthm}
     
}
\usepackage{thm-restate}

\iftoggle{arxiv}
{
    \usepackage{mathrsfs}
    \usepackage[margin=2cm]{geometry}
    \usepackage[bitstream-charter]{mathdesign}
     
    \usepackage[scaled=0.92]{PTSans}

    \usepackage{subfig}

    \usepackage{fullpage}
    \usepackage[noend]{algpseudocode}
    \usepackage{algorithm}
   
}
{
    \usepackage{mathrsfs}
}

\DeclareMathAlphabet{\mathbfsf}{\encodingdefault}{\sfdefault}{bx}{n}

\numberwithin{equation}{section}

\iftoggle{arxiv}
{
    }
{
    
}
\iftoggle{arxiv}
{
    
}
{
    
}

\usepackage{thmtools}

\Crefname{equation}{Eq.}{Eqs.}
\Crefname{assumption}{Assumption}{Assumptions}
\Crefname{condition}{Condition}{Conditions}
\Crefname{claim}{Claim}{Claims}
\Crefname{property}{Property}{Properties}
\Crefname{construction}{Construction}{Constructions}

\declaretheoremstyle[
    headformat=\normalfont\textcolor{\thmcolordark}{\bfseries\NAME\,\NUMBER}\NOTE,%
    notefont={\normalfont\textcolor{\thmcolordark}{\bfseries}}, 
    notebraces={}{},
    bodyfont=\normalfont\itshape,
    spaceabove = 6pt,
    spacebelow = 6pt,
    ]{coloredthmversion}

\declaretheoremstyle[
    headformat=\normalfont\textcolor{\thmcolordark}{\bfseries\NAME\,\NUMBER}\NOTE,%
    bodyfont=\normalfont\itshape,
    spaceabove = 6pt,
    spacebelow = 6pt,
    ]{coloredthm}

\declaretheoremstyle[
    headformat=\normalfont\textcolor{\thmcolordark}{\bfseries\NAME\,\NUMBER}\NOTE,%
    bodyfont=\normalfont,
    spaceabove = 6pt,
    spacebelow = 6pt,
    ]{coloreddef}

\iftoggle{customthms}
{
    \theoremstyle{coloredthmversion}
}
{}

\iftoggle{customthms}{
  \theoremstyle{coloredthm}
  \newtheorem{theorem}{Theorem}
  \newtheorem{lemma}{Lemma}[section]
  \newtheorem{corollary}{Corollary}[section]
  \newtheorem{proposition}[lemma]{Proposition}
}
{}

\newtheorem*{thminformal*}{Informal Theorem}

\iftoggle{customthms}{
    \theoremstyle{coloreddef}
    \newtheorem{definition}{Definition}[section]
    \newtheorem{component}{Component}
    
    \newtheorem{remark}{Remark}[section]
    \newtheorem{property}{Property}[section]
    \newtheorem{contribution}{Contribution}
}
{
  
}

\newtheorem{assumption}{Assumption}[section]
\newtheorem{condition}{Condition}[section]

\makeatletter
\newcommand{\neutralize}[1]{\expandafter\let\csname c@#1\endcsname\count@}
\makeatother

\iftoggle{customthms}{
    \newtheoremstyle{named}{}{}{\itshape}{}{\bfseries}{}{.5em}{\Cref{#3} {\normalfont (informal)} }{}
    \theoremstyle{named}
    
    \theoremstyle{plain}
}
{}

\newtheorem*{theorem*}{Theorem}
\newtheorem*{lemma*}{Lemma}
\newtheorem*{corollary*}{Corollary}
\newtheorem*{proposition*}{Proposition}
\newtheorem*{claim*}{Claim}
\newtheorem*{fact*}{Fact}
\newtheorem*{observation*}{Observation}
\newtheorem*{definition*}{Definition}
\newtheorem*{remark*}{Remark}
\newtheorem*{example*}{Example}

\def\ddefloop#1{\ifx\ddefloop#1\else\ddef{#1}\expandafter\ddefloop\fi}
\def\ddef#1{\expandafter\def\csname bb#1\endcsname{\ensuremath{\mathbb{#1}}}}
\ddefloop ABCDEFGHIJKLMNOPQRSTUVWXYZ\ddefloop

\def\ddefloop#1{\ifx\ddefloop#1\else\ddef{#1}\expandafter\ddefloop\fi}
\def\ddef#1{\expandafter\def\csname frak#1\endcsname{\ensuremath{\mathfrak{#1}}}}
\ddefloop ABCDEFGHIJKLMNOPQRSTUVWXYZ\ddefloop

\def\ddefloop#1{\ifx\ddefloop#1\else\ddef{#1}\expandafter\ddefloop\fi}
\def\ddef#1{\expandafter\def\csname fr#1\endcsname{\ensuremath{\mathfrak{#1}}}}
\ddefloop ABCDEFGHIJKLMNOPQRSTUVWXYZ\ddefloop

\def\ddefloop#1{\ifx\ddefloop#1\else\ddef{#1}\expandafter\ddefloop\fi}
\def\ddef#1{\expandafter\def\csname eul#1\endcsname{\ensuremath{\EuScript{#1}}}}
\ddefloop ABCDEFGHIJKLMNOPQRSTUVWXYZ\ddefloop

\def\ddefloop#1{\ifx\ddefloop#1\else\ddef{#1}\expandafter\ddefloop\fi}
\def\ddef#1{\expandafter\def\csname scr#1\endcsname{\ensuremath{\mathscr{#1}}}}
\ddefloop ABCDEFGHIJKLMNOPQRSTUVWXYZ\ddefloop

\def\ddefloop#1{\ifx\ddefloop#1\else\ddef{#1}\expandafter\ddefloop\fi}
\def\ddef#1{\expandafter\def\csname b#1\endcsname{\ensuremath{\mathbf{#1}}}}
\ddefloop ABCDEFGHIJKLMNOPQRSTUVWXYZ\ddefloop

\def\ddefloop#1{\ifx\ddefloop#1\else\ddef{#1}\expandafter\ddefloop\fi}
\def\ddef#1{\expandafter\def\csname bhat#1\endcsname{\ensuremath{\hat{\mathbf{#1}}}}}
\ddefloop ABCDEFGHIJKLMNOPQRSTUVWXYZ\ddefloop

\def\ddefloop#1{\ifx\ddefloop#1\else\ddef{#1}\expandafter\ddefloop\fi}
\def\ddef#1{\expandafter\def\csname btil#1\endcsname{\ensuremath{\tilde{\mathbf{#1}}}}}
\ddefloop ABCDEFGHIJKLMNOPQRSTUVWXYZ\ddefloop

\def\ddefloop#1{\ifx\ddefloop#1\else\ddef{#1}\expandafter\ddefloop\fi}
\def\ddef#1{\expandafter\def\csname bst#1\endcsname{\ensuremath{\mathbf{#1}^\star}}}
\ddefloop ABCDEFGHIJKLMNOPQRSTUVWXYZ\ddefloop

\def\ddefloop#1{\ifx\ddefloop#1\else\ddef{#1}\expandafter\ddefloop\fi}
\def\ddef#1{\expandafter\def\csname bst#1\endcsname{\ensuremath{\mathbf{#1}^\star}}}
\ddefloop abcdeghijklmnopqrstuvwxyz\ddefloop

\def\ddefloop#1{\ifx\ddefloop#1\else\ddef{#1}\expandafter\ddefloop\fi}
\def\ddef#1{\expandafter\def\csname bhat#1\endcsname{\ensuremath{\hat{\mathbf{#1}}}}}
\ddefloop abcdefghijklmnopqrstuvwxyz\ddefloop

\def\ddefloop#1{\ifx\ddefloop#1\else\ddef{#1}\expandafter\ddefloop\fi}
\def\ddef#1{\expandafter\def\csname b#1\endcsname{\ensuremath{\mathbf{#1}}}}
\ddefloop abcdeghijklnopqrstuvwxyz\ddefloop

\def\ddefloop#1{\ifx\ddefloop#1\else\ddef{#1}\expandafter\ddefloop\fi}
\def\ddef#1{\expandafter\def\csname barb#1\endcsname{\ensuremath{\bar{\mathbf{#1}}}}}
\ddefloop abcdefghijklmnopqrstuvwxyz\ddefloop

\def\ddef#1{\expandafter\def\csname c#1\endcsname{\ensuremath{\mathcal{#1}}}}
\ddefloop ABCDEFGHIJKLMNOPQRSTUVWXYZ\ddefloop
\def\ddef#1{\expandafter\def\csname h#1\endcsname{\ensuremath{\widehat{#1}}}}
\ddefloop ABCDEFGHIJKLMNOPQRSTUVWXYZ\ddefloop
\def\ddef#1{\expandafter\def\csname hc#1\endcsname{\ensuremath{\widehat{\mathcal{#1}}}}}
\ddefloop ABCDEFGHIJKLMNOPQRSTUVWXYZ\ddefloop
\def\ddef#1{\expandafter\def\csname t#1\endcsname{\ensuremath{\widetilde{#1}}}}
\ddefloop ABCDEFGHIJKLMNOPQRSTUVWXYZ\ddefloop
\def\ddef#1{\expandafter\def\csname tc#1\endcsname{\ensuremath{\widetilde{\mathcal{#1}}}}}
\ddefloop ABCDEFGHIJKLMNOPQRSTUVWXYZ\ddefloop

\DeclareMathOperator*{\argmin}{arg\,min}

\let\Pr\relax
\DeclareMathOperator{\Pr}{\mathbb{P}}

\newcommand{\E}{\mathbb{E}}
\newcommand{\Exp}{\mathbb{E}}

\newcommand{\Normal}{\mathrm{N}}

\newcommand{\eye}{\mathbf{I}}

\newcommand{\ballkr}[1][r]{\cB_{k}(r)}
\newcommand{\op}{\mathrm{op}}

\newcommand{\rmd}{\mathrm{d}}

\DeclareMathSymbol{\shortminus}{\mathbin}{AMSa}{"39}

\newcommand{\R}{\mathbb{R}}

\Crefname{component}{Component}{Components}
\Crefname{contribution}{Contribution}{Contributions}
\newcommand{\componentref}[1]{%
  \hyperref[#1]{C\ref*{#1}}%
}
\Crefname{claim}{Claim}{Claims}
\Crefname{property}{Property}{Properties}

\newcommand{\hyp}[1]{{\color{hyppurple}{H#1}}}
\newcommand{\hypanchor}[1]{\hypertarget{hyp:#1}{\hyp{#1}}} %
\newcommand{\hypref}[1]{\hyperlink{hyp:#1}{\hyp{#1}}}     %

\newcommand{\iclrpar}[1]{\textbf{\takeawaybold{#1}}}

\iftoggle{iclr}
{
  \newcommand{\iclrstyle}{\textstyle}
}
{
  \newcommand{\iclrstyle}{}
}

\newcommand{\thmbolddark}[1]{{\color{\thmcolordark}\textbf{#1}}}

\newcommand{\stopgrad}{\mathrm{sg}}

\newcommand{\flowpol}{\pi}
\newcommand{\piideal}{\pi^{\star}}
\newcommand{\flowideal}{b^{\star}}

\newcommand{\ddt}{\frac{\rmd }{\rmd t}}
\renewcommand{\Normal}{\mathrm{N}}

\makeatletter
\newcommand\addtometadatalist[5][]{%
  \begingroup
  \if\relax#3\relax\def\sep{}\else\def\sep{#5}\fi
  \let\protect\@unexpandable@protect
  \xdef#3{\expandafter{#3}\sep #4[#1]{#2}}%
  \endgroup
}

\newcommand\metadatalist{}
\newcommand\metadataformat[2][]{{\small \textbf{#1:} #2}}
\newcommand\metadata[2][]{\addtometadatalist[#1]{#2}{\metadatalist}{\metadataformat}{\\}}
\makeatother

\newcommand{\paperwebsite}[1]{\metadata[Website]{\url{#1}}}
\newcommand{\papercode}[1]{\metadata[Code]{\url{#1}}}
\newcommand{\paperdocs}[1]{\metadata[Documentation]{\url{#1}}}
\newcommand{\paperblog}[1]{\metadata[Blog]{\url{#1}}}

\newcommand{\taskname}[1]{{\texttt{#1}}}
\newcommand{\algname}[1]{\textbf{\texttt{#1}}}
\newcommand{\archname}[1]{\texttt{#1}}

\newcommand{\lifttask}{\taskname{Lift}\xspace}
\newcommand{\can}{\taskname{Can}\xspace}
\newcommand{\squaretask}{\taskname{Square}\xspace}
\newcommand{\transport}{\taskname{Transport}\xspace}
\newcommand{\toolhang}{\taskname{Tool}-\taskname{Hang}\xspace}
\newcommand{\pusht}{\taskname{Push}-\taskname{T}\xspace}
\newcommand{\kitchen}{\taskname{Kitchen}\xspace}
\newcommand{\metaworld}{\taskname{MetaWorld}\xspace}
\newcommand{\adroit}{\taskname{Adroit}\xspace}

\newcommand{\libero}{\taskname{LIBERO}\xspace}
\newcommand{\liberoobject}{\taskname{LIBERO Object}\xspace}
\newcommand{\liberogoal}{\taskname{LIBERO Goal}\xspace}
\newcommand{\liberospatial}{\taskname{LIBERO Spatial}\xspace}
\newcommand{\liberoten}{\taskname{LIBERO 10}\xspace}

\newcommand{\Flow}{\algname{Flow}\xspace}
\newcommand{\regression}{RCP\xspace}
\newcommand{\Regression}{\algname{Regression}\xspace}

\newcommand{\straightflow}{straight-flow\xspace}
\newcommand{\Straightflow}{\algname{SF}\xspace}

\newcommand{\residualregression}{residual regression\xspace}
\newcommand{\Residualregression}{\algname{RR}\xspace}
\newcommand{\minimaliterativepolicy}{minimal iterative policy\xspace}
\newcommand{\mminimaliterativepolicy}{Minimal iterative policy\xspace}
\newcommand{\Minimaliterativepolicy}{\algname{MIP}\xspace}
\newcommand{\pizero}{\archname{$\pi_0$}\xspace}

\newcommand{\chiunet}{\archname{Chi-UNet}\xspace}
\newcommand{\sudeepdit}{\archname{Sudeep-DiT}\xspace}
\newcommand{\chitransformer}{\archname{Chi-Transformer}\xspace}
\newcommand{\rnn}{\archname{RNN}\xspace}
\newcommand{\mlp}{\archname{MLP}\xspace}

\newcommand{\KitchenEnv}{\textsc{Franka}-\textsc{Kitchen}}

\usepackage{preamble_new/color-edits}

\addauthor{ms}{magenta}

\newcommand{\ignore}[1]{}

\iftoggle{arxiv}
{
\input{preamble_new/arxiv_title}
}
{}

\usepackage{preamble_new/color-edits}
\usepackage{adjustbox}
\usetikzlibrary{calc}
\addauthor{ms}{magenta}
\addauthor{cp}{blue}
\addauthor{nch}{cyan}
\addauthor{ga}{pink}
\addauthor{gs}{green}

\makeatletter
\renewcommand{\maketitle}{
    \newpage
    \null
    \begingroup
    \raggedright
    {\LARGE \bfseries \@title \par}
    \vskip 1.5em
    {\large
    \lineskip .5em
    \begin{tabular}[t]{l}
    \@author
    \end{tabular}\par}
    \vskip 1em
    {\large \@date \par}
    \endgroup
    \par
    \vskip 1.5em
}
\makeatother

\title{Much Ado About Noising: Dispelling the Myths of Generative Robotic Control}

\author{\small
Chaoyi Pan\footnote{\texttt{\{chaoyip, giria, naichieh, clairej, gqu, nboffi, gshi, msimchow\}@andrew.cmu.edu} \label{foot:cmu}}\textsuperscript{,$\$$} ~
 Giri Anantharaman\textsuperscript{\ref{foot:cmu}} ~
Nai-Chieh Huang\textsuperscript{\ref{foot:cmu}} 
\small
Claire Jin\textsuperscript{\ref{foot:cmu}} ~
Daniel Pfrommer\footnote{\texttt{dpfrom@mit.edu} 
\label{foot:mit}}\\
\small
Chenyang Yuan\footnote{\texttt{\{chenyang.yuan,frank.permenter\}\@tri.global}\label{foot:tri}} ~
Frank Permenter\textsuperscript{\ref{foot:tri}} ~
Guannan Qu\textsuperscript{\ref{foot:cmu},$\dagger$} ~
Nicholas Boffi\textsuperscript{\ref{foot:cmu},$\dagger$} ~
Guanya Shi\textsuperscript{\ref{foot:cmu},$\dagger$} ~ \\
\small 
Max Simchowitz\textsuperscript{\ref{foot:cmu},$\dagger$} \\
\vspace{-.3em}
 \rule{.38\textwidth}{.7pt}
\\
\footnotesize
$^{\$}$Project lead. $^\dagger$Equal advising.  \\
\footnotesize $^{a}$Carnegie Mellon University 
$^{b}$Massachusetts Institute of Technology ~~
$^{c}$Toyota Research Institute ~~
}

\date{\vspace{-0.5cm}}

\paperwebsite{https://simchowitzlabpublic.github.io/much-ado-about-noising-project/}
\papercode{https://github.com/simchowitzlabpublic/much-ado-about-noising}
\paperdocs{https://simchowitzlabpublic.github.io/much-ado-about-noising/}

\begin{document}
\begin{tcolorbox}[
    colback=blue!60!gray!5, colframe=gray!50,
    boxrule=0pt,
    arc=2mm%
  ]
  \maketitle
  \vspace{-1em}
  \tcbline
  
\begin{minipage}[t]{1.0\linewidth}
    \centering
    \includegraphics[width=1.0\textwidth]{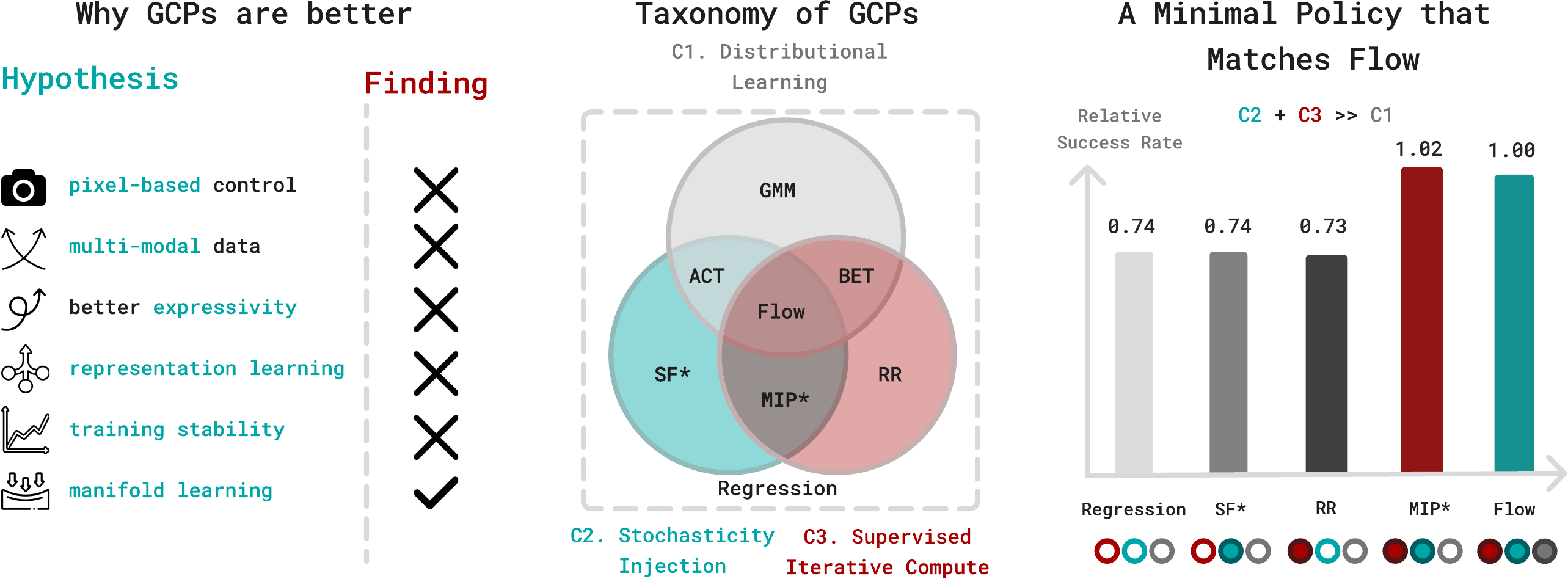}
    \captionof{figure}{  %
        \footnotesize
        \emph{Left:} After careful ablation on each component over 28 common behavior cloning benchmarks with diverse input modalities (state, pixel, point cloud and language), architectures (with both raw and pre-trained models, like \pizero) and tasks (standard single-task benchmarks and multi-task benchmark like \libero), we refute a number of popularly held misconceptions about why \textbf{generative control policies} (GCPs) outperform regression policies (RCP) on these tasks.
        \emph{Center:} We identify that  the most important factor contributing to GCP success is a combination of \emph{stochastic injection} (\componentref{comp:stoch}) and \emph{supervised iterative computation} (\componentref{comp:sic}).
        Surprisingly, distribution learning (\componentref{comp:distr}) is the least important factor, due to the absence of learned  multi-modality (\Cref{sec:multimodality}).
        \emph{Right:} The average relative success rate to flow of 7 most challenging tasks. We propose a simple two-step \textbf{minimal iterative policy} (\Minimaliterativepolicy) whose performance matches that of flow-based GCPs.}
    \label{fig:teaser}
\end{minipage}

  \tcbline
  \vspace{-1em}
  \vskip 0.5cm
  \makeatletter
  \ifdefempty{\metadatalist}{}{\metadatalist\par}
  \makeatother
\end{tcolorbox}

\begin{abstract}
  
Generative models, like flows and diffusions, have recently emerged as popular and efficacious policy parameterizations in robotics. There has been much speculation as to the factors underlying their successes, ranging from capturing multi-modal action distribution to expressing more complex behaviors.
In this work, we perform a comprehensive evaluation of popular generative control policies (GCPs) on common behavior cloning (BC) benchmarks.
We find that GCPs \emph{do not} owe their success to their ability to capture multi-modality or to express more complex observation-to-action mappings.
Instead, we find that their advantage stems from \emph{iterative computation}, as long as intermediate steps are supervised during training and this supervision is paired with a suitable level of \emph{stochasticity}.
As a validation of our findings, we show that a \minimaliterativepolicy~(\Minimaliterativepolicy), a lightweight two-step regression-based policy, essentially matches the performance of flow GCPs, {and often outperforms distilled shortcut models}.
Our results suggest that the distribution-fitting component of GCPs is less salient than commonly believed, and point toward new design spaces focusing solely on control performance.
\iftoggle{iclr}{Code and supplementary materials are available at \href{https://simchowitzlabpublic.github.io/much-ado-about-noising-project/}{our website}.}{}

\end{abstract}

\section{Introduction}

\iftoggle{arxiv}{}
{
    \begin{figure}[h]
        \vspace{-1.5em}
        \centering
        \includegraphics[width=0.95\textwidth]{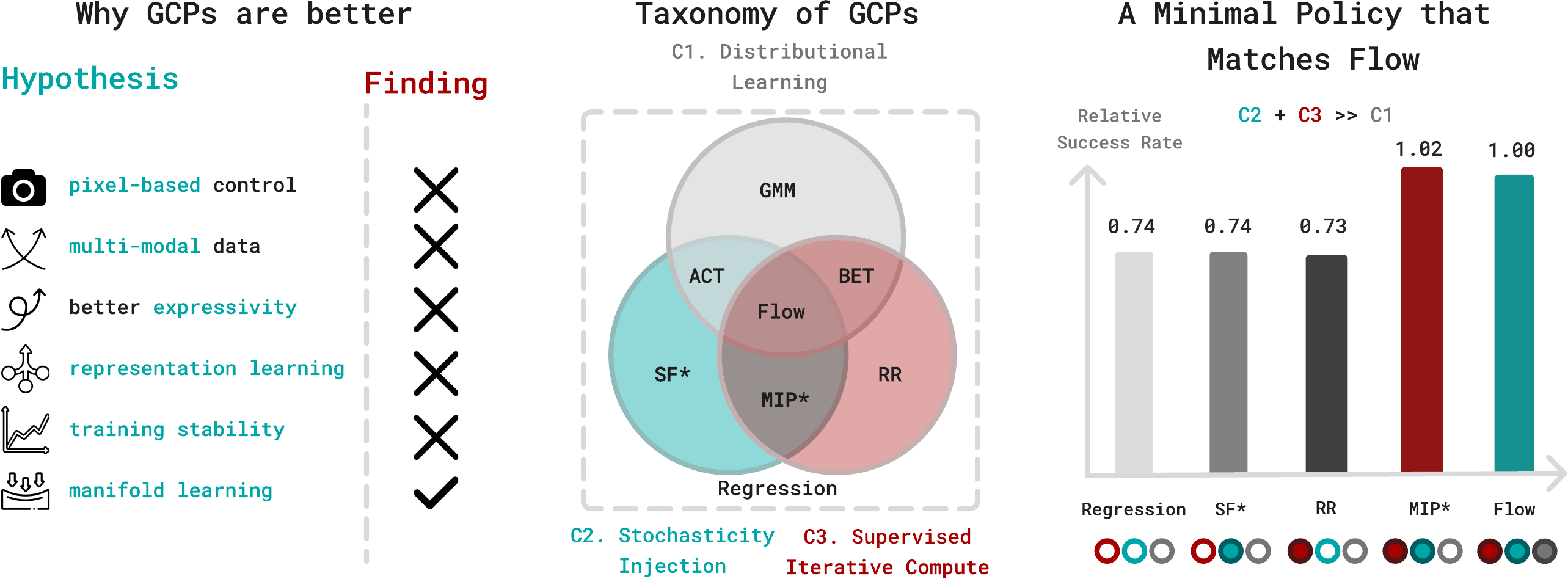}
        \vspace{-1.0em}
        \captionof{figure}{  %
            \footnotesize
            \emph{Left:} After careful ablation on each component over 28 common behavior cloning benchmarks with diverse input modalities (state, pixel, point cloud and language), architectures (with both raw and pre-trained models, like \pizero) and tasks (standard single-task benchmarks and multi-task benchmark like \libero), we refute a number of popularly held misconceptions about why \textbf{generative control policies} (GCPs) outperform regression policies (RCP) on these tasks.
            \emph{Center:} We identify that the most important factor contributing to GCP success is a combination of \emph{stochastic injection} (\componentref{comp:stoch}) and \emph{supervised iterative computation} (\componentref{comp:sic}).
            Surprisingly, distribution learning (\componentref{comp:distr}) is the least important factor, due to the absence of learned  multi-modality (\Cref{sec:multimodality}).
            \emph{Right:} The average relative success rate to flow of 7 most challenging tasks. We propose a simple two-step \textbf{minimal iterative policy} (\Minimaliterativepolicy) whose performance matches that of flow-based GCPs.}
        \label{fig:teaser}
        \vspace{-1.0em}
    \end{figure}
}

Long-horizon, dexterous manipulation tasks such as furniture assembly, food preparation, and manufacturing have been a holy grail in robotics.
Recent large robot action models~\citep{teamCarefulExaminationLarge2025,black2024pi_0,kim2024openvla} have made substantial breakthroughs towards these goals by imitating expert demonstrations of diverse qualities.
We provide a more comprehensive review of related work in~\cref{sec:related}, but highlight here a key trend:
while supervised learning from demonstration, also known as \emph{behavior cloning} (BC), has been applied across domains for decades \citep{pomerleau1988alvinn}, its recent success in robotic manipulation has coincided with the adoption of what we term \takeawaybold{generative control policies} (GCPs):  robotic control policies that use generative modeling architectures, such as diffusion models, flow models, and autoregressive transformers, as parameterizations of the mapping from observation to action.
Given the seemingly transformative nature of GCPs for robot learning, there has been much speculation about the origin of their superior performance relative to policies trained with a regression loss, henceforth \takeawaybold{regression control policies} (RCPs).
GCPs, by modeling conditional distributions over actions, are uniquely suited to the multi-task pretraining paradigm popular in today's large robotic models.
However, a number of hypotheses regarding the superiority of GCPs pertain even in the \emph{single task} setting \citep{chi2023diffusion,reuss2023goal}:
\iftoggle{arxiv}{
    \begin{itemize}[leftmargin=2em,itemsep=-0em,topsep=0em]
        \item[\hypanchor{1}.] Better performance on pixel-based control
        \item[\hypanchor{2}.] Capturing multi-modality in the training data
        \item[\hypanchor{3}.] Greater expressivity due to iterative computation of the observation-to-action mapping
        \item[\hypanchor{4}.] Representation learning due to stochastic data augmentation
        \item[\hypanchor{5}.] Improved training stability and scalability
    \end{itemize}}{
    \begin{itemize}[leftmargin=2.5em, itemsep=-0.0em, parsep=0pt, topsep=0pt]
        \vspace{-0.5em}
        \item[\hypanchor{1}.] Better performance on pixel-based control
        \item[\hypanchor{2}.] Capturing multi-modality in the training data
        \item[\hypanchor{3}.] Greater expressivity due to iterative computation of the observation-to-action mapping
        \item[\hypanchor{4}.] Representation learning due to stochastic data augmentation
        \item[\hypanchor{5}.] Improved training stability and scalability
              \vspace{-0.5em}
    \end{itemize}
}
\iftoggle{iclr}{}{In this work, we systematically investigate these hypotheses to understand the mechanism by which GCPs have attained superior performance over RCPs. We aim to answer:
    \\
    \rule{\textwidth}{1pt}
    \begin{quote}
        Is there \emph{really} a benefit to using GCPs for behavior cloning, or are  their  claimed successes ... \takeawaybold{much ado about noising?}
    \end{quote}
    \vspace{-.5em}
    \rule{\textwidth}{1pt}
}

\iclrpar{The gap between generative modeling and generative control.} The objective for generative modeling in text and image domains is fundamentally different from the goal in a control task. In the former, one aims to generate high-quality and \emph{diverse} samples from the original data distribution. In the latter, it suffices to select \emph{any} action that leads to better downstream performance. Whereas much of the generative modeling literature has focused on the distribution of the \emph{generated variable} \citep{lee2023convergence},
we aim to understand if it is necessary to reproduce the expert data distribution---for example by capturing any multi-modality---to attain strong control performance. If not, is most salient to capture about the \emph{conditioning relationship} mapping $o \to a$?

\iftoggle{arxiv}
{
    \subsection{Contributions.}
}
{
    \iclrpar{Contributions.}
}
This paper adopts careful experimental methodology to rigorously test the key design components  (\Cref{sec:parsing}) that contribute to the observed success of GCPs, and to account for the key mechanisms by which they contribute to improved performance in behavior cloning (\Cref{sec:whatworks}).
We restrict our study to flow-based GCPs, given their popularity and adoption in industry \citep{black2024pi_0,intelligence$p_05$VisionLanguageActionModel2025,nvidia2025gr00t}.

We begin by first identifying which factors \emph{do not} contribute to the advantage of GCPs over RCPs.
\begin{contribution}[\thmbolddark{Neither multi-modality nor policy expressivity account for GCPs' success}, \Cref{sec:dispelling}] Through careful benchmarking, we show that RCPs with appropriate architectures are highly competitive on both state- and image-based (\hypref{1}) robot learning benchmarks as well as vision-language-action (VLA) model finetuning (\Cref{sec:myth_performance}). Performance gaps only arise on certain tasks requiring high precision. However, we show that neither multi-modality (\hypref{2}, \Cref{sec:multimodality}) nor the ability to express more complex functions via multiple integration steps (\hypref{3}, \Cref{sec:expressivity}) satisfactorily accounts for this phenomenon. In fact,  GCPs do not even provide greater trajectory diversity compared to RCPs (\Cref{sec:diversity}).
\end{contribution}
Essential to this finding is controlling for architecture: to our knowledge, we are the first work to carefully benchmark expressive architectures popularized for Diffusion~\citep{chi2023diffusion,dasariIngredientsRoboticDiffusion2024} as regression policies. To determine what contributes to GCPs performance on these high-precision tasks (beyond architectural optimization), we parse the design space of generative control policies into three components, depicted in \Cref{fig:teaser} (left).
\begin{contribution}[\thmbolddark{Exposing the design space of GCPs}, \Cref{sec:core_design_components}] We introduce a novel taxonomy that parses the three essential design components of GCPs:
    \iftoggle{iclr}{\vspace{-.3em}}{}
    \iftoggle{arxiv}{
        \begin{itemize}[topsep=0em,itemsep=-.05em,leftmargin=2em]
            \item[\componentref{comp:distr}.]\emph{Distributional Learning}:  matching a conditional distribution of actions given observations.\item[\componentref{comp:stoch}.] \emph{Stochasticity Injection}: injecting noise during training to improve the learning dynamics.
            \item[\componentref{comp:sic}.] \emph{Supervised Iterative Computation}: generating output with multiple steps, each of which receives supervision during training.
        \end{itemize}
    }{
        \begin{itemize}[topsep=0em,itemsep=-.05em,parsep=0.05em,leftmargin=2em]
            \item[\componentref{comp:distr}.]\emph{Distributional Learning}:  matching a conditional distribution of actions given observations.
            \item[\componentref{comp:stoch}.] \emph{Stochasticity Injection}: injecting noise during training to improve the learning dynamics.
            \item[\componentref{comp:sic}.] \emph{Supervised Iterative Computation}: generating output with multiple steps, each of which receives supervision during training.
        \end{itemize}
    }
\end{contribution}
With this taxonomy in hand,~\Cref{sec:algorithms_with_different_design_components} introduces a family of algorithms, each of which lies along a spectrum between GCPs and RCPs by exhibiting different combinations of the above components. While we find that neither \componentref{comp:stoch} nor \componentref{comp:sic}  in isolation  improve over regression, we find their combination yields a policy whose performance is competitive with flow, leading to our next contribution.
\begin{contribution}[\thmbolddark{\Minimaliterativepolicy{}: the power of \componentref{comp:stoch}+\componentref{comp:sic}}, \Cref{sec:performance_comparison,sec:algorithms_with_different_design_components}] As an algorithmic ablation that only combines \componentref{comp:stoch}+\componentref{comp:sic}, we devise a \emph{minimal iterative policy} (\Minimaliterativepolicy{}), which invokes only two iterations, one-step of stochasticity during training, and deterministic inference. Despite its simplicity, \Minimaliterativepolicy{} essentially matches the performance of flow-based GCPs across state-, pixel- and 3D point-cloud-based BC tasks, exposing that the combination of \componentref{comp:stoch}+\componentref{comp:sic} is responsible for the observed success of GCPs. %
        {In addition, we find that \Minimaliterativepolicy{} \textbf{often outperforms shortcut/few-step policies}  (\iftoggle{arxiv}{\Cref{rem:shortcut}}{\Cref{sec:comp_consistency_models}}). This confirms our findings that distributional learning (which few-step policies, but not \Minimaliterativepolicy{}, achieve) is not needed in robotic control.  }
\end{contribution}
As described in \iftoggle{arxiv}{\Cref{rem:shortcut}}{\Cref{sec:comp_consistency_models}}, \Minimaliterativepolicy{} is substantively distinct from flow-map-based models~\citep{boffiFlowMapMatching2025,boffiHowBuildConsistency2025}, including consistency models~\citep{songConsistencyModels2023,kim2023consistency} and their extensions~\citep{gengMeanFlowsOnestep2025,fransOneStepDiffusion2024}, in that the latter do satisfy \componentref {comp:distr}, and require training over a continuum of noise levels.

\begin{contribution}[\takeawaybold{Attributing the benefits of  \componentref{comp:stoch}+\componentref{comp:sic}}, \Cref{sec:whatworks}] We identify that a property we term \emph{manifold adherence} captures the inductive bias of GCPs and \Minimaliterativepolicy relative to RCPs, even in the absence of lower validation loss. We explain how this property is a useful proxy for closed-loop performance in control tasks.
    Finally, we expose how \componentref{comp:sic}, through iterative computation, encourages manifold adherence, but only if stochasticity during training (\componentref{comp:stoch}) is present to mitigate compounding errors across iteration steps (as described in \Cref{sec:stochasticity_injection}).
\end{contribution}
Manifold adherence in~\cref{sec:manifold_adherence} measures the generated action's plausibility given out of distribution observations, where only off-manifold component is evaluated rather than the distance to the neighbors~\citep{pari2021surprising}. Note that manifold adherence reflects a favorable inductive bias during learning, rather than brute expressivity of more complex behavior (\hypref{3}). Moreover, \componentref{comp:stoch}  provides more of a supporting role to \componentref{comp:sic}, rather than enhancing data-augmentation in its own right (\hypref{4}).
In addition, we find that \componentref{comp:stoch}+\componentref{comp:sic} also enhance scaling behavior (\hypref{5}), likely due to better model utilization through decoupling across iterations. Finally, we identify that the subtle interplay between architecture choice, policy parameterization and task can affect performance by an even greater magnitude than the choice of policy parametrization (\Cref{sec:architecture}).

\iclrpar{Takeaway.} In robotic applications, our findings suggest that the distributional formulation of GCPs --- sampling from a \emph{distribution} of actions given observations --- is the least important facet that contributes to their success.
Rather, our work highlights that \componentref{comp:stoch}+\componentref{comp:sic} offer an exciting and under-explored sandbox for future algorithm design in continuous control and beyond.

\section{Preliminaries}
\label{sec:prelim}

\newcommand{\gen}{a}
\newcommand{\cond}{o}
\newcommand{\genspace}{A}
\newcommand{\condspace}{O}
\newcommand{\Dtrain}
{p_{\mathrm{train}}}
\newcommand{\Phieul}{\Phi_{\theta,\mathrm{eul}}}
We consider a continuous control setting with observations $o\in O$ and actions $a\in A$ where $O$ is the observation space and $A$ is the action space.
We learn a policy $\pi: \condspace \to \Delta(\genspace)$ from observations to (distributions over) actions to maximize the probability of success $J(\pi)$ on a given task, which we refer to as ``performance.''
\iftoggle{iclr}{}{This can be formulated as maximizing reward in an Markov Decision Process, which for completeness we formalize in \Cref{app:MDP}.}
We consider the performance of policies learned via BC---that is, supervised learning from a distribution of (observation, actions pairs) drawn from a training distribution $\Dtrain$.
\iftoggle{iclr}{}{We now describe two popular classes of control policies, and their respective training objectives.}
In applications, the actions $a$ are often a short-open loop sequence of actions, or \emph{action-chunks}, which have been shown to work more effectively for complex tasks with end-effector position commands~\citep{zhaoLearningFineGrainedBimanual2023}. See \Cref{sec:related} for an unabridged  related work.

\iclrpar{Regression Control Policies (RCPs).} A historically common policy choice for BC is regression control policies (RCPs) \citep{pomerleau1988alvinn,bain1995framework,ross2011reduction,osa2018algorithmic}, given by a deterministic map $\pi: \condspace \to \genspace$. In applications, it is parameterized by a neural network $\pi_\theta$ and trained so as to minimize the $L_2$-loss on training data:
\begin{align}
    \iclrstyle \pi_{\theta} \approx \argmin_\theta \Exp \|\pi_{\theta}(o) - a\|^2, \quad (o,a) \sim \Dtrain.
\end{align}

\iclrpar{Generative Control Policies (GCPs).}
Generative control policies (GCPs) parameterize a \emph{distribution} of actions $a$ given an observation $o$.
This is often accomplished in practice by representing the policy $\pi_\theta$ with a generative model such as a diffusion~\citep{chi2023diffusion}, flow~\citep{zhang2024flowpolicy}, or tokenized autoregressive transformer~\citep{shafiullah2022behavior}.
Given their popularity, we focus on flow-based GCPs (flow-GCPs).
A flow-GCP learns a conditional flow field \citep{lipman2023flow,chisari2024learning,nguyen2025flowmp,albergo2022building,heitz2023iterative,liu2022flow} $b:[0,1] \times \genspace \times \condspace \to \genspace$ by minimizing the objective
\begin{align}
    b_{\theta} \approx \iftoggle{iclr}{\textstyle}{} \argmin_{\theta} \Exp\|b_t(I_t \mid o) -  \dot I_t\|^2, \quad t \sim \mathrm{Unif}([0,1]), \:\: z\ \sim \Normal(0,\eye),
    \label{eq:flow}
\end{align}
where again $(o,a) \sim \Dtrain$,
$I_t = t a + (1-t) z$ is the stochastic interpolant between the training action $a$ and noise variable $z$, and where $\dot I_t = a - z$ is the time derivative of $I_t$. We note that this is a special case of the stochastic interpolant framework \citep{albergo2022building,albergo2023stochastic,albergo2024stochastic}, which permits a larger menu of design decisions.
A flow model then predicts an action by integrating a flow. In the limit of infinite discretization steps, this amounts to sampling $a \sim \pi_\theta(\cdot \mid o)$ by sampling $z \sim \Normal(0,\eye)$, and then setting $a = a_1$, where $\{a_t\}_{t\in[0,1]}$ solves the ODE:
\begin{align}
    \iftoggle{iclr}{\textstyle}{}   \ddt a_t = b_t(a_t \mid o) \qquad \text{with initial condition} \qquad a_0=z. \label{eq:ode}
\end{align}
In practical implementation, sampling is conducted via discretized Euler integration (see~\cref{app:euler_integration} for details).
This yields a policy $a = \flowpol_\theta(z,o)$ which is a deterministic function of the initial noise $z$ and the observation $o$.
All experiments, unless otherwise stated,  perform $9$ integration steps. We reiterate that other GCPs, e.g. based on diffusion models and autoregressive transformers, have been studied elsewhere. We choose to focus on flow models due to their state-of-the-art performance \citep{chi2023diffusion,chisari2024learning,zhang2024flowpolicy} and deployment in industry \citep{black2024pi_0,intelligence$p_05$VisionLanguageActionModel2025,nvidia2025gr00t}.

\iclrpar{Multi-Modality in Robot Learning.} Past work has conjectured that for salient robotic control tasks, $\Dtrain(a \mid o)$ exhibit \emph{multi-modality}, i.e. the conditional  distribution of $a$ given $o$ has multiple modes~\citep{shafiullah2022behavior,zhaoLearningFineGrainedBimanual2023,florence2022implicit}. This motivated the earliest use of GCPs \citep{chi2023diffusion} (\hypref{2}).  \Cref{sec:multimodality} calls into question the extent to which GCPs do in fact learn multi-modal distributions of $a \mid o$ on popular benchmarks\iftoggle{iclr}{.}{, including those claimed to highlight multimodality as a core challenge.}

\section{Multi-modality and expressivity do not explain GCPs' performance}
\label{sec:dispelling}

This section demonstrates that neither advantages on pixel-based control (\hypref{1}), nor multi-modality (\hypref{2}), nor improved expressivity (\hypref{3}) fully account for the GCPs performance relative to RCPs.
    {
        Instead, our analysis indicates that the advantage of GCPs is \textbf{largely due to architectural innovations} found in GCPs—specifically, the adoption of powerful models like Transformers and UNets, along with the use of action chunking techniques.
    }
\Cref{sec:nearest_neighbor_hypothesis_study,sec:diversity} addresses other hypotheses, such as $k$-nearest neighbor approximation and the behavior diversity. %

\begin{figure}[t!]
    \iftoggle{iclr}{
        \vspace{-0.6cm}
    }{}
    \centering
    \includegraphics[width=1.0\textwidth]{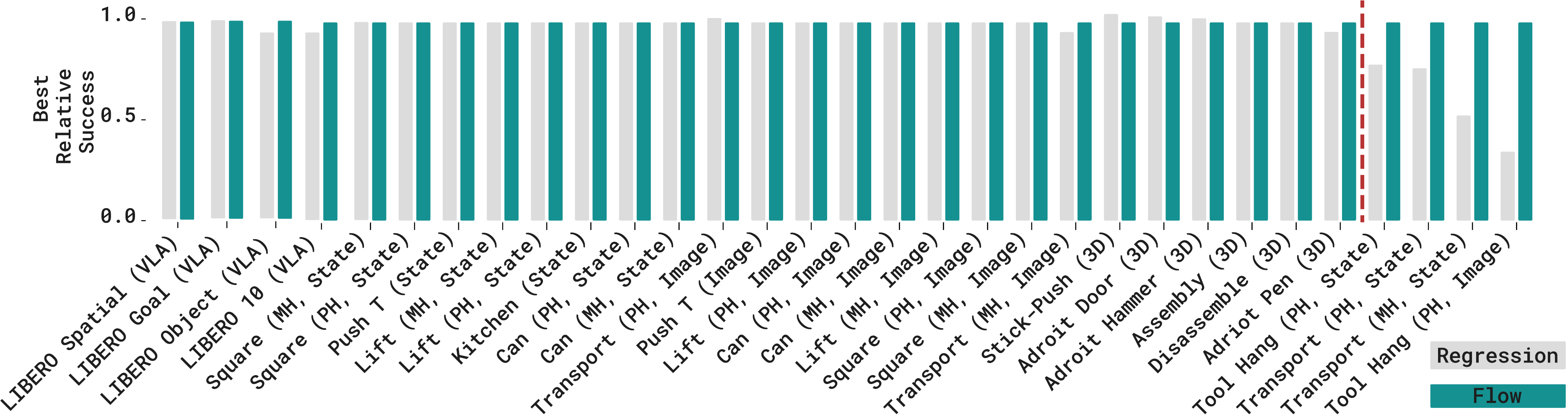}
    \iftoggle{iclr}{
        \vspace{-0.6cm}
    }{}
    \caption{
        \footnotesize
        \textbf{Relative performance of RCPs compared to GCPs across common benchmarks.}
        For single-task benchmarks, we implement  \chitransformer, \sudeepdit and \chiunet. For each architecture, we average performance of the best training checkpoint across 3 seeds.
        For multi-task benchmarks, we use \pizero as base policy and finetune it on full \libero benchmark (130 tasks).
        We then report the performance of the best-performing architecture, chosen individually for both RCPs and GCPs.
        For Flow, we always do 9 step Euler integrations, where its performance plateaued.
        For readability, RCPs success rates are plotted relative to flow, with flow normalized to performance of $1$ per task.  Tasks are grouped by observation modality, and ordered by relative RCPs performance. Red dashed line indicates threshold at which RCP attains $<95\%$ success of GCPs. \emph{Note that RP and \Flow perform comparably on most Image, 3D-based and VLA-based multi-task benchmarks.}
    }
    \label{fig:rp_vs_gp_best}
    \iftoggle{iclr}{
        \vspace{-0.8cm}
    }{}
\end{figure}

\subsection{When controlled for architecture, GCPs only outperform on  few tasks}
\label{sec:myth_performance}
We first isolate the tasks in which GCPs exhibit stronger performance by comparing across 28 popular BC benchmarks including multi-task benchmarks like \libero (detailed in~\cref{sec:task_settings}), encompassing diverse data quality, modalities (\textbf{state}, \textbf{point clouds}, \textbf{image} and \textbf{language}), and domains (e.g., \metaworld, \taskname{Robomimic}, \adroit, \taskname{D4RL}, \taskname{Meta-World}, \libero).
Crucially, we implement RCPs using the \textbf{exact same architectures} as their corresponding flow models by simply setting the noise level and initial noise to zero: $z=0$, $t = 0$, and study three widely-used architectures (\chitransformer, \sudeepdit, \chiunet as well as pre-trained VLA models like \pizero~\citep{black2024pi_0}; detailed in \cref{sec:control_architecture}).
This architectural alignment enables RCPs to benefit from the sophisticated network designs typically reserved for GCPs, ensuring a fair comparison.

Under controlled comparison, we find GCPs and RCPs achieve parity across the vast majority of \iftoggle{arxiv}{state-based, image-based, and VLA-based}{} BC benchmarks.
Performance gaps emerge only on a small subset of tasks requiring \textbf{high precision} (e.g. precise insertion tasks). We report best-case results in~\cref{fig:rp_vs_gp_best} and comprehensive ablations (including worst-case architectures and loss variants) in~\cref{sec:gp_vs_rp_full_result}.

\iftoggle{arxiv}{
    Our evaluation yields three key insights:
    \begin{itemize}[itemsep=0pt, leftmargin=*]
        \item \textbf{Rare Benefit of GCPs:} GCPs outperform RCPs by $>5\%$ on only a handful of tasks.
        \item \textbf{Modality Independence:} Contrary to popular belief, observation modality does \emph{not} correlate with GCP advantage.
        \item \textbf{Architectural Dominance:} Architecture choice dictates performance far more than the generative vs. regression distinction.
    \end{itemize}
}{
    Our evaluation yields three key insights:
    (a)\textbf{Rare Benefit of GCPs:} GCPs outperform RCPs by $>5\%$ on only a handful of tasks.
    (b)\textbf{Modality Independence:} Contrary to popular belief, observation modality does \emph{not} correlate with GCP advantage.
    (c)\textbf{Architectural Dominance:} Architecture choice dictates performance far more than the generative vs. regression distinction.
}

We posit that the perceived superiority of GCPs in prior work was confounded by architectural asymmetry.
To our knowledge, this is the first study to benchmark \sudeepdit, \chiunet, and \pizero backbones as regression policies.
In~\cref{sec:architecture}, we demonstrate that when equipped with these modern backbones—or even tuned \mlp baselines—RCPs are highly competitive.
Furthermore, we find that hyperparameters such as \textbf{action-chunking} horizon~\citep{zhao2023learning,chi2023diffusion,zhang2025actionchunkingexploratorydata} exert a greater influence on success rate than the choice of objective function (\cref{sec:action_chunk_size_study}).
\iftoggle{arxiv}{
    \\
    \par\noindent\rule{\linewidth}{0.4pt}
    \begin{quote}
        Design decisions like architecture and action-chunking have a significant and consistent impact on control performance. In contrast, the choice between GCPs and RCPs is largely negligible outside of high-precision regimes.
    \end{quote}
    \par\noindent\rule{\linewidth}{0.4pt}
}
{
    \takeawaybold{Design decisions like architecture and action-chunking have a significant and consistent impact on control performance. In contrast, the choice between GCPs and RCPs is largely negligible outside of high-precision regimes.}
}

\subsection{GCPs' performance does not arise from multi-modality}
\label{sec:multimodality}

Earlier literature  suggested that capturing multi-modality, as defined in \Cref{sec:prelim}, was precisely the root of the observed performance benefits of GCPs \citep{chi2023diffusion, reuss2023goal}.
However, examining \cref{fig:rp_vs_gp_best}, we see that many tasks which have been understood to be multimodal (e.g., \pusht) do not show substantial performance gaps between RCPs and GCPs.
On the other hand, RCPs and GCPs differ only on tasks that demand high precision (e.g. \toolhang, \transport).
In this section, we provide additional evidence that \takeawaybold{multimodality is not the main factor responsible for witnessed performance advantages of GCPs}.

\begin{figure}[ht]
    \centering
    \begin{minipage}[t]{0.32\linewidth}
        \centering
        \includegraphics[width=1.0\linewidth]{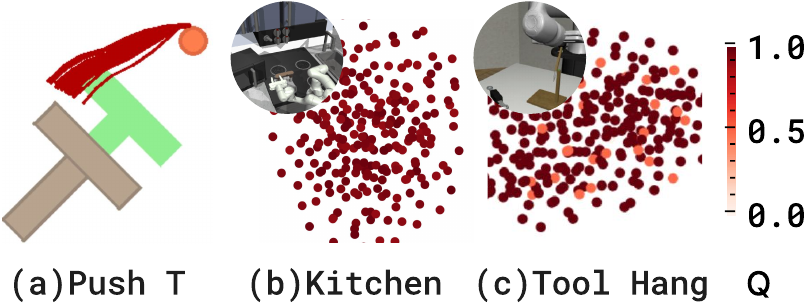}
        \captionof{figure}{\footnotesize \textbf{A. Visualized action distribution with Q values.}
            Distinct modes are \textbf{not}
            observed in planned actions even at symmetric and ambiguous states. (\kitchen{} and \toolhang{},  t-SNE visualization.) In \pusht, we all trajectories goes to one side. For the rest, there is no clear clustering of actions or Q.  %
        }
        \label{fig:action_distribution}
    \end{minipage}
    \hfill
    \begin{minipage}[t]{0.32\linewidth}
        \vspace{-1.5cm}
        \centering
        \scriptsize
        \begin{tabular}{@{}lccc@{}}
            \toprule
            \textbf{Task} & $z\equiv 0$ & {$\Normal(0,\eye)$} & \textbf{Mean} $z$ \\
            \midrule
            \pusht        & 0.97        & 0.97                & 0.95              \\
            \kitchen      & 0.99        & 0.99                & 0.97              \\
            \toolhang     & 0.78        & 0.80                & 0.76              \\
            \bottomrule
        \end{tabular}
        \captionof{table}{\footnotesize \textbf{B. Performance comparison of different sampling strategies}. We compare sampling $z=0$, $z\sim \Normal(0,\eye)$, and mean  over 64 $z^{(i)}\sim \Normal(0,\eye)$. Different sampling strategies  show minor performance difference, indicating  absence of distinct action modes.}
        \label{tab:mean_action}
    \end{minipage}
    \hfill
    \begin{minipage}[t]{0.32\linewidth}
        \vspace{-1.5cm}
        \centering
        \scriptsize
        \begin{tabular}{@{}lcc@{}}
            \toprule
            \textbf{Dataset} & \Flow & \algname{Reg.} \\
            \midrule
            Original         & 0.78  & 0.58           \\
            Deterministic    & 0.72  & 0.64           \\
            \bottomrule
        \end{tabular}
        \captionof{table}{\footnotesize \textbf{C. GCPs outperforms RCPs with deterministic experts.} Policy average success rate over 3 architectures, 3 seeds and 3 architectures given different dataset: one from original human demonstration and another collected by rolling out a flow policy in deterministic mode starting from zero noise. }
        \label{tab:deterministic_relabel}
    \end{minipage}
    \vspace{-0.5cm}
\end{figure}

\iclrpar{Evidence A: GCPs exhibit unstructured action distributions.}
For fixed observations, we draw multiple action samples by denoising from different initial latents and visualize the resulting action set with their Q values $Q(a,o)$.
We deliberately choose \emph{symmetry-critical} or \emph{high-ambiguity} states to \emph{maximize} potential multi-modality:
(a) \pusht{} at the symmetry axis of the T-shape, where taking the left or right path is equivalent,
(b) \kitchen{} from an initial state with multiple first-subtask choices, and
(c) \toolhang{} at the insertion pre-contact pose where human demonstrators pause for varying durations.
In (a-c) we observe \emph{single} clusters rather than distinct modes (high-dimensional actions visualized with t-SNE); see \cref{fig:action_distribution}.
Moreover, adherence to action cluster means do not correlate with performance: We color-code actions by Q-value, i.e. Monte-Carlo-estimated rewards-to-go (\cref{sec:q_function_estimation}). Highest returns are distributed evenly across samples.%

\iclrpar{Evidence B: Taking mean actions does not meaningfully degrade GCPs' performance.}
We evaluate flow policy's performance with three sampling strategies: zero noise $a=\flowpol(z=0, o)$, stochastic sampling $a=\flowpol(z, o), z \sim \Normal(0,I)$, and \emph{mean action} $a=\mathbb{E}_{z\sim\Normal(0,I)}[\flowpol(z, o)]$ (via Monte Carlo approximation).
If the learned distribution were \iftoggle{arxiv}{strongly}{} multi-modal, or if their distributions lied on a manifold whose \emph{curvature} was crucial to task success, the conditional mean would \emph{collapse} modes and severely degrade performance.
However, \cref{tab:mean_action} shows that replacing stochastic sampling with the mean action only slightly affects performance, indicating absence of distinct action modes.

\iclrpar{Evidence C: GCPs outperform RCPs on certain tasks even with deterministic experts.}
To fully remove any residual multi-modality, we recollect the dataset with trained flow policy evaluated in deterministic mode ($z=0$) detailed in \cref{sec:deterministic_dataset_generation}.
The new dataset is fully deterministic because action labels are provided by a deterministic policy evaluated in a deterministic environment. While the gap in performance between GCPs and RCPs shrinks somewhat,  we still find that GCPs still outperforms RCPs, as in \cref{tab:deterministic_relabel}, suggesting that capturing some ``hidden'' stochasticity or multimodality in the data does not suffice to explain the gap between the two.
\iftoggle{arxiv}{
    \\
    \hrule
    \begin{quote}
        Collectively, (A)--(C) indicate that the commonly cited explanation---``GCPs win because demonstrations are multi-modal''---does not hold for most studied behavior cloning benchmarks.%
    \end{quote}
    \hrule
    \vspace{.25em}
}
{
    \takeawaybold{Collectively, (A)--(C) indicate that the commonly cited explanation---``GCPs win because demonstrations are multi-modal''---does not hold for most studied behavior cloning benchmarks.}
}

\textbf{Multi-modality and data coverage.} The absence of observed multimodality is likely attributable to the large observation dimension of tasks relative to total number of demonstrations.  That is, we rarely see two ``conflicting'' actions for nearby observation vectors (note: to grid a space of dimension $d$ requires $2^{d}$ points). Some degree of ``hidden'' multi-modality may still be present, as indicated by the slight narrowing of the performance gap in \Cref{tab:deterministic_relabel}. Still, our central claim is that multi-modality is not \emph{sufficient} to explain the full difference in performance. \iftoggle{iclr}{}{Understanding to what extent multimodality appears in the multi-task setting is an exciting direction for future research.}

\subsection{Limitations of the expressivity of GCPs in the absence of multimodality}
\label{sec:expressivity}
\iftoggle{arxiv}{
    \noindent
    \begin{figure}[ht]
        \centering
        \includegraphics[width=1.0\linewidth]{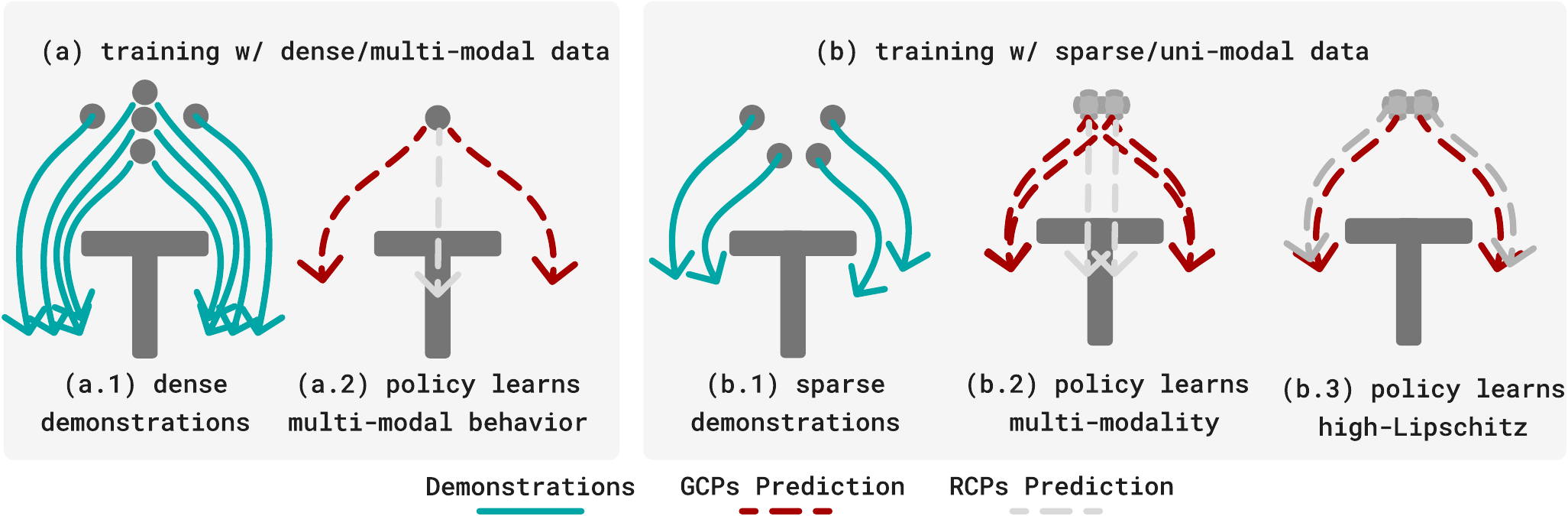}
        \captionof{figure}{\footnotesize \textbf{GCP behavior given different types of data.}
            (a) Given true multi-modal data (a.1), where expert have two behaviors at the same state, GCPs can learn both modes while RCP collapse to the middle (a.2).
            (b) In reality, the data is often sparse given high-dimensional space. Given the sparse data (b.1), GCPs have two possible behaviors: (b.2) still learn both modes given close-by states, (b.3) learn a high-Lipschitz policy to quickly switch between modes.   In our experiments, we find  that both \textbf{GCPs and RCPs} learn (b.3) in high-dimensional tasks (\Cref{sec:multimodality}). In this regime, \Cref{thm:informal} then suggests that, from a pure expressivity perspective, GCPs have a limited advantage over RCPs.
        }
        \label{fig:multimodality_vs_lipschitz}
    \end{figure}
}
{
}
An alternative to learning explicit multimodality is to represent rapid transition between actions as the observation changes. This is depicted in \Cref{fig:multimodality_vs_lipschitz}, where data that appears multi-modal can be fit with a policy that has a high Lipschitz constant, i.e. in which $\nabla_o \pi(a \mid o)$ is large. This reflects a broader principle in control that we need only capture the mapping from observation to a single effective action, rather than  reproduce the  distribution over all possible actions.

One may still conjecture that GCPs more easily  higher-Lipschitz policies by leveraging iterative computation, as compared to RCPs. This is because  deeper networks can express larger-Lipschitz functions more easily \citep{telgarsky2016benefits}, and many have equated the multi-step computation in flow-based generative models to depth \citep{chen2018neural}. Step-by-step generation is known to drastically increase expressivity in other domains as well, such as autoregressive language models~\citep{li2024chainthoughtempowerstransformers},

However, flow-based generative models use their multi-step computation to express complex distributions over the \emph{generated variable} \citep{ho2020denoising,song2021denoising,zhang2022fast,nichol2021improved}.
It is less clear if the iteration computation assists with represent complex \emph{observation-to-action} mappings. %
Thus, we ask:
\iftoggle{arxiv}{
    \begin{quote}
        Does the iterative computation in GCPs aid in learning more complex observation-to-action mappings, even if the learned action distributions for a fixed observation are themselves are relatively simple (i.e. unimodal)?
    \end{quote}
}{
    \takeawaybold{Does the iterative computation in GCPs aid in learning more complex observation-to-action mappings, even if the learned action distributions for a fixed observation are themselves are relatively simple (i.e. unimodal)?}
}

\iftoggle{arxiv}{
    \begin{figure}[ht]
        \centering
        \includegraphics[width=1.0\linewidth]{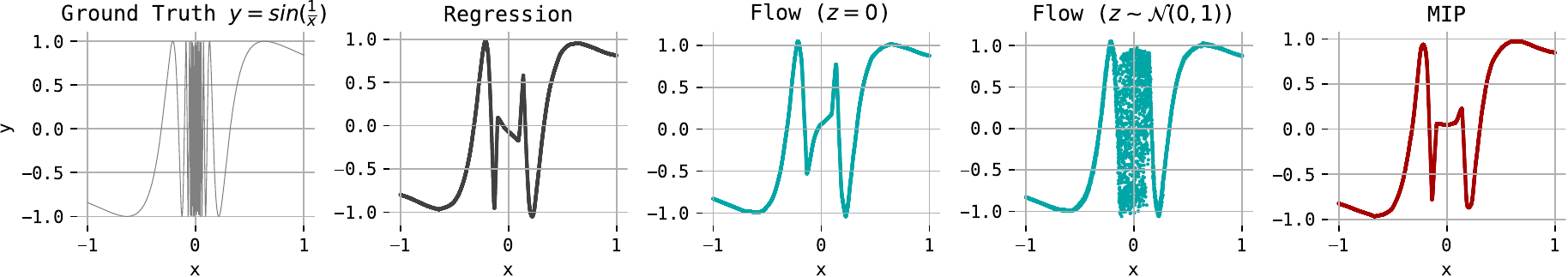}
        \captionof{figure}{\footnotesize \textbf{The Myth of Superior Expressivity: Fitting High-Frequency Functions.}
            We evaluate GCPs, RCPs, and \minimaliterativepolicy (\Minimaliterativepolicy, \Cref{sec:parsing}) on fitting the high-Lipschitz function $y = \sin(1/x)$ ($N=1024$, 4-layer MLP).
            Contrary to the belief that iterative evaluation yields sharper function approximation, both GCPs and RCPs fail to capture the high-frequency structure given limited network capacity.
            While RCPs succumb to spectral bias by averaging the oscillations, GCPs merely trade this averaging for stochastic variance.
            Crucially, when initial noise is fixed ($x_0=0$), the Flow policy collapses to the exact same mean-seeking behavior as regression. Only when averaged over initial noise variance to we start to see a tradeoff from epistemic uncertainty to aleatoric variance.
            This demonstrates that GCPs do not inherently bypass the spectral limitations of the underlying backbone to achieve greater Lipschitz expressivity.}
        \label{fig:high_lipschitz_function_learning}
    \end{figure}
}{
}

We now provide evidence that suggests ``\textbf{no}.'' We show   that {in the absence of multi-modality (as shown in~\cref{sec:multimodality}), GCPs cannot express more complex mappings from the conditioning variable $o$ to the generated variable $a$ than RCPs can.
We begin by considering a ground-truth conditional flow field $b_t^\star(o \mid a)$.
Let  $\piideal_\theta(z, o)$  represent the exactly integrated $\flowideal$ from initial noise $z$ to \iftoggle{iclr}{generated variable}{} $a$.
Given the absence of multi-modality (\Cref{sec:multimodality}), we assume that the distribution of $a \mid o$ is $\kappa$-log-concave (\Cref{app:theory}), satisfied by many classical unimodal distributions.
We prove that the Lipschitz constant of $\piideal_\theta(z , o)$ with respect to $o$, a measure of the expressivity of the $o\to a$ mapping, is bounded by that of $b^\star_t$:
\begin{theorem}[Informal]
    \label{thm:informal}
    Let $\|\cdot\|$ denote either the matrix operator or Frobenius norm, and suppose that the distribution of $a \mid o$ is $\kappa$-log-concave. Moreover, suppose that the flow field $\flowideal_t(a \mid o)$ is $L$-Lipschitz: $\|\nabla_o \flowideal_t(a \mid o)\| \le L$. Then, with {infinite} integration steps, \iftoggle{iclr}
    {$ \|\nabla_o\piideal_{\theta}(z, o)\| \le L \cdot \sqrt{1+\kappa^{-1}}$.}{we  have the bound
        \begin{align}
            \label{eq:informal}
            \|\nabla_o\piideal_{\theta}(z, o)\| \le L \cdot \sqrt{1+\kappa^{-1}}.
        \end{align}
    }
\end{theorem}
See \Cref{app:theory} for a formal statement and proof\iftoggle{iclr}{.}{, adopting a careful argument from \citet{daniels2025contractivity}.}
A classical example of a log concave distribution is $a \mid o \sim \Normal(\mu(o), \frac{1}{\kappa})$; as long as the variance $1/\kappa$ is bounded \emph{above} (even in the limit of a Dirac\iftoggle{iclr}{}{ measure}), there is at most a constant-multiplicative factor increase in the Lipschitz constant.
When training a flow, $\flowideal_t(a \mid o)$ is approximated by the neural network.
Thus, in the prototypical unimodal example of $\kappa$-log-concave  distributions, GCPs are not arbitrarily more expressive than RCPs.
In fewer words: \takeawaybold{more integration steps (i.e. more iterative computation), even infinitely many, need not enable greater expressivity of high Lipschitz $o\to a$ mappings}.

\begin{wraptable}{r}{0.5\textwidth}
    \iftoggle{iclr}{\vspace{-0.2cm}}{}
    \centering
    \scriptsize
    \begin{tabular}{@{}l c c c c c@{}}
        \toprule
        \textbf{Method} & \multicolumn{2}{c}{\pusht} & \kitchen       & \multicolumn{2}{c}{\toolhang}                                   \\
        \cmidrule(lr){2-3} \cmidrule(lr){5-6}
                        & State                      & Image          & State                         & State          & Image          \\
        \midrule
        \Regression     & {\tiny $0.90$}             & {\tiny $0.55$} & {\tiny $14.07$}               & {\tiny $1.71$} & {\tiny $1.65$} \\
        \Flow           & {\tiny $0.45$}             & {\tiny $0.20$} & {\tiny $12.43$}               & {\tiny $1.41$} & {\tiny $1.37$} \\
        \bottomrule
    \end{tabular}
    \vspace{-0.25cm}
    \caption{
        \footnotesize
        \textbf{Policy Lipschitz constant comparison.}
        Lipschitz constant is averaged over 100 states.
    }
    \label{tab:compact_lipschitz}
    \iftoggle{iclr}{\vspace{-0.6cm}}{}
\end{wraptable}

To verify our theoretical prediction, we quantify learned policies' Lipchitz constants with a zeroth-order proxy:
starting from dataset states $s_t$ with observation $o_t$, we inject small Gaussian perturbations in the executed action to reach a \emph{feasible} nearby state $s_{t+1}^{(i)}$ with observation $o_{t+1}^{(i)}$, then measure input--output sensitivity via finite differences of the policy around the perturbed states (full algorithm and per-architecture results in \cref{sec:lipchitz_study_details}).
This construction (i) avoids reliance on noisy higher-order gradients in complex architectures, and (ii) keeps evaluations on feasible observation to prevent conflating expressivity with model error on dynamically infeasible states.
As predicted by our theory, GCPs are not strictly more expressive than RCPs as shown in \cref{tab:compact_lipschitz}.
On the contrary, RCPs show increased Lipschitz constants off the manifold of training data, ruling out the assumption that GCPs win due to expressing policies with greater sensitivity to the input variable. We note that our methodology, which perturbs actions rather than states, is compatible pixel observations. To summarize:
\iftoggle{arxiv}{
    \hrule
    \begin{quote}
        In the absence of multimodality, GCPs do not enjoy an advantage over RCPs in  expressing high Lipschitz behavior, such as  rapid transitions between action modes.
    \end{quote}
    \hrule
}{
    \takeawaybold{In the absence of multimodality, GCPs do not enjoy an advantage over RCPs in  expressing high Lipschitz behavior, such as  rapid transitions between action modes.}
}

\iftoggle{arxiv}{
    \subsection{ GCPs and RCPs Exhibit Comparable Behavior Diversity}\label{sec:diversity}

    \begin{figure}[ht]
        \centering
        \includegraphics[width=1.0\linewidth]{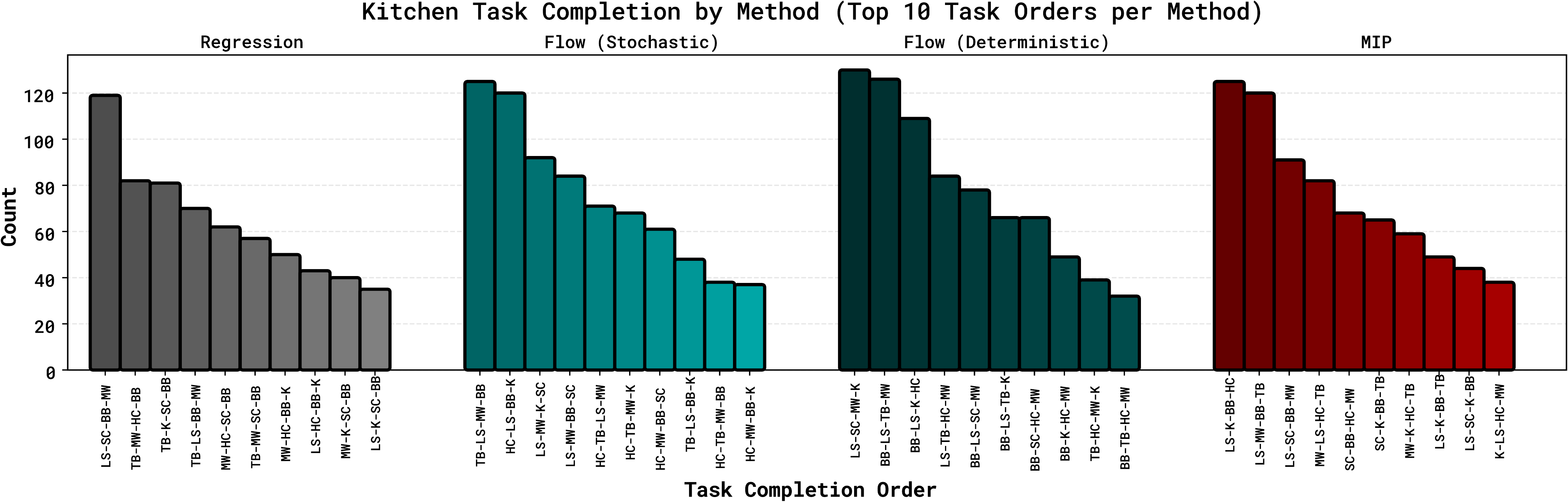}
        \captionof{figure}{\footnotesize \textbf{Task completion order in Kitchen environment with different methods.}
            We plot the count of different task completion orders for different methods to evaluate the diversity of the policies.
            The x-axis shows the task completion order, where each sub-task is represented by its initials.
            For each run, we collect 1000 trajectories with the same seed shared by all methods.
            For flow, we evaluate both stochastic and deterministic modes.
        }
        \label{fig:kitchen_methods_comparison}
    \end{figure}

    We conclude by rebutting a commonly believed hypothesis is that GCPs can express more diverse behaviors than RCPs by capturing the full distribution of expert actions~\citep{shafiullah2022behavior}.\footnote{Note that the expert might be stochastic but unimodal, so the findings in this section do not directly follow form those in \Cref{sec:multimodality}.}
    We evalute different variants of GCPs and RCPs on \KitchenEnv, where the expert shows multiple task completion orders.
    As demonstrated in \cref{fig:kitchen_methods_comparison}, GCPs with both stochastic and deterministic sampling show similar task completion order diversity.
    Deterministic policies like regression and \Minimaliterativepolicy (to be introduced in \Cref{sec:parsing}) also demonstrate similar task completion order diversity.
    This indicates that, given sparse expert demonstrations, both GCPs and RCP learns high-Lipschitz policies to switch between different modes given different observations (corresponding to (b.2) case in~\Cref{fig:multimodality_vs_lipschitz}).  RCPs and GCPs are equally good at learning such behaviors (\Cref{fig:kitchen_methods_comparison}),  which explain why we see similar performance for both policy parametrizations, even on seemingly multi-modal tasks like \KitchenEnv.
}
{}

\newcommand{\pirr}{\pi^{\Residualregression}}
\newcommand{\pisf}{\pi^{\Straightflow}}
\newcommand{\pimip}{\pi^{\Minimaliterativepolicy}}
\newcommand{\bsf}{b^{\Straightflow}}
\newcommand{\tfix}{t_{\star}}
\newcommand{\pimipideal}{\pi^{\Minimaliterativepolicy,\mathrm{ideal}}}

\newcommand{\ahat}{\hat{a}}
\newcommand{\ahatts}{\hat{a}^{\algname{TSD}}}
\newcommand{\pits}{\pi^{\algname{TSD}}}
\newcommand{\ahatmip}{\hat{a}^{\Minimaliterativepolicy}}

\section{ Minimal Iterative Policy (\Minimaliterativepolicy): Isolating the Source of GCPs' Success }

\label{sec:parsing}

In this section, we introduce a number of intermediates between RCPs and GCPs that isolate which design decisions contribute to the latter's superior performance. This leads to a Minimal Iterative Policy (\Minimaliterativepolicy), which matches GCPs performance, thereby identifying the source of GCPs' success.

\iftoggle{iclr}{}{
}

\label{sec:core_design_components}
We begin with a taxonomy of the  three key algorithmic components (\iftoggle{arxiv}{\Cref{fig:taxonomy}}{\Cref{fig:teaser}}) present in GCPs.
\iftoggle{iclr}
{
    \cref{sec:algorithms_with_different_design_components}
    below proposes algorithmic variants which ablate these components. We find that \minimaliterativepolicy (\Minimaliterativepolicy, \Cref{comp:sic,comp:stoch}) is the reduced variant which matches the performance of flow (\Cref{sec:performance_comparison}), whereas other variants match or perform worse than regression.
}
{}

\begin{component}\label{comp:distr}
    {}\takeawaybold{Distributional learning} denotes training a model to fit a conditional distribution $\gen \sim \pi_\theta(\cond)$ of actions given observations, as opposed to deterministic predictions (i.e., $\gen = \pi_\theta(\cond)$). \footnotemark
\end{component}

\begin{component}\label{comp:stoch}\takeawaybold{Stochasticity injection} denotes the injection of additional stochastic inputs into the neural network during training time (e.g., the variable $z$ in \Cref{eq:flow}).
\end{component}

\begin{component}\label{comp:sic}\takeawaybold{Supervised Iterative Computation (SIC)} denotes the iterative refinement of predictions by feeding the previous outputs into the same network again during inference, {and providing \emph{supervision signals} at every step of the generation procedure at training time}. For example, in flow GCPs, we integrate a supervised flow field $b_t(\gen_t \mid \cond)$ over time to get the final action $\gen$, and that $b_t$ receives an independent supervisory signal for each $t$ at training time (\Cref{eq:flow}).
\end{component}

\iftoggle{arxiv}{
    \begin{figure}[h]
        \centering
        \includegraphics[width=1.0\textwidth]{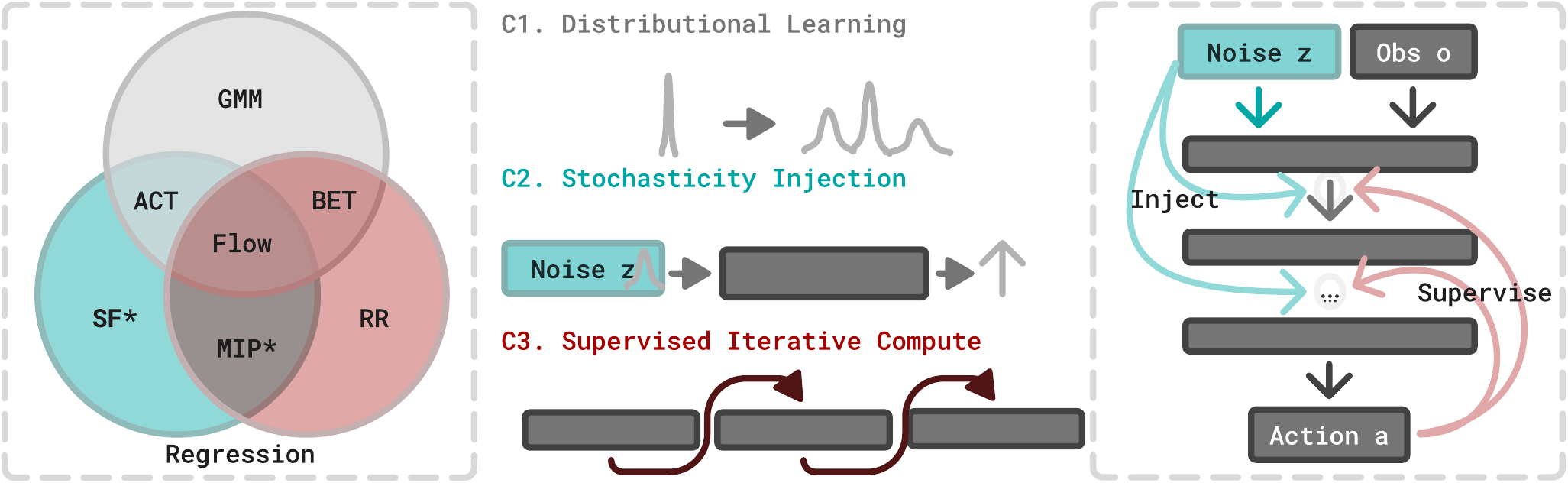}
        \caption{
            \textbf{Taxonomy of GCPs.} We elucidate 3 key design components of GCPs: distributional learning (fitting a distribution), stochasticity injection (injecting noise during training), and supervised iterative computation (multiple generation steps, each with its own supervised learning loss).
            To ablate different design components, we introduce Straight Flow (\Straightflow, \Cref{eq:straight_flow}), Residual Regression (\Residualregression, \Cref{eq:residual_regression}) and Minimal Iterative Policy (\Minimaliterativepolicy, \Cref{eq:mip_training}), which are variants of flow that only exhibit \Cref{comp:stoch} and \Cref{comp:sic}, respectively.}
        \label{fig:taxonomy}
    \end{figure}
}{
}
\footnotetext{Note that \Cref{comp:distr} refers to \emph{training} a model to fit a conditional distribution, not necessarily to the sampling. For example, training $b_\theta$ via flow model but conducting deterministic inference with $\Phieul(z = 0 \mid o)$ is still considered distributional learning.}

\iftoggle{iclr}
{
}
{
    From here, \cref{sec:algorithms_with_different_design_components}
    proposes algorithmic variants which ablate these components: two novel variants we call \minimaliterativepolicy (\Minimaliterativepolicy, \Cref{comp:sic,comp:stoch}) and straight-flow (\Straightflow, \Cref{comp:stoch} only), as well as a residual regression baseline (\Residualregression, \Cref{comp:sic} only). We evaluate the performance of different variants on challenging tasks in~\cref{sec:performance_comparison}, {finding that \Minimaliterativepolicy exhibits virtually the \emph{same} performance as \Flow across tasks}, whereas \Straightflow matches the performance of \Regression and \Residualregression exhibits even {worse performance}.
    This establishes that \Cref{comp:sic,comp:stoch}: SIC, when combined with {stochasticity injection}, drive performance. Finally, we contrast \Minimaliterativepolicy with other popular step policies ( \Cref{sec:mip_v_shortcut}).
}

\subsection{\Minimaliterativepolicy: a minimal  intermediate between RCPs and GCPs}
\label{sec:algorithms_with_different_design_components}

We introduce a range of policies which lie along the spectrum between RCP and flow-based GCPs via varying combinations of \Cref{comp:sic,comp:stoch}, culminating in the Minimal Iterative Policy (\Minimaliterativepolicy). These policies do not satisfy  \Cref{comp:distr}, because   \Cref{sec:multimodality,sec:diversity} suggests that this is not needed.
In particular, we consider networks $\pi_\theta(o,I_t,t)$ that predict \emph{actions}, not velocities, and given observations $o$, time indices $t$, and interpolants $I_t$  corresponding to noising actions. We state all networks below of $L_2$ minimization, but our findings remain consistent when minimizing $L_1$ error instead (\Cref{sec:loss_norm_type_ablation_study}).

\iftoggle{arxiv}{
    \iclrpar{Regression as Single-Step Denoising.} We begin by expression a regression policies (RCPs) as solving a single-step denoising problem, obtained by minimizing the $L_2$ prediction error of the action given observation and null action interpolant:
    \begin{align}
        \iclrstyle \pi^{\algname{RCP}}_{\theta} \approx  \argmin_\theta \underset{}{\Exp} \left[\|(\pi_{\theta}(o,
            I_0 = 0,t=0) - a\|^2\right].
        \label{eq:regression}
    \end{align}
    where $(o,a) \sim \Dtrain$.%
    In the limit of infinity data, RCPs predict the conditional mean of $a \mid o$ by mapping any noise $z$ to the same action $a$ given $o$.\footnote{Note that in our comparisons between RCP and GCP (\Cref{sec:dispelling}) in, we use the \Cref{eq:regression} to implement RCPs on GCP architectures.}

    \iclrpar{Straight Flow} (\Straightflow, ours). Next we introduce Straightflow (\Straightflow), which adds only  stochasticity injection \Cref{comp:stoch} to RCPs. This is achieved by setting the interpolant $I_0$ to be Gaussian:
    \begin{align}
        \pisf_{\theta} \approx \iftoggle{iclr}{%
        }{} \argmin_\theta \Exp \iftoggle{iclr}{\textstyle}{} \|\pi_{\theta}(o, I_0 = z, t=0) - a\|^2,
        \label{eq:straight_flow}
    \end{align}
    where $(o,a) \sim \Dtrain, z \sim \Normal(0,\eye)$.
    Inference is performed in a single step, by setting $a = \pisf_{\theta}(o,z,t=0)$. Equivalently, \Straightflow can be viewed as a flow model in which the flow field is constrained to be straight.

    Like RCPs, the optimal  \Straightflow policy is  the {conditional mean} of $a \mid o$. The only difference between the two is injection of stochastic input $z$ during training.
    Our experiments with \Straightflow precisely isolate this effect---for example, determining if the additional stochasticity during training improves learning dynamics, or behaves like data augmentation.
    Like \Minimaliterativepolicy below, we set $I_0 = 0$ at inference time, as stochasticity at inference time has little effect on policy performance.
}{}

\iclrpar{Two-Step Denoising.}
\iftoggle{arxiv}{As a next step towards GCPs,  we now consider a \takeawaybold{two-step denoising} (\algname{TSD}) policy. }{
    As a simplification of flow-based GCPs, we consider a \takeawaybold{two-step denoising} (\algname{TSD}) policy.
}
As discussed in \iftoggle{arxiv}{\Cref{rem:shortcut}}{\Cref{sec:comp_consistency_models}}, this parametrization is superficially similar to, but substantively different than, popular flow-map/consistency/shortcut models \citep{boffiHowBuildConsistency2025}.
\algname{TSD} performs two steps of denoising, one from zero, and a second from a fixed index $\tfix = .9$:
\begin{align}
    \iclrstyle \pits_{\theta} \approx  \argmin_\theta \underset{}{\Exp} \left[\|(\pi_{\theta}(o,
    I_0 = z,t=0) - (\tfix)^{-1}I_{\tfix})\|^2 + \|(\pi_{\theta}(o, I_{\tfix}, \tfix) - a)\|^2\right].
    \label{eq:tsd}
\end{align}
where $(o,a) \sim \Dtrain, z \sim \Normal(0,\eye)$, and $I_t = ta + (1-t)z$ is the same interpolant used in flow models, and  where $\tfix = .9$ is fixed. The normalization  by $\tfix$ in \Cref{eq:tsd} comes from the identity $\tfix a = \Exp_z[I_{\tfix}]$.
We then sample
$ \ahatts_0 \gets \pi_{\theta}(o,z,0)$ and $\ahatts \gets \pi_{\theta}(o, \tfix   \ahatts_0 + (1-\tfix)z,\tfix)$.

\iclrpar{Minimal Iterative Policy.} We find that $\pits$ performs equivalently to a minimal policy which only adds training noise in the second step and has no stochasticity at inference time, which we call the \minimaliterativepolicy.

\begin{AIbox}{Minimal Iterative Policy ($\Minimaliterativepolicy$; ours) }
    Minimal Iterative Policy  (\Minimaliterativepolicy), representing \Cref{comp:sic,comp:stoch}, is trained via
    \begin{align}
        \iclrstyle  \pimip_{\theta} \approx & %
        \argmin_\theta \underset{}{\E} (\|(\pi_{\theta}(o,I_0 = 0,t=0) -   a)\|^2 + \|(\pi_{\theta}(o, I_{\tfix}, \tfix) - a)\|^2),
        \label{eq:mip_training}
    \end{align}
    where $(o,a) \sim \Dtrain, z \sim \Normal(0,\eye), \tfix:=.9$. At inference time, we compute:
    \begin{align}
        \ahatmip_0\gets \pimip_{\theta}(o, 0, t = 0), \quad \ahatmip \gets \pimip_{\theta}(o, \tfix \ahatmip_0, \tfix).
        \label{eq:pimip_inference}
    \end{align}
\end{AIbox}
\mminimaliterativepolicy provides a \emph{minimal} implementation that still exhibits competitive performance with flow. Starting, with \algname{TSD} and replace $(\tfix)^{-1}I_{\tfix}$ in the first term of the loss in \Cref{eq:tsd} with its expectation $a=(\tfix)^{-1}\Exp[I_{\tfix}]$.  We set the initial noise $I_0 = 0$ to be zero, so that $z$ only contributes to the second training loss. Finally, we sample with $z=0$  to isolate the effect of adding {stochasticity at training time},  without stochasticity at inference time \iftoggle{iclr}{(c.f.~\cref{tab:mean_action}) }{(as suggested by~\cref{tab:mean_action})}.
Since we provide supervision for both first step $\pimip_\theta(o, I_0=0, t=0)$ and second step $\pimip_\theta(o, I_0=I_{\tfix}, t=\tfix)$ with ground truth action $a$, \Minimaliterativepolicy also exemplifies SIC in its simplest form. We compare \Minimaliterativepolicy to Shorctu Models in \iftoggle{arxiv}{\Cref{rem:shortcut}}{\Cref{sec:comp_consistency_models}}.

{
\iclrpar{Additional methods.}
}

\iftoggle{arxiv}{}{
    Straight Flow (\Straightflow, ours), representing only \Cref{comp:stoch}, further simplifies
    MIP to a single stage by setting the interpolation index $t^* = 1$ and removing the second term: $\pi_{\Straightflow_\theta} \approx \E \left[\|\pi_{\theta}(o,z,t = 0) - a\|^2\right]$.}
Finally, we study \residualregression (\Residualregression), which replaces $I_{\tfix}$ in \Cref{eq:mip_training} with its expectation over $z$: $\Exp[I_{\tfix}]=\tfix a$. This preserves SIC (\Cref{comp:sic}) yet removes stochasticity injection.
Full details are provided in
\cref{sec:minimum_iterative_policy_design_ablation}.

To summarize, \minimaliterativepolicy (\Minimaliterativepolicy), \straightflow (\Straightflow) and \residualregression (\Residualregression) represent all combinations of \Cref{comp:sic,comp:stoch} without exhibiting \Cref{comp:distr}.

\subsection{\Cref{comp:sic,comp:stoch} drive performance: \Minimaliterativepolicy matches \Flow }
\label{sec:performance_comparison}

\begin{figure}[!ht]
    \centering
    \iftoggle{iclr}{
        \vspace{-0.3cm}
    }{}
    \includegraphics[width=1.0\textwidth]{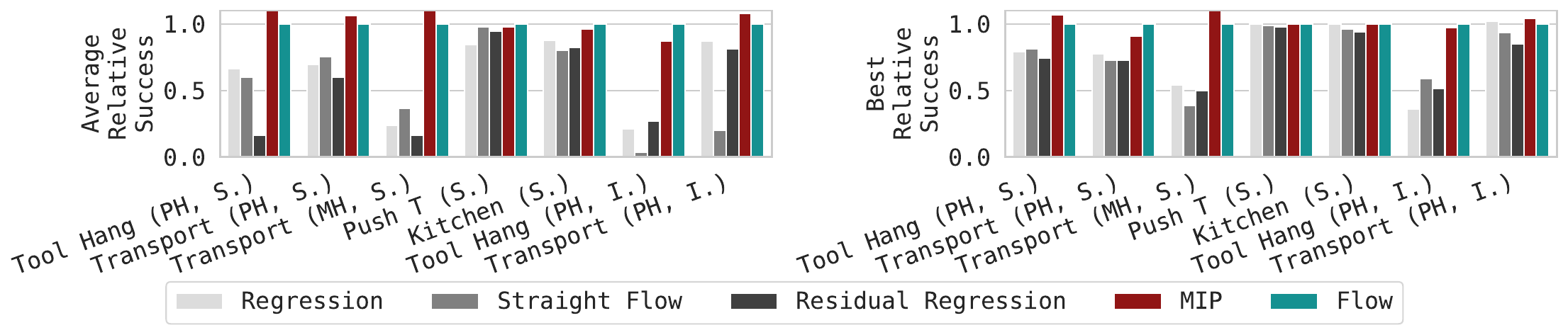}
    \vspace{-0.5cm}
    \caption{
        \footnotesize
        \textbf{Performance comparison between \Minimaliterativepolicy and its variants on single-task benchmarks.}
        Average relative success rate on worst architecture and the best relative success rate on optimal architecture are reported. ``$S$'': state; ``$I$'': image.
    }
    \iftoggle{iclr}{
        \vspace{-0.5cm}
    }{}
    \label{fig:mip_vs_flow_vs_regression}
\end{figure}

\iftoggle{arxiv}{
    \begin{table}[h]
        \centering
        \small
        \begin{tabular}{l|cccc}
            \toprule
            \textbf{Method}             & \liberoobject & \liberogoal   & \liberospatial & \liberoten    \\
            \midrule
            \regression ($\ell_2$ loss) & 92.6          & 94.6          & 97.2           & 78.0          \\
            \regression ($\ell_1$ loss) & 95.2          & 88.0          & 95.8           & 62.4          \\
            \Flow                       & \textbf{97.4} & 95.0          & 95.8           & 81.6          \\
            \Minimaliterativepolicy     & 95.8          & \textbf{95.2} & \textbf{97.6}  & \textbf{82.2} \\
            \bottomrule
        \end{tabular}
        \caption{\textbf{Performance comparison on multi-task \libero benchmark.}
            We report the success rate of the checkpoint trained with 50k gradient steps of finetuning \pizero on the full \libero dataset. We implement \Minimaliterativepolicy with $t^* = 0.9$ and integrate flow with 10 steps.
            For regression, we train with both $\ell_2$ and $\ell_1$ loss as suggested in~\citep{kim2024openvla}.
        }
        \label{tab:mip_multitask}
    \end{table}
}
{
}

Based on the design space parsing in~\cref{sec:parsing}, we are able to systematically ablate different design components' contribution to the final performance in~\cref{fig:mip_vs_flow_vs_regression,tab:mip_multitask}.
Our evaluation shows that either  stochasticity injection (\Cref{comp:stoch}, exhibit by \Straightflow) or supervised iterative computation (\Cref{comp:sic}, exhibited by \Residualregression)  in isolation do not match the success of GCPs.
\Minimaliterativepolicy, being the only method which combines \emph{supervised} iterative computation and stochasticity injection, achieves success on par with flow.
Thus we conclude: the performance of GCPs comes from combining stochastic injection and iterative computation. Distributional training appears to be the least important factor.%
\iftoggle{arxiv}{
}{
    To further rule out the effect of distributional training and demonstrate the computation efficiency of \Minimaliterativepolicy, we compare it with consistency models in \Cref{sec:comp_consistency_models}, where we find that \Minimaliterativepolicy matches the performance of common consistency models with half of the training time.
}
\begin{remark}
    \Cref{sec:mip_variants_results} exhibits two further variants which preserve \Cref{comp:stoch,comp:sic}: one that does not supervise intermediate steps, and a second which does not condition a time step $\tfix$. The latter does not enable network to learn separate functions across time steps. Both perform even worse than regression, confirming the importance of supervision of intermediate steps and decoupling network behavior across time steps.
\end{remark}

\iftoggle{arxiv}{
    \subsection{ \Minimaliterativepolicy  compares favorably to shortcut policies}
    \label{sec:mip_v_shortcut}
    \label{rem:shortcut}
    \Minimaliterativepolicy is superficially similar to Shortcut Models~\citep{boffiFlowMapMatching2025,boffiHowBuildConsistency2025,songConsistencyModels2023,gengMeanFlowsOnestep2025}, as  both perform inferences in few-steps.  Shortcut models correctly learn target distributions (i.e. satisfy \Cref{comp:distr}) by integrating a flow field. On the other hand,  \Minimaliterativepolicy are trained to predict the conditional mean of the interpolant, which is not a valid objective for distribution fitting. The performance of \Minimaliterativepolicy supports our overall theme that, in robotic control applications, faithfully capturing the full conditional distribution over actions is not needed for control performance.

    \begin{wraptable}{r}{0.5\textwidth}
        \centering
        \scriptsize
        \begin{tabular}{l|ccc}
            \toprule
            \textbf{Method}         & \multicolumn{2}{c}{\transport} & \toolhang                                        \\
            \midrule
                                    & mh                             & ph                 &                             \\
            \midrule
            \Flow                   & 0.52/0.40                      & 0.80/\textbf{0.73} & 0.84/0.70                   \\
            \Minimaliterativepolicy & \textbf{0.62}/\textbf{0.46}    & 0.80/0.69          & \textbf{0.92}/\textbf{0.88} \\
            \algname{CTM}           & 0.57/0.32                      & \textbf{0.90}/0.58 & 0.56/0.26                   \\
            \bottomrule
        \end{tabular}
        \caption{\textbf{Performance comparison between \Minimaliterativepolicy and shortcut policies.} Report best/average performance across 5 checkpoints with 3 random seeds. Task is state-based. For \algname{CTM}, we report the performance with 2 integration steps, which is the same as \Minimaliterativepolicy. Note that \Minimaliterativepolicy is always best or near-best on average-over-seed performance, whereas CTM's average  performance struggles. }
        \label{tab:mip_vs_shortcut}
    \end{wraptable}

    While being competitive with flow models performance-wise, \Minimaliterativepolicy takes less integration steps (number of function evaluations (NFEs) = 2) compared to flow models (NFEs = 9).
    To further validate the computation efficiency of \Minimaliterativepolicy, we compare it with consistency models which accelerate the sampling process of flow by distilling the learned flow into a shortcut model~\citep{songConsistencyModels2023,boffiFlowMapMatching2025,fransOneStepDiffusion2024,gengMeanFlowsOnestep2025}.
    We benchmark \Minimaliterativepolicy against consistency trajectory model (CTM)~\citep{kim2023consistency}, where latter is trained in two-stage manner.
    Thus, CTM requires twice as many training time compared to \Minimaliterativepolicy.
    As shown in \Cref{tab:mip_vs_shortcut}, \Minimaliterativepolicy matches, and often outperforms CTM on most challenging tasks since CTM exhibits certain level of performance degradation compared to the teacher flow models.
    This again highlights that the fact that distributional learning is not necessary condition for GCPs performance and bypassing it offers computation efficiency at training and inference time.
    We further compare \Minimaliterativepolicy with other few-step methods like Lagrangian map distillation (LMD)~\citep{boffiFlowMapMatching2025} and present full results in \Cref{sec:comp_consistency_models}.
}{
}

\section{Inductive Bias, not Expressivity, Explains \Minimaliterativepolicy's Performance}
\label{sec:whatworks}

\subsection{Manifold adherence, not reconstruction, drives performance}
\label{sec:manifold_adherence}

\iftoggle{iclr}{
    \renewcommand{\arraystretch}{0.9}
    \begin{wrapfigure}{r}{0.5\textwidth}
        \setlength{\tabcolsep}{3pt}
        \centering
        \scriptsize
        \vspace{-0.4cm}
        \begin{tabular}{@{}l c c c c c@{}}
            \toprule
            \textbf{Metric}    & \textbf{\Regression} & \textbf{\Straightflow} & \textbf{\Residualregression} & \textbf{\Minimaliterativepolicy} & \textbf{\Flow} \\
            \midrule
            Off-manifold $L_2$ & $0.067$              & $0.063$                & $0.062$                      & $0.054$                          & $0.042$        \\
            Validation $L_2$   & $0.290$              & $0.234$                & $0.224$                      & $0.195$                          & $0.217$        \\
            \bottomrule
        \end{tabular}
        \captionof{table}{\footnotesize \textbf{Comparison of different methods on manifold adherence and reconstruction error. }
            Results are averaged across 3 different architectures and 32 states on state-based \toolhang with deterministic dataset.
        }
        \vspace{-0.5cm}
        \label{tab:off_manifold}
    \end{wrapfigure}
}{
    \renewcommand{\arraystretch}{0.9}
    \begin{wrapfigure}{r}{0.6\textwidth}
        \setlength{\tabcolsep}{3pt}
        \centering
        \scriptsize
        \vspace{-0.0cm}
        \begin{tabular}{@{}l l c c c c c@{}}
            \toprule
            \textbf{Dataset}               & \textbf{Metric}    & \textbf{\Regression} & \textbf{\Straightflow} & \textbf{\Residualregression} & \textbf{\Minimaliterativepolicy} & \textbf{\Flow} \\
            \midrule
            \multirow{4}{*}{Original}      & Off-manifold $L_2$ & $0.058$              & $0.061$                & $0.057$                      & $0.043$                          & $0.032$        \\
                                           & Off-manifold $L_1$ & $0.072$              & $0.073$                & $0.071$                      & $0.057$                          & $0.046$        \\
                                           & Validation $L_2$   & $0.073$              & $0.071$                & $0.062$                      & $0.069$                          & $0.074$        \\
                                           & Validation $L_1$   & $0.110$              & $0.106$                & $0.124$                      & $0.104$                          & $0.116$        \\
            \midrule
            \multirow{4}{*}{Deterministic} & Off-manifold $L_2$ & $0.067$              & $0.063$                & $0.062$                      & $0.054$                          & $0.042$        \\
                                           & Off-manifold $L_1$ & $0.082$              & $0.078$                & $0.077$                      & $0.063$                          & $0.051$        \\
                                           & Validation $L_2$   & $0.290$              & $0.234$                & $0.224$                      & $0.195$                          & $0.217$        \\
                                           & Validation $L_1$   & $0.336$              & $0.374$                & $0.386$                      & $0.331$                          & $0.356$        \\
            \bottomrule
        \end{tabular}
        \captionof{table}{\footnotesize \textbf{Comparison of different methods on manifold adherence and reconstruction error. }
            Results are averaged across 3 different architectures and 32 states on state-based \toolhang.
            Validation $L_2$/$L_1$ norm is evaluated on validation set from expert trajectories.
            Off-manifold $L_2$/$L_1$ norm is evaluated on out-of-distribution states.
        }
        \label{tab:off_manifold}
    \end{wrapfigure}
}

\Minimaliterativepolicy, and the absence of multimodality, suggest a better ability to approximate the expert more accurately on training data.
We test this by evaluating the $L_2$-error, i.e., reconstruction error, on validation set.
Surprisingly, we find that \Minimaliterativepolicy, \Flow, and RCP exhibit the \emph{same} validation loss; hence validation loss does predict their relative performance.
\Cref{sec:validation_loss_is_not_a_good_proxy_for_policy_performance} reveals that validation loss doesn't correlate with performance across other axes of variation. Indeed, policy performance requires taking good actions on \emph{o.o.d. states} under compounding error at deployment time \citep{simchowitz2025pitfalls}.

\iftoggle{arxiv}{
    \begin{wrapfigure}{r}{0.3\textwidth}
        \centering
        \includegraphics[width=1.0\linewidth]{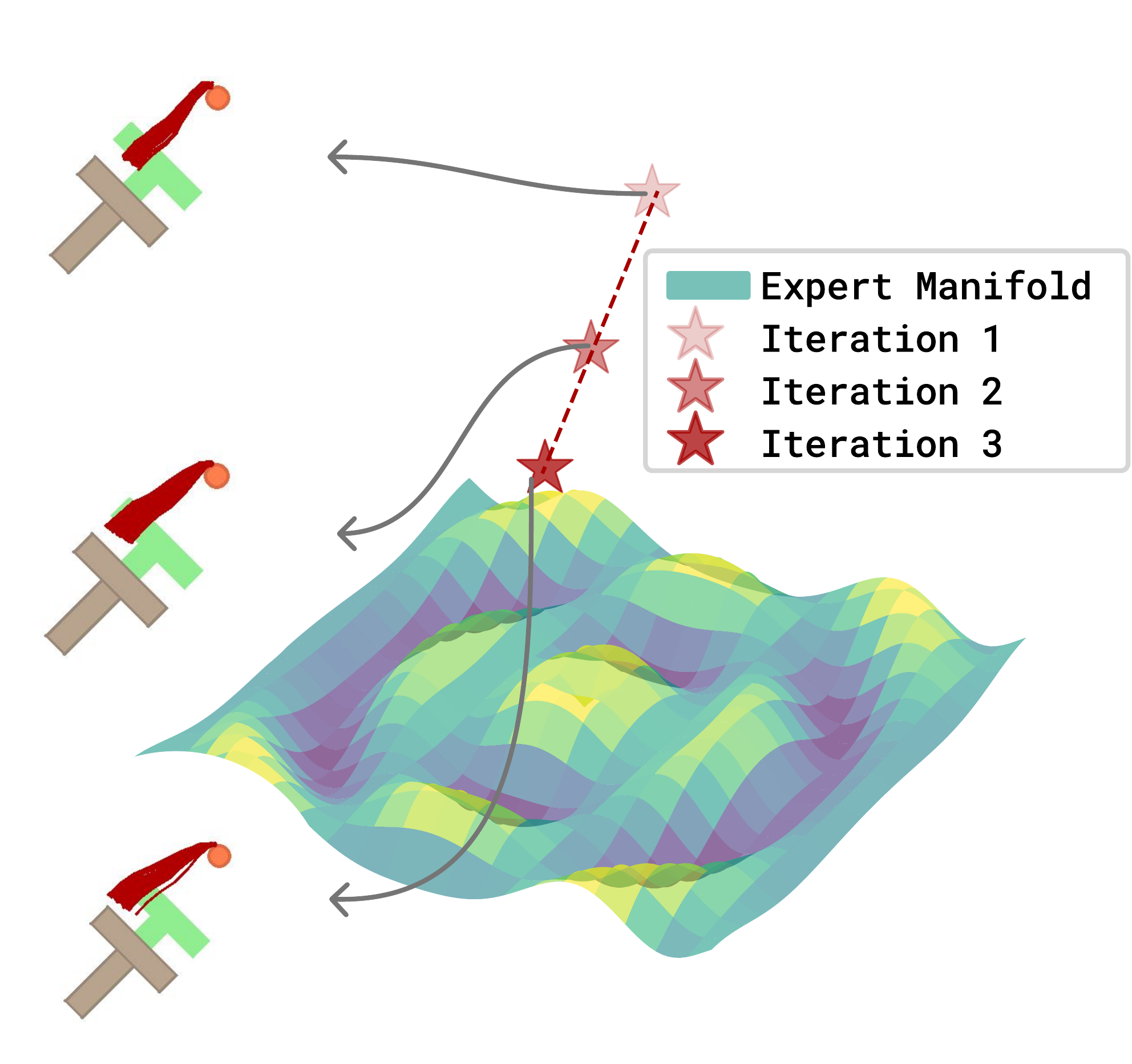}
        \captionof{figure}{
            \footnotesize \textbf{Manifold adherence illustration.}
            Sampled trajectories in \pusht tasks from flow model with different NFEs.
        }
        \label{fig:manifold_adherence}
    \end{wrapfigure}
}{}

Thus, we study a proxy which reflects performance in o.o.d. situations. We perturb expert trajectories in dataset as described in \Cref{sec:lipchitz_evaluation_method}, and evalute a novel metric that we call the \emph{off-manifold norm}. Informally, this measures the projection error of a predicted action $a$ onto the space spanned by expert actions at neighboring states; see~\cref{sec:manifold_adherence_evaluation_details} a for formal definition. Our metric assesses the quality of actions under simulated compounding error.
\Cref{tab:off_manifold} reports both $L_2$ validation loss and off-manifold $L_2$ norm for different methods: while all methods achieve low validation loss, only \Minimaliterativepolicy and \Flow are able to achieve low off-manifold $L_2$ norm, indicating their better manifold adherence. As \Straightflow does not exhibit the same benefit,
we conclude that supervised iterative computation facilitates projection onto the manifold of expert actions by refining the prediction across sequential steps.
\iftoggle{arxiv}{
    \Cref{fig:manifold_adherence} provides additional illustraion of manifold hypothesis: with more iterations, flow model samples more plausible trajectories, which goes to the side of T-shape object rather than colliding right into it.
}{}
\Cref{app:toy} provides additional confirmation of this hypothesis on comprehensive toy experiments:  GCPs are no better than RCP at fitting high frequency functions, but exhibit lower on-manifold error, suitably defined.

\iclrpar{Why manifold adherence matters for control.} We conjecture that, for high-precision tasks, the sensitivity to errors is not homogeneous across error directions in action space. Our findings present preliminary evidence that some form an ``on-manifold inductive bias'' directly aligns with minimizing error along relevant directions, yet is permissive to error in directions of lesser consequence. We think that rigorously establishing this hypothesis is an exicting direction for future work.

\iclrpar{No known mechanism accounts for greater manifold adherence in GCPs vs. RCPs.} There is a growing body of literature that shows that, if training data are supported on a given low dimensional manifold $\cM$, then generative models learn to project onto $\cM$~\citep{boffi2024shallowdiffusionnetworksprovably,permenter2024interpretingimprovingdiffusionmodels}. However, to our knowledge, there is no work that explains why this inductive bias would be \emph{stronger} than what would be achieved with a well-trained regression model. Specifically,  if   $o \mid a$ lies in some (local) manifold, regression too should learn to project onto it.

One might conjecture that the iterative computation provides many changes to predict an action that ``stick'' to the action manifold. However, such a mechanism would require that  once an on-manifold action is predicted, subsequent predictions do not nudge the prediction off-manifold.
In \Cref{app:reg}, we show that
simple arguments based on implicit regularization in linear models do not suffice to explain this hypothesis, at least for \Minimaliterativepolicy. Much like the usefulness of manifold adherence for control described above, the mechanism behind manifold adherence remains a mystery for future study.

\subsection{Stochasticity stabilizes iterative computation}
\label{sec:stochasticity_injection}

\iftoggle{arxiv}{
    \begin{wrapfigure}{r}{0.3\textwidth}
        \centering
        \includegraphics[width=1.0\linewidth]{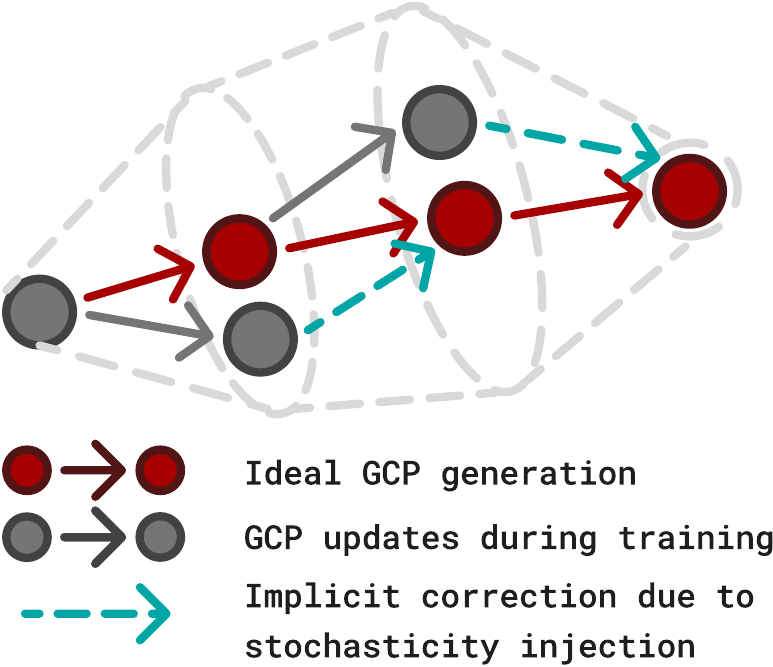}
        \captionof{figure}{
            \footnotesize \textbf{Stochasticity Stabilizes Iteration.}
            Noise injection broadens the generation path into a ``tube.''
            This creates provides supervision when the imperfectly trained GCP goes off-distribution , ensuring robust iterative computation.
        }
        \label{fig:stochasticity_injection}
    \end{wrapfigure}
}{
}

We recall from \Cref{fig:mip_vs_flow_vs_regression} that \Straightflow matches regression, whilst  \Residualregression under-performs regression. This suggests that sequential action generation is highly brittle in the absence of stochasticity
\citep{permenter2024interpretingimprovingdiffusionmodels}.
Our findings support the hypothesis that stochasticity injection serves to provide ``coverage'' of the generative process\iftoggle{arxiv}{ as illustrated in \Cref{fig:stochasticity_injection}}{}.
Note that this is different from task MDP-level augmentation like image augmentation or exploratory data collection since the augmentation happens in iterative generative process.
Specifically, we can think of learning to perform two-stage action generation as an ``internal'' behavior cloning problem \citep{ren2024diffusion} under the dynamics induced by the generative process.
Injecting stochasticity amounts to enhancing coverage of the action $\ahat_0$ in the first step of \Minimaliterativepolicy, thus enable iterative improvement with more NFEs (\Cref{sec:method_nfe_comparison}).
Its benefits are analogous to trajectory noising effective in other behavior cloning applications \citep{laskey2017dart,block2023butterfly,block2024provable,simchowitz2025pitfalls,zhang2025actionchunkingexploratorydata}.
Similar benefits are found in the improved sensitivity analysis of diffusion relative to flows \citep{albergo2024stochastic}.

\subsection{Architecture remains essential for  scaling}
\label{sec:architecture}
\begin{figure}[!ht]
    \centering
    \includegraphics[width=1.0\textwidth]{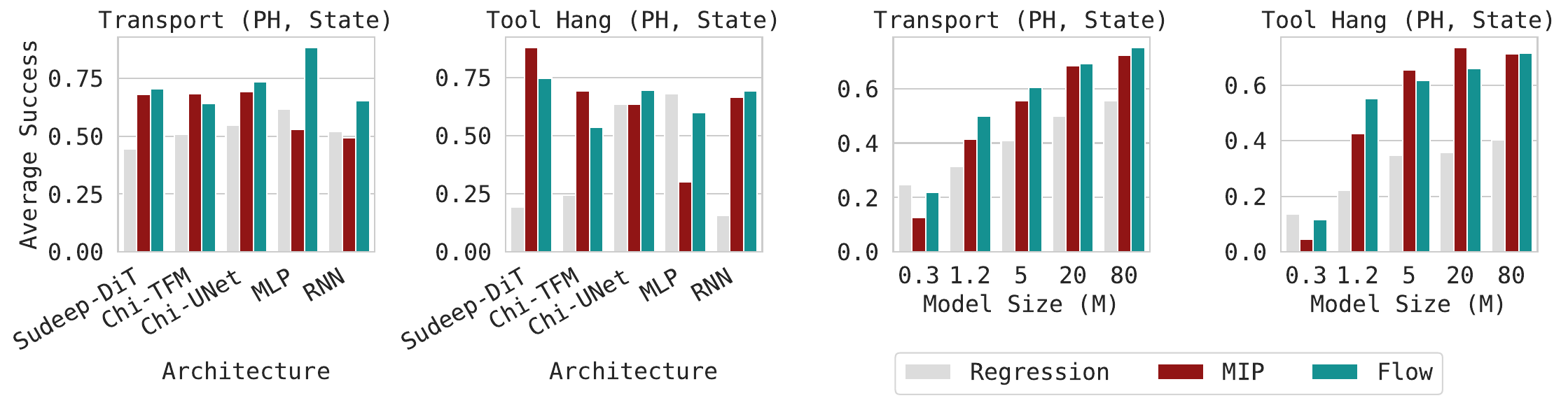}
    \vspace{-0.7cm}
    \caption{\footnotesize \textbf{Architecture and model size ablation.} Success rate are averaged across 3 seeds and 5 checkpoints on \toolhang and \transport tasks. Left 2 plots: architecture ablation. Right 2 plots: Model size ablation. \iftoggle{iclr}{}{While all methods performance scales with model size, regression can outperform flow and \Minimaliterativepolicy with smaller capacity.}\iftoggle{iclr}{}{, highlighting the importance of aligning the model capacity when comparing different methods.}}
    \label{fig:architecture_and_model_size}
    \vspace{-0.5cm}
\end{figure}
\iftoggle{iclr}{Regression}{While all methods do scale, regression},
enjoys stronger relative performance at the smallest model sizes but scales more poorly than flow and \Minimaliterativepolicy with increased model capacity (\cref{fig:architecture_and_model_size}).
We conjecture that supervised iterative computation can better utilize larger models, both by introducing more supervision steps at training, and by providing more parameters to represent different computations at successive generation steps.
Nevertheless, \emph{architecture design} plays an incredibly significant role.
To showcase its importance, we ablate the performance of different method's average performance across both the 3 architectures above, and the more traditional \mlp and \rnn architectures, implemented with modern best practices including FiLM conditioning \citep{perez2018film}, and skip-connections \citep{he2016deep}/LayerNorm \citep{ba2016layer} where appropriate (details in \Cref{sec:control_architecture}).
As demonstrated in \cref{fig:architecture_and_model_size}, the combination of training method and architecture design has a strong yet somewhat erratic effect on both GCPs and RCP performance.
In \toolhang, RCP achieves the best performance with an \mlp architecture. In \transport, \mlp with flow can even outperform more expressive architectures like \chitransformer.
The coupling between training and architecture choice  highlights the importance of controlling architecture design when comparing across methods.

\newcommand{\atsign}{\makeatletter @ \makeatother }
\section{Related Work}
\label{sec:related}

\iclrpar{Robotic Behavior Cloning. }
Behavior cloning (BC), also known as learning from demonstrations (LfD), has become a popular paradigm to enable robots to conduct complex, diverse and long-horizon manipulation tasks by learning from expert demonstrations~\citep{argallSurveyRobotLearning2009,zhuRobotLearningDemonstration2018,zhaoLearningFineGrainedBimanual2023,chiUniversalManipulationInterface2024,lin2024data}.
In parallel, ``robot foundation models'' scale BC with internet-pretrained vision-language transformer-based backbones~\citep{brohan2022rt,zitkovich2023rt,o2024open} and large-scale teleoperation datasets~\citep{kim2024openvla,teamOctoOpenSourceGeneralist2024}.
More recently, to better model continuous actions, generative models like diffusion and flow have been adopted to replace the tokenization method in transformers to achieve more expressive policies~\citep{nvidia2025gr00t,black2024pi_0,intelligence$p_05$VisionLanguageActionModel2025,liu2024rdt}.
This work focuses on the generative modeling part of the behavior cloning pipeline, ablating the key design choices that lead to the success of generative control policies.

\iclrpar{Generative Modeling.}
The recent success of behavior cloning policies is built upon a rapid evolution of generative modeling techniques, starting from tokenization methods~\citep{brownLanguageModelsAre2020,chenDecisionTransformerReinforcement2021,pertschFASTEfficientAction2025} and adversarial methods~\citep{brockLargeScaleGAN2019,goodfellow2020generative,ho2016generative}.
Later, probabilistic generative models with iterative computation like diffusion models~\citep{ho2020denoising,songScoreBasedGenerativeModeling2021,lu2025dpm,songDenoisingDiffusionImplicit2022,nichol2021improved,karras2022elucidating} became a popular choice for generative modeling thanks to their better training stability and sampling quality.
Flow models~\citep{lipman2023flow,albergo2022building,liu2022flow} and consistency/shortcut models~\citep{songConsistencyModels2023,songImprovedTechniquesTraining2023,meng2023distillation,boffiFlowMapMatching2025,gengMeanFlowsOnestep2025} were later developed to achieve faster sampling while maintaining the expressivity of diffusion models.
Though there have been extensive studies on probabilistic generative modeling's effectiveness in image and text generation~\citep{lee2023convergence,chenSamplingEasyLearning2023}, its mechanism in control, especially the key design choices, are still opaque in decision making.

\iclrpar{Generative Control Policies.}
To model diverse and complex behaviors, GCPs parameterize the relationship between observations and actions as a distribution rather than a deterministic function.
Early works use transformers with tokenizers~\citep{chenDecisionTransformerReinforcement2021,shafiullah2022behavior}, energy functions~\citep{florence2022implicit,dasariIngredientsRoboticDiffusion2024} and VAEs~\citep{zhaoLearningFineGrainedBimanual2023} to parameterize the distribution.
Diffusion models~\citep{reuss2023goal,chi2023diffusion,ke20243d,dongCleanDiffuserEasytouseModularized2024,janner2022planning,yangEquiBotSIM3EquivariantDiffusion2024} were introduced for their better expressivity of complex and multi-modal behaviors, followed by flow-based~\citep{zhang2024flowpolicy,black2024pi_0,intelligence$p_05$VisionLanguageActionModel2025} and flow-map/consistency-model/shortcut-model-based acceleration methods~\citep{hu2024adaflow,prasadConsistencyPolicyAccelerated2024,shengMP1MeanFlowTames2025}.

\iclrpar{Theoretical Literature on GCPs.} \citet{block2024provable} established that GCPs can imitate arbitrary expert distributions. Given our findings on the absence of multi-modality, a more closely related theoretical findings is that of \citet{simchowitz2025pitfalls}, which elucidates how GCPs can circumvent certain worst-case compounding error phenomena in continuous-control imitation learning.
Though the proposed mechanism is different, that finding is conceptually similar to our own: GCPs benefits arise from their favorable out-of-distribution properties, rather than raw expressivity of fitting in-distribution expert behavior.

\subsection{Previous Works' Connection with GCP's Taxonomy.}
\label{sec:previous_works_connection}

We classify GCPs into three components: distributional learning, stochasticity injection, and supervised iterative computation.
Starting from regression, it has none of the three components.
To model a more complex distribution, Gaussian Mixture Model (GMM)~\citep{zhuRobotLearningDemonstration2018} was used to parameterize the distribution, trained with cross entropy loss.
To make the network be able to represent more complex distirbutions, prior to diffusion, non-parametric method like VAEs~\citep{zhaoLearningFineGrainedBimanual2023} was used to parameterize the distribution, trained with reconstruction loss.
During the training, a latent variables is predicted to predict the style the motion by mapping it from a noise $z$.
Another line of work try to improve the policy expressivity by introducing iterative compute, like implicit behavior cloning~\citep{florence2022implicit,dasariIngredientsRoboticDiffusion2024} and behavior transformer~\citep{shafiullah2022behavior}.
In IBC, the idea is to allow the network predict the energy function of the action rather the action itself.
Compared to diffusion, the major difference is that they do not explicitly injecting noise during training and no intermediate supervision is provided for the intermediate results.
Similarly, in behavior transformer, a two step policy is introduced to first predict the policy class and then refine it with another network to achieve higher precision control. 
Lastly, flow-based GCPs~\citep{zhang2024flowpolicy,black2024pi_0,intelligence$p_05$VisionLanguageActionModel2025}, which holds all the three components and demonstrate state-of-the-art performance on popular benchmarks.
In this paper, we look into a new combination that haven't been explored before, which is the combination of stochasticity injection and supervised iterative computation.

\section{Discussion}

Our comprehensive evaluation reveals a fundamental divergence between the objectives of generative modeling in vision or text and those in robotic control. We demonstrate that for control, fitting the exact data distribution (\componentref{comp:distr}) is secondary; rather, the inductive bias of manifold adherence—facilitated by stochastic iterative computation (\componentref{comp:stoch}+\componentref{comp:sic})—is paramount. This insight not only demystifies the success of GCPs but also enables the design of streamlined architectures like \Minimaliterativepolicy.

\iclrpar{Theoretical Gaps.} While we empirically identify manifold adherence as a proxy for closed-loop performance, a theoretical framework explaining why stochastic supervision with MSE loss induces this behavior remains elusive. Developing this theoretical grounding is a critical next step to replace exhaustive empirical benchmarking with principled policy design.

\iclrpar{Broader Applications.} Finally, our analysis focuses on behavior cloning. It remains an open question whether the benefits of the \componentref{comp:stoch}+\componentref{comp:sic} paradigm persist in other settings, such as RL-finetuning, large-scale pretraining, or long-horizon planning. Future work should explore whether the "myths" of generative control hold true in these broader domains.

\section*{Acknowledgements}
MS and GA acknowledge a TRI University 2.0 Fellow and Google Robotics Research Award. MS and CP thank Nur Muhummad (Mahi) Shuffiulah for his insightful feedback, and thank MS also thanks Aviral Kumar, Sarvesh Patil, and Andrej Risteski for their thoughtful suggestions. GS holds concurrent appointments as an Assistant Professor at Carnegie Mellon University and as an Amazon Scholar. This paper describes work performed at Carnegie Mellon University and is not associated with Amazon.

{
  \bibliographystyle{iclr2026_conference}
  \bibliography{refs}
}
\newpage

\tableofcontents
\newpage
\appendix
\iftoggle{iclr}{
    \section{Learning Dynamics of GCPs and RCPs Given Different Types of Data}
    \label{app:different_learning_dynamics}
    \begin{figure}[h]
        \centering
        \includegraphics[width=1.0\linewidth]{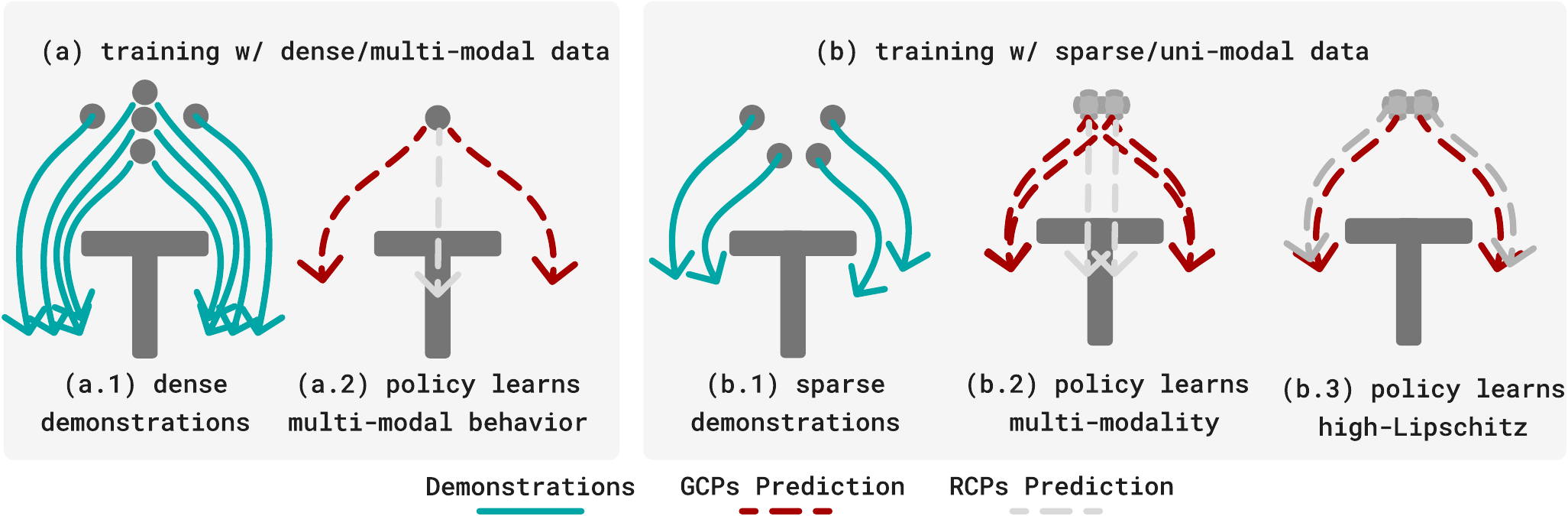}
        \caption{\textbf{GCP behavior given different types of data.}
            (a) Given true multi-modal data (a.1), GCPs learn both modes while RCPs collapse (a.2).
            (b) Given sparse data (b.1) in high-dimensional space, both policies can learn high-Lipschitz policies to quickly switch between modes (b.3). We find both \textbf{GCPs and RCPs} learn (b.3) in high-dimensional tasks (\Cref{sec:multimodality}). \Cref{thm:informal} suggests that GCPs have a limited advantage over RCPs from a pure expressivity perspective.
        }
        \label{fig:multimodality_vs_lipschitz}
    \end{figure}
}{}

\section{Additional Policy Parametrizations}
\label{sec:minimum_iterative_policy_design_ablation}

This section further elaborates the design space of \Minimaliterativepolicy in stochasticity injection, iterative computation and intermediate supervision.

\subsection{Full Abalation of \Minimaliterativepolicy Variants}

This section formally describes the training process of all \Minimaliterativepolicy with different stochasticity injection and supervised iterative computation design.

\paragraph{Residual Regression (\Residualregression)} removes all stochasticity in training and the training objective is:

\begin{align}
     & \pi^{\Residualregression}_{\theta} \approx \iftoggle{iclr}{\textstyle}{} \argmin_\theta \Exp_{(o,a) \sim \Dtrain, z \sim \Normal(0,\eye)} \\\iftoggle{iclr}{\textstyle}{} &\left(\|(\pi_{\theta}(o,I_0 = 0,t=0) -  \tfix a)\|^2 + \|(\pi_{\theta}(o, \stopgrad(\pi_{\theta}(o,I_0 = 0,t=0)), \tfix) - a)\|^2\right).
    \label{eq:residual_regression}
\end{align}

\paragraph{Two-Step Denoising (\algname{TSD})} The training objective is:

\begin{align*}
     & \pi^{\algname{TSD}}_{\theta} \approx \iftoggle{iclr}{\textstyle}{} \argmin_\theta \Exp_{(o,a) \sim \Dtrain, z \sim \Normal(0,\eye)} \\\iftoggle{iclr}{\textstyle}{} &\left(\|(\pi_{\theta}(o,I_0,t=0) -  \tfix a)\|^2 + \|(\pi_{\theta}(o, \stopgrad(\pi_{\theta}(o,I_0,t=0)) + (1-\tfix)z, \tfix) - a)\|^2\right).
\end{align*}

where $I_0 = z$. Compared to \Minimaliterativepolicy, \algname{TSD} adds stochasticity to both first step training.

\paragraph{\Minimaliterativepolicy with Data Augmentation (\algname{MIP-Dagger})}
To understand the importance of decoupling for enabling iterative computation, we propose an additional variant of \Minimaliterativepolicy that lies between \Minimaliterativepolicy and \Residualregression, where the two steps are partially coupled. Since the training method of second iteration is similar to data augmentation, we call this variant \algname{MIP-Dagger}:

\begin{align*}
     & \pi^{\algname{MIP-Dagger}}_{\theta} \approx                           %
    \argmin_\theta \underset{(o,a) \sim \Dtrain, z \sim \Normal(0,\eye)}{\E} \\ & (\|(\pi_{\theta}(o,I_0 = 0,t=0) -  \tfix a)\|^2 + \|(\pi_{\theta}(o, \tfix \stopgrad(\pi_{\theta}(o,I_0 = 0,t=0)) + (1-\tfix)z, \tfix) - a)\|^2),
\end{align*}

where the major difference compared to \Minimaliterativepolicy is the second step takes in the interpolant between first step output and noise rather than the action and noise.

\paragraph{\Minimaliterativepolicy without intermediate supervision (\algname{MIP-NoSupervision})}
To understand the effect of intermediate supervision on iterative computation, we propose one variant of \Minimaliterativepolicy that removes the supervision of intermediate computation steps while preserving stochasticity injection at training time, named \algname{MIP-NoSupervision}:

\begin{align*}
     & \pi^{\algname{MIP-NoSupervision}}_{\theta} \approx                                               %
    \argmin_\theta \underset{(o,a) \sim \Dtrain, z \sim \Normal(0,\eye)}{\E}                            \\
     & (\|(\pi_{\theta}(o, \tfix \stopgrad(\pi_{\theta}(o,I_0 = 0,t=0)) + (1-\tfix)z, \tfix) - a)\|^2),
\end{align*}

where the first step's output is unsupervised.

\paragraph{\Minimaliterativepolicy without $t$ conditioning} By removing $t$ conditioning in \Minimaliterativepolicy, it degenerates to \Straightflow. Here we present the multi-step integration process for straight flow when action distribution is Dirac delta. The integrator from $s$ to $t$ is:

\begin{align*}
    a_t = \frac{t - s}{1 - s} \pi_{\theta}(o, s \cdot a_s) + \frac{1 - t}{1 - s} a_s
\end{align*}

\subsection{Additional Noise Injection Methods}
\label{sec:additional_noise_injection_methods}

While \Minimaliterativepolicy only injects noise to action, we also explore the possibility of injecting noise to observation.
We propose two variants of \Minimaliterativepolicy: \algname{MIP-Obs} and \algname{MIP-Dagger-Obs}.
In \algname{MIP-Obs}, we perturb the first step's observation with noise $z$, while the second step's training is the same as the original \Minimaliterativepolicy with decoupled training.
In \algname{MIP-Dagger-Obs}, we perturb the first step's observation with noise $z$, and the second step's training is conditioned on the first step's output, making it similar to Dagger.
Major differnce compared to the original MIP: perturb the first step's observation.
In both variants, we fixed $t_* = 0.9$ and all observation perturbation happens at observation embedding space with normalized features.

\begin{align*}
    \tiny
    \pi^{\textsc{MIP-Obs}}_{\theta}        & \approx \argmin_\theta \underset{\substack{(o + (1-t_*)z,a) \sim p_{\text{train}}                                                               \\ z \sim \mathcal{N}(0,I)}}{\mathbb{E}} \Bigg[ \|(\pi_{\theta}(o + \textcolor{primalcolor}{(1-t_*)z},I_0 = 0,t=0) -  t_* a)\|^2                                                                                                   \\
                                           & \quad + \|(\pi_{\theta}(o, I_{t_*}, t_*) - a)\|^2 \Bigg]                                                                                        \\
    \pi^{\textsc{MIP-Dagger-Obs}}_{\theta} & \approx \argmin_\theta \underset{\substack{(o,a) \sim p_{\text{train}}                                                                          \\ z \sim \mathcal{N}(0,I)}}{\mathbb{E}} \Bigg[ \|(\pi_{\theta}(o+\textcolor{primalcolor}{{(1-t_*)z}},I_0 = 0,t=0) -  t_* a)\|^2                                                                                                   \\
                                           & \quad + \|(\pi_{\theta}(o, t_* \text{stopgrad}(\pi_{\theta}(o+\textcolor{primalcolor}{{(1-t_*)z}},I_0 = 0,t=0)) + 1-t_*)z, t_*) - a)\|^2 \Bigg]
\end{align*}

We find that perturbing observations introduces data conflicts and degrades performance (\cref{tab:observation_perturbation_comparison}). In a two-step model, selecting noise levels that prevent observation overlap becomes challenging and brittle, leading to training instability across architectures.

\begin{table}[h]
    \centering
    \begin{tabular}{ll|cc}
        \toprule
        Architecture    & Method (L2)              & \transport (ph) & \toolhang (ph) \\
        \midrule
        \chitransformer & \Regression              & 0.50/0.45       & 0.50/0.37      \\
        \chitransformer & \Minimaliterativepolicy  & 0.79/0.69       & 0.92/0.85      \\
        \chitransformer & \algname{MIP-Dagger-Obs} & 0.00/0.00       & 0.00/0.00      \\
        \chitransformer & \algname{MIP-Obs}        & 0.61/0.46       & 0.13/0.08      \\
        \chitransformer & \Flow                    & 0.81/0.71       & 0.89/0.75      \\
        \midrule
        \sudeepdit      & \Regression              & 0.65/0.54       & 0.31/0.25      \\
        \sudeepdit      & \Minimaliterativepolicy  & 0.80/0.69       & 0.80/0.72      \\
        \sudeepdit      & \algname{MIP-Dagger-Obs} & 0.00/0.00       & 0.00/0.00      \\
        \sudeepdit      & \algname{MIP-Obs}        & 0.00/0.00       & 0.00/0.00      \\
        \sudeepdit      & \Flow                    & 0.79/0.65       & 0.73/0.61      \\
        \midrule
        \chiunet        & \Regression              & 0.66/0.59       & 0.73/0.59      \\
        \chiunet        & \Minimaliterativepolicy  & 0.81/0.72       & 0.82/0.71      \\
        \chiunet        & \algname{MIP-Dagger-Obs} & 0.00/0.00       & 0.00/0.00      \\
        \chiunet        & \algname{MIP-Obs}        & 0.00/0.00       & 0.00/0.00      \\
        \chiunet        & \Flow                    & 0.83/0.75       & 0.87/0.73      \\
        \bottomrule
    \end{tabular}
    \caption{Performance comparison of different methods with observation perturbation on state-based tasks. For each methods and architecture, we report the average and best performance across 5 checkpoints with 3 random seeds.}
    \label{tab:observation_perturbation_comparison}
\end{table}

\subsection{Experiment Results}
\label{sec:mip_variants_results}

We benchmark all methods on the \toolhang task, given it is the one with the largest gap between RCP and GCPs.
From~\cref{tab:mip_variants}, we can see that the important part is to add stochasticity injection between two iterations, and intermediate supervision is also important to realize the potential of iterative computation.

\begin{table}[ht]
    \centering
    \begin{tabular}{@{}l c c@{}}
        \toprule
        \textbf{Method}             & \textbf{NFEs} & \textbf{Success Rate} \\
        \midrule
        \algname{TSD}               & 2             & 0.80                  \\
        \algname{MIP}               & 2             & 0.80                  \\
        \algname{MIP-NoSupervision} & 2             & 0.42                  \\
        \algname{MIP-Dagger}        & 2             & 0.64                  \\
        \algname{RR}                & 2             & 0.54                  \\
        \Straightflow               & 1             & 0.54                  \\
        \Straightflow               & 3             & 0.55                  \\
        \Straightflow               & 9             & 0.52                  \\
        \bottomrule
    \end{tabular}
    \caption{Success rates across different \Minimaliterativepolicy variants and \Residualregression on \toolhang task over 5 checkpoints across 3 architectures.}
    \label{tab:mip_variants}
\end{table}

\section{Control Experiments}
\label{sec:control_settings}

\subsection{Task Settings}
\label{sec:task_settings}

This section introduces all the tasks presented in the main paper.
To reach a sound conclusion, use common benchmarks appears in previous works:

\paragraph{\taskname{Robomimic}}
Robomimic~\citep{robomimic2021} is a large-scale robotic manipulation benchmark designed to study imitation learning and offline reinforcement learning.
It contains five manipulation tasks (\lifttask, \can, \squaretask, \transport, \toolhang) with \emph{proficient human (PH)} teleoperated demonstrations, and for four of them, additional \emph{mixed proficient/non-proficient human (MH)} demonstration datasets are provided (9 variants in total).
We report results on both \emph{state-based} and \emph{image-based} observations, since these two modalities pose distinct challenges.
Among the tasks, \toolhang requires extremely precise end-effector positioning and fine-grained contact control, while \transport demands high-dimensional control and coordination over extended horizons.

\paragraph{\pusht}
\pusht~\citep{florence2022implicit} is adapted from the Implicit Behavior Cloning (IBC). The task involves pushing a T-shaped block to a fixed target location using a circular end-effector. Randomized initializations of both the block and the end-effector introduce significant variability. The task is contact-rich and requires modeling complex object dynamics for precise block placement. Two observation variants are considered: (\emph{i}) raw RGB image observations and (\emph{ii}) state-based observations containing object pose and end-effector position.

\paragraph{\kitchen}
The Franka \kitchen environment is designed to test the ability of IL and offline RL methods to perform long-horizon, multi-task manipulation.
It includes 7 interactive objects, with human demonstration data consisting of 566 sequences, each completing 4 sub-tasks in arbitrary order (e.g., opening a cabinet, turning a knob).
Success is measured by completing as many of the demonstrated sub-tasks as possible, regardless of order. This setup explicitly introduces both short-horizon and long-horizon multimodality, requiring policies to generalize across compositional tasks.

\paragraph{\metaworld}
\metaworld is a large-scale suite of diverse manipulation tasks built in MuJoCo, where agents must perform challenging object interactions using a robotic gripper. We adopt the 3D observation setting using point cloud representations, ported from the DP3 framework~\citep{ze3DDiffusionPolicy2024}, to better evaluate geometric reasoning and spatial generalization. Tasks in MetaWorld are categorized into different difficulty levels, with benchmarks testing few-shot adaptation and multi-task transfer learning.

\paragraph{\adroit}
\adroit is a suite of dexterous manipulation tasks featuring a 24-DoF anthropomorphic robotic hand.
Tasks include pen rotation, door opening, and object relocation, all of which demand precise, coordinated multi-finger control. Following DP3~\citep{ze3DDiffusionPolicy2024}, we use point cloud observations to capture fine-grained 3D object-hand interactions. Policies are trained using VRL3, highlighting the challenges of high-dimensional control and sim-to-real transfer in dexterous manipulation.

\paragraph{\libero}
\libero is a common multi-task benchmark to evaluate VLA's generalization ability. It is composed of 130 tasks and can be categorized into multiple categories, including object, goal, spatial, and 10-task. The 10-task is long horizon and considered the most challenging to solve.

\subsection{Architecture Design}
\label{sec:control_architecture}

We study four policy backbones—\chitransformer, \sudeepdit, \chiunet, \rnn, and \mlp—under a common training recipe and data interface. Unless otherwise specified, \emph{all models are capacity-matched to $\sim$20M parameters} to enable fair comparison.

\paragraph{\chiunet} is adopted from Diffusion Policy~\citep{chi2023diffusion} which built on top of 1D temporal U-Net~\citep{janner2022planning} with FiLM conditioning~\citep{perez2018film} on observation $o$ and flow time $t$.
\chiunet has a strong inductive bias for the temporal structure of the action and tends to smooth out the action.

\paragraph{\chitransformer} follows the time–series diffusion transformer from Diffusion Policy~\citep{chi2023diffusion}, where the noisy action tokens $a_t$ form the input sequence and a \emph{positional embedding} of the flow time $t$ is prepended as the first token; observations $o$ are mapped by a shared MLP into an observation-embedding sequence that conditions the decoder stack.
Compared to \chiunet, \chitransformer\ uses token-wise self-attention over the whole action sequence, thus can model less-smooth and more complex actions.

\paragraph{\sudeepdit} is a DiT-style (Diffusion Transformer) conditional noise network specialized for policies adopted from DiT-Policy~\citep{dasariIngredientsRoboticDiffusion2024}: observation $o$ are first encoded into observation vectors;
the flow time $t$ is embedded via \emph{positional embedding}; an encoder–decoder transformer then fuses these with initial noise $z$ to predict next action.
The key ingredient of \sudeepdit is replacing standard cross-attention with \emph{adaLN-Zero} blocks—adaptive LayerNorm modulation using the mean encoder embedding and the time embedding, with zero-initialized output-scale projections—stabilizing diffusion training at scale.
Compared to \chitransformer, \sudeepdit has adaLN-based conditioning (instead of vanilla cross-attention) and an explicit encoder-decoder split, yielding better training stability.

\paragraph{\rnn}
The \rnn backbone processes sequences with a stacked LSTM/GRU.
For each action time step in the chunk, the input vector concatenates: the current noised action $a_t$, a time embedding for $t$, and a observation embedding for $o$. The RNN outputs are fed to a MLP head with LayerNorm$+$ApproxGELU$+$Dropout blocks before output the action with final linear head.
All linear and recurrent weights use \emph{orthogonal initialization} (biases zero), and RNN layer dropout is applied when depth$>$1.

\paragraph{\mlp}
The \mlp\ backbone flattens the action and observation, appending the time embedding.
Each mlp block has LayerNorm, ApproxGELU and Dropout blocks with residual connection and \emph{orthogonal} weight initialization throughout. Each block output is then modulated with FiLM conditioning.

\paragraph{\archname{DP3}} built on top of \chiunet with extra 3d perception encoder. We use the exact same architecture as 3D diffusion policy~\citep{ze3DDiffusionPolicy2024}.

\paragraph{Model hyperparameters}
In the main experiments, we align the model capacity to 20M parameters for default if not specified, with detailed hyperparameters report in \Cref{tab:model_hyperparameters}.

\begin{table}[ht]
    \centering
    \label{tab:backbones}
    \setlength{\tabcolsep}{6pt}
    \begin{tabular}{lcccc}
        \toprule
        Backbone        & Heads & Layers & Embedding dim & Dropout \\
        \midrule
        \sudeepdit      & 8     & 8      & 256           & 0.1     \\
        \chiunet        & --    & --     & 256           & --      \\
        \chitransformer & 4     & 8      & --            & 0.1     \\
        \rnn            & --    & 8      & 512           & 0.1     \\
        \mlp            & --    & 8      & 512           & --      \\
        \bottomrule
    \end{tabular}
    \caption{Model hyperparameters. }
    \label{tab:model_hyperparameters}
\end{table}

\subsection{Finetuning \pizero on \libero}
\label{sec:pizero_finetuning}

For \pizero finetuning experiments, we use \texttt{lerobot} framework~\citep{cadene2024lerobot} to finetune \pizero on \libero.
Our flow-based finetuning experiments match their reported results.
To finetune \pizero to regression policy, we use the same architecture but set the initial noise always to zero and let the model directly predict the action.
To finetune \pizero to \Minimaliterativepolicy, we use the same practice where we modify the time step $t$ to be uniformly sample from $\{0, t^*\}$ uniformly and set the initial noise to zero.
We train all policies until convergence with 50k gradient steps on 1 node with 8 H100 GPUs.

\subsection{Full Results for Flow and Regression Comparison}
\label{sec:gp_vs_rp_full_result}

In the paper, we only present the aggregated results across 3 architectures. \Cref{fig:rp_vs_gp_avg} present the full results across all architectures with different training methods.

\begin{figure}[ht]
    \centering
    \includegraphics[width=1.0\textwidth]{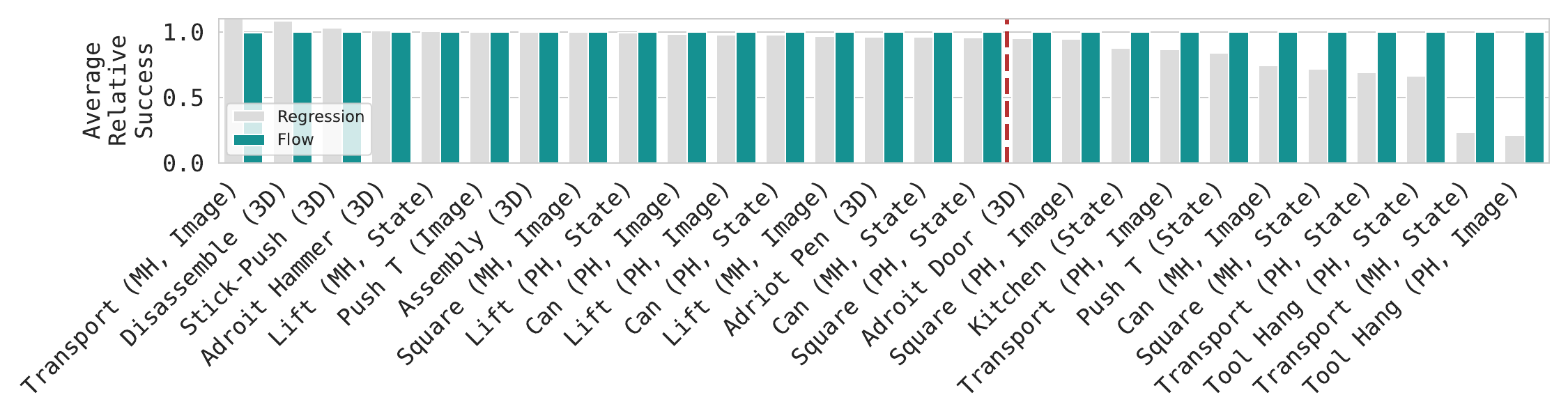}
    \caption{
        \textbf{Relative performance of RCP compared to GCP across common benchmarks (worst-case architecture).} For each task, we implement \chitransformer, \sudeepdit and \chiunet. For each architecture, we average performance of the last 5 training checkpoints across three seeds. We then report the performance of the worst-performing architecture, chosen individually for both RCP and GCP, to demonstrate method robustness. For \Flow, we always do 9 step Euler integrations, where its performance plateaued. For readability, RCP success rates are plotted relative to flow, with flow normalized to performance of $1$ per task. Tasks are grouped by observation modality, and ordered by relative RCP performance. Red dashed line indicates threshold at which RCP attains $<95\%$ success of GCP.
    }
    \label{fig:rp_vs_gp_avg}
\end{figure}

To further rule out the effect of training method, we also compare different methods' performance with $\ell_1$, which is observed to be superior for regression policy~\citep{kim2024openvla}.
We also benchmark the performance of flow model and \Minimaliterativepolicy with $\ell_1$ loss to understand the effect of loss function on the performance of GCPs.
\Cref{tab:ell1_ell2_comparison} shows the performance comparison of different methods with $\ell_1$ and $\ell_2$ loss, where we find that $\ell_1$ loss generally outperforms $\ell_2$ loss, especially for regression policy.
We attribute the superior performance of $\ell_1$ loss to the fact that it can capture the expert behavior better by learning the medium instead of the mean of the action.
However, even with $\ell_1$ loss, we still observe that Regression $<$ \Minimaliterativepolicy $\approx$ Flow, highlighting the importance of the stochasticity injection and iterative computation is independent of the loss function.
\begin{table}[h]
    \centering
    \begin{tabular}{ll|cc}
        \toprule
        Architecture    & Method                           & Transport (ph) & Tool Hang (ph) \\
        \midrule
        \sudeepdit      & \Regression $\ell_1$             & 0.72/0.64      & 0.76/0.65      \\
        \sudeepdit      & \Regression $\ell_2$             & 0.50/0.45      & 0.50/0.37      \\
        \sudeepdit      & \Minimaliterativepolicy $\ell_1$ & 0.81/0.73      & 0.91/0.84      \\
        \sudeepdit      & \Minimaliterativepolicy $\ell_2$ & 0.79/0.69      & 0.92/0.85      \\
        \sudeepdit      & \Flow $\ell_1$                   & 0.83/0.76      & 0.93/0.84      \\
        \sudeepdit      & \Flow $\ell_2$                   & 0.81/0.71      & 0.89/0.75      \\
        \midrule
        \chitransformer & \Regression $\ell_1$             & 0.67/0.57      & 0.44/0.33      \\
        \chitransformer & \Regression $\ell_2$             & 0.65/0.54      & 0.31/0.25      \\
        \chitransformer & \Minimaliterativepolicy $\ell_1$ & 0.80/0.68      & 0.85/0.77      \\
        \chitransformer & \Minimaliterativepolicy $\ell_2$ & 0.80/0.69      & 0.80/0.72      \\
        \chitransformer & \Flow $\ell_1$                   & 0.77/0.69      & 0.81/0.71      \\
        \chitransformer & \Flow $\ell_2$                   & 0.79/0.65      & 0.73/0.61      \\
        \midrule
        \chiunet        & \Regression $\ell_1$             & 0.84/0.71      & 0.71/0.55      \\
        \chiunet        & \Regression $\ell_2$             & 0.66/0.59      & 0.73/0.59      \\
        \chiunet        & \Minimaliterativepolicy $\ell_1$ & 0.85/0.68      & 0.76/0.67      \\
        \chiunet        & \Minimaliterativepolicy $\ell_2$ & 0.81/0.72      & 0.82/0.71      \\
        \chiunet        & \Flow $\ell_1$                   & 0.85/0.69      & 0.87/0.71      \\
        \chiunet        & \Flow $\ell_2$                   & 0.83/0.75      & 0.87/0.73      \\
        \bottomrule
    \end{tabular}
    \caption{Comparison of $\ell_1$ vs $\ell_2$ norm across different methods and architectures. Report average/best performance across 5 checkpoints with 3 random seeds.}
    \label{tab:ell1_ell2_comparison}
\end{table}

\iftoggle{iclr}{
    \begin{table}[h]
        \centering
        \small
        \begin{tabular}{l|cccc}
            \toprule
            \textbf{Method}             & \liberoobject & \liberogoal   & \liberospatial & \liberoten    \\
            \midrule
            \regression ($\ell_2$ loss) & 92.6          & 94.6          & 97.2           & 78.0          \\
            \regression ($\ell_1$ loss) & 95.2          & 88.0          & 95.8           & 62.4          \\
            \Flow                       & \textbf{97.4} & 95.0          & 95.8           & 81.6          \\
            \Minimaliterativepolicy     & 95.8          & \textbf{95.2} & \textbf{97.6}  & \textbf{82.2} \\
            \bottomrule
        \end{tabular}
        \caption{\textbf{Performance comparison on multi-task \libero benchmark.}
            We report the success rate of the checkpoint trained with 50k gradient steps of finetuning \pizero on the full \libero dataset. We implement \Minimaliterativepolicy with $t^* = 0.9$ and integrate flow with 10 steps.
            For regression, we train with both $\ell_2$ and $\ell_1$ loss as suggested in~\citep{kim2024openvla}.
        }
        \label{tab:mip_multitask}
    \end{table}
}
{
}

\subsection{Dataset Quality Ablation}

GCPs are believed to handle data with diverse quality better.
To test that assumption, we manually corrupt the expert dataset and inject stochactity and multi-modality in to the dataset.
In \cref{tab:performance_comparison}, we compare 4 different datasets (3 of them collected by ourselves).
In the collected dataset, we manually inject noise to the policy and add delay the policy from time to time to introduce multi-modality that is common in the real world.

\begin{table}[!ht]
    \centering
    \scriptsize
    \begin{tabular}{lcccccc}
        \toprule
        \textbf{Architecture} & \textbf{Method}         & \textbf{NFEs} & \textbf{Delayed \& Noisy Policy} & \textbf{Delayed Policy}  & \textbf{Zero-Flow}        & \textbf{Proficient Human} \\
                              &                         &               & \textbf{(Worst Quality)}         & \textbf{(Mixed Quality)} & \textbf{(Better Quality)} & \textbf{(Good Quality)}   \\
        \midrule
        \chiunet              & \Regression             & 1             & 0.70/0.63                        & 0.80/0.72                & 0.76/0.65                 & 0.76/0.62                 \\
        \chiunet              & \Straightflow           & 1             & 0.70/0.62                        & 0.82/0.76                & 0.84/0.77                 & 0.62/0.38                 \\
        \chiunet              & \Minimaliterativepolicy & 2             & 0.80/0.72                        & 0.82/0.61                & 0.74/0.64                 & 0.80/0.68                 \\
        \chiunet              & \Flow                   & 9             & 0.76/0.68                        & 0.74/0.50                & 0.76/0.54                 & 0.84/0.70                 \\
        \midrule
        \chitransformer       & \Regression             & 1             & 0.38/0.22                        & 0.40/0.31                & 0.42/0.26                 & 0.50/0.24                 \\
        \chitransformer       & \Straightflow           & 1             & 0.46/0.35                        & 0.68/0.50                & 0.56/0.41                 & 0.62/0.48                 \\
        \chitransformer       & \Minimaliterativepolicy & 2             & 0.56/0.49                        & 0.70/0.54                & 0.64/0.56                 & 0.72/0.68                 \\
        \chitransformer       & \Flow                   & 9             & 0.56/0.34                        & 0.54/0.48                & 0.62/0.49                 & 0.68/0.54                 \\
        \midrule
        \sudeepdit            & \Regression             & 1             & 0.42/0.29                        & 0.36/0.28                & 0.42/0.32                 & 0.30/0.19                 \\
        \sudeepdit            & \Straightflow           & 1             & 0.66/0.41                        & 0.60/0.54                & 0.72/0.57                 & 0.68/0.50                 \\
        \sudeepdit            & \Minimaliterativepolicy & 2             & 0.66/0.56                        & 0.74/0.58                & 0.70/0.61                 & 0.86/0.78                 \\
        \sudeepdit            & \Flow                   & 9             & 0.56/0.45                        & 0.66/0.58                & 0.72/0.65                 & 0.78/0.68                 \\
        \bottomrule
    \end{tabular}
    \caption{Performance comparison across different methods and data quality levels. We evaluate on the task \toolhang with state observations using 10M parameter networks. Success rates are reported as averages over 5 checkpoints across 3 seeds.}
    \label{tab:performance_comparison}
\end{table}

\subsection{Full Results for \Minimaliterativepolicy and its variants}
\label{sec:control_full_results}

\begin{table}[ht]
    \centering
    \footnotesize
    \begin{adjustbox}{width=\textwidth,center}
        \begin{tabular}{ll|ccccccccccc}
            \toprule
            \textbf{Architecture} & \textbf{Method} & \multicolumn{2}{c}{\lifttask} & \multicolumn{2}{c}{\can}    & \multicolumn{2}{c}{\squaretask} & \multicolumn{2}{c}{\transport} & \toolhang                   & \pusht                      & \kitchen                                                                                                                                   \\
            \midrule
                                  &                 & mh                            & ph                          & mh                              & ph                             & mh                          & ph                          & mh                          & ph                 &                             &                             &                             \\
            \midrule
            \sudeepdit            & Flow            & \textbf{1.00}/0.99            & \textbf{1.00}/\textbf{1.00} & \textbf{1.00}/0.94              & \textbf{1.00}/1.00             & 0.88/0.75                   & \textbf{1.00}/0.94          & 0.40/0.27                   & 0.80/0.70          & 0.86/0.75                   & \textbf{0.98}/\textbf{0.95} & 0.98/0.96                   \\
            \sudeepdit            & Regression      & \textbf{1.00}/0.99            & \textbf{1.00}/\textbf{1.00} & 0.92/0.90                       & \textbf{1.00}/0.98             & 0.72/0.53                   & 0.94/0.86                   & 0.12/0.06                   & 0.50/0.44          & 0.52/0.39                   & 0.92/0.83                   & 0.98/0.92                   \\
            \sudeepdit            & Straight Flow   & \textbf{1.00}/0.98            & \textbf{1.00}/\textbf{1.00} & 0.96/0.90                       & \textbf{1.00}/0.99             & 0.72/0.66                   & 0.96/0.93                   & 0.20/0.14                   & 0.56/0.48          & 0.70/0.59                   & 0.90/0.86                   & 0.96/0.91                   \\
            \sudeepdit            & MIP             & \textbf{1.00}/0.99            & \textbf{1.00}/\textbf{1.00} & 0.98/0.95                       & \textbf{1.00}/\textbf{1.00}    & 0.90/0.81                   & 0.98/\textbf{0.94}          & 0.44/0.38                   & 0.76/0.68          & \textbf{0.92}/\textbf{0.88} & 0.95/0.92                   & \textbf{1.00}/\textbf{0.97} \\
            \chitransformer       & Flow            & \textbf{1.00}/\textbf{1.00}   & \textbf{1.00}/\textbf{1.00} & \textbf{1.00}/0.93              & \textbf{1.00}/0.98             & 0.78/0.74                   & 0.96/0.89                   & 0.44/0.34                   & \textbf{0.88}/0.64 & 0.68/0.54                   & 0.91/0.89                   & \textbf{1.00}/0.96          \\
            \chitransformer       & Regression      & \textbf{1.00}/0.99            & \textbf{1.00}/0.99          & 0.98/0.92                       & \textbf{1.00}/0.96             & 0.74/0.61                   & 0.92/0.85                   & 0.28/0.20                   & 0.68/0.51          & 0.40/0.36                   & 0.93/0.88                   & 0.98/0.91                   \\
            \chitransformer       & Straight Flow   & \textbf{1.00}/0.99            & \textbf{1.00}/\textbf{1.00} & 0.98/0.92                       & \textbf{1.00}/0.99             & 0.68/0.58                   & 0.96/0.89                   & 0.24/0.16                   & 0.62/0.54          & 0.60/0.55                   & 0.94/0.90                   & 0.96/0.92                   \\
            \chitransformer       & MIP             & \textbf{1.00}/\textbf{1.00}   & \textbf{1.00}/\textbf{1.00} & 0.96/0.95                       & \textbf{1.00}/1.00             & 0.86/0.73                   & 0.96/0.89                   & 0.42/0.37                   & 0.80/0.68          & 0.76/0.69                   & 0.94/0.92                   & 0.98/0.96                   \\
            \chiunet              & Flow            & \textbf{1.00}/\textbf{1.00}   & \textbf{1.00}/\textbf{1.00} & \textbf{1.00}/0.98              & \textbf{1.00}/\textbf{1.00}    & 0.90/0.78                   & 0.98/0.94                   & 0.52/0.40                   & 0.80/\textbf{0.73} & 0.84/0.70                   & \textbf{0.98}/0.94          & \textbf{1.00}/0.97          \\
            \chiunet              & Regression      & \textbf{1.00}/\textbf{1.00}   & \textbf{1.00}/\textbf{1.00} & \textbf{1.00}/0.96              & \textbf{1.00}/0.99             & \textbf{0.94}/\textbf{0.82} & \textbf{1.00}/0.91          & 0.22/0.16                   & 0.64/0.55          & 0.68/0.64                   & 0.95/0.89                   & 0.92/0.88                   \\
            \chiunet              & Straight Flow   & \textbf{1.00}/1.00            & \textbf{1.00}/\textbf{1.00} & \textbf{1.00}/0.92              & \textbf{1.00}/0.99             & \textbf{0.94}/0.79          & 0.98/0.90                   & 0.22/0.15                   & 0.64/0.52          & 0.50/0.00                   & 0.93/0.88                   & 0.86/0.79                   \\
            \chiunet              & MIP             & \textbf{1.00}/\textbf{1.00}   & \textbf{1.00}/\textbf{1.00} & \textbf{1.00}/\textbf{0.98}     & \textbf{1.00}/0.99             & 0.92/0.81                   & \textbf{1.00}/\textbf{0.94} & \textbf{0.62}/\textbf{0.46} & 0.80/0.69          & 0.80/0.64                   & 0.97/\textbf{0.95}          & \textbf{1.00}/0.96          \\
            \bottomrule
        \end{tabular}
    \end{adjustbox}
    \caption{Performance comparison of Flow and Regression methods across different \textbf{state-based} robotic manipulation tasks. For each task, we report the best checkpoint performance / averaged performance over last 5 checkpoints. Each experiment is run with 3 seeds and we report the average performance across all seeds.}
    \label{tab:results_state}
\end{table}

\begin{table}[ht]
    \centering
    \footnotesize
    \begin{adjustbox}{width=\textwidth,center}
        \begin{tabular}{ll|cccccccccc}
            \toprule
            Architecture    & Method        & \multicolumn{2}{c}{Lift}    & \multicolumn{2}{c}{Can}     & \multicolumn{2}{c}{Square}  & \multicolumn{2}{c}{Transport} & Tool Hang                   & PushT                                                                                                                             \\
            \midrule
                            &               & mh                          & ph                          & mh                          & ph                            & mh                          & ph                 & mh                          & ph                          &                    &                             \\
            \midrule
            \sudeepdit      & Flow          & \textbf{1.00}/\textbf{1.00} & \textbf{1.00}/1.00          & 0.96/0.94                   & \textbf{1.00}/0.99            & 0.82/0.76                   & 0.96/0.94          & 0.32/0.20                   & 0.84/0.83                   & \textbf{0.78}/0.57 & \textbf{0.92}/\textbf{0.89} \\
            \sudeepdit      & Regression    & \textbf{1.00}/0.99          & \textbf{1.00}/\textbf{1.00} & 0.92/0.81                   & \textbf{1.00}/\textbf{1.00}   & 0.74/0.67                   & 0.94/0.84          & 0.14/0.08                   & 0.74/0.56                   & 0.28/0.18          & 0.83/0.77                   \\
            \sudeepdit      & Straight Flow & \textbf{1.00}/0.99          & \textbf{1.00}/0.99          & 0.98/0.95                   & \textbf{1.00}/0.98            & 0.82/0.72                   & \textbf{1.00}/0.93 & 0.26/0.19                   & 0.86/0.83                   & 0.46/0.40          & 0.85/0.79                   \\
            \sudeepdit      & MIP           & \textbf{1.00}/0.99          & \textbf{1.00}/\textbf{1.00} & \textbf{1.00}/0.96          & \textbf{1.00}/0.98            & 0.90/0.83                   & \textbf{1.00}/0.92 & 0.50/0.31                   & 0.90/0.84                   & 0.76/\textbf{0.66} & 0.91/0.87                   \\
            \chitransformer & Flow          & \textbf{1.00}/0.99          & \textbf{1.00}/1.00          & 0.98/0.92                   & \textbf{1.00}/0.96            & 0.70/0.66                   & 0.98/0.93          & 0.24/0.22                   & 0.80/0.77                   & 0.54/0.40          & 0.89/0.85                   \\
            \chitransformer & Regression    & \textbf{1.00}/0.98          & \textbf{1.00}/0.98          & \textbf{1.00}/0.94          & \textbf{1.00}/0.96            & 0.76/0.70                   & 0.98/0.90          & 0.40/0.27                   & 0.94/0.87                   & 0.44/0.36          & 0.85/0.81                   \\
            \chitransformer & Straight Flow & \textbf{1.00}/\textbf{1.00} & \textbf{1.00}/\textbf{1.00} & \textbf{1.00}/0.95          & \textbf{1.00}/0.98            & 0.90/0.78                   & 0.98/\textbf{0.94} & 0.32/0.25                   & 0.86/0.70                   & 0.36/0.28          & 0.86/0.80                   \\
            \chitransformer & MIP           & \textbf{1.00}/0.98          & \textbf{1.00}/1.00          & 0.96/0.91                   & \textbf{1.00}/0.98            & 0.72/0.21                   & 0.90/0.04          & 0.18/0.06                   & 0.86/0.69                   & 0.60/0.48          & 0.87/0.83                   \\
            \chiunet        & Flow          & \textbf{1.00}/\textbf{1.00} & \textbf{1.00}/\textbf{1.00} & \textbf{1.00}/\textbf{0.97} & \textbf{1.00}/0.98            & 0.90/0.79                   & 0.96/0.90          & 0.24/0.16                   & 0.78/0.61                   & 0.48/0.37          & \textbf{0.92}/0.87          \\
            \chiunet        & Regression    & \textbf{1.00}/0.96          & \textbf{1.00}/0.99          & 0.84/0.70                   & 0.98/0.87                     & 0.74/0.66                   & 0.94/0.86          & 0.18/0.10                   & 0.66/0.64                   & 0.30/0.23          & 0.90/0.85                   \\
            \chiunet        & Straight Flow & \textbf{1.00}/0.94          & \textbf{1.00}/0.99          & 0.98/0.93                   & \textbf{1.00}/0.96            & 0.72/0.68                   & 0.92/0.62          & 0.00/0.00                   & 0.50/0.22                   & 0.06/0.02          & 0.87/0.82                   \\
            \chiunet        & MIP           & \textbf{1.00}/\textbf{1.00} & \textbf{1.00}/\textbf{1.00} & \textbf{1.00}/0.95          & \textbf{1.00}/0.98            & \textbf{0.92}/\textbf{0.84} & 0.96/0.91          & \textbf{0.52}/\textbf{0.37} & \textbf{0.96}/\textbf{0.91} & 0.56/0.50          & 0.83/0.78                   \\
            \bottomrule
        \end{tabular}
    \end{adjustbox}
    \caption{Performance comparison of Flow and Regression methods across different \textbf{image-based} robotic manipulation tasks. For each task, we report the best checkpoint performance / averaged performance over last 5 checkpoints. Each experiment is run with 3 seeds and we report the average performance across all seeds.}
    \label{tab:results_image}
\end{table}

\begin{table}[!ht]
    \centering
    \tiny
    \setlength\tabcolsep{3pt}
    \begin{tabular}{l|r|ccc|ccc}
        \toprule
        Architecture & Method     & \multicolumn{3}{c|}{\taskname{Adroit}} & \multicolumn{3}{c}{\metaworld}                                                                                                             \\
        \cmidrule(lr){3-5} \cmidrule(lr){6-8}
                     &            & Hammer                                 & Door                           & Pen                      & Stick-Push               & Assembly                 & Disassemble              \\
        \midrule
        DP3          & Flow       & $0.96 \pm 0.02$                        & $\mathbf{0.60 \pm 0.06}$       & $\mathbf{0.54 \pm 0.11}$ & $0.92 \pm 0.04$          & $\mathbf{0.98 \pm 0.03}$ & $0.72 \pm 0.14$          \\
                     & Regression & $\mathbf{0.97 \pm 0.04}$               & $0.52 \pm 0.16$                & $0.47 \pm 0.08$          & $\mathbf{0.95 \pm 0.06}$ & $\mathbf{0.98 \pm 0.03}$ & $\mathbf{0.78 \pm 0.08}$ \\
        \bottomrule
    \end{tabular}
    \caption{Performance comparison of Flow and Regression methods using DP3 architecture across different \textbf{point-cloud-based} robotic manipulation tasks. For each task, we report the best checkpoint performance / averaged performance over last 5 checkpoints. Each experiment is run with 3 seeds and we report the average performance across all seeds.}
    \label{tab:dp3_flow_vs_regression}
\end{table}

For \kitchen, the task has multiple stages. In the main results, we only report the performance of the last stage since it is the most challenging one. \Cref{tab:kitchen_performance_comparison} shows the performance comparison across different design choices on \kitchen task.

\begin{table}[!ht]
    \centering
    \footnotesize
    \begin{tabular}{llcccc}
        \toprule
        \textbf{Architecture}            & \textbf{Method}         & \textbf{P1} & \textbf{P2} & \textbf{P3} & \textbf{P4} \\
        \midrule
        \multirow{4}{*}{\chiunet}        & \Flow                   & 1.0         & 1.0         & 1.0         & 0.98        \\
                                         & \Minimaliterativepolicy & 1.0         & 1.0         & 1.0         & 0.94        \\
                                         & \Regression             & 0.98        & 0.94        & 0.94        & 0.86        \\
        \midrule
        \multirow{4}{*}{\chitransformer} & \Flow                   & 1.0         & 1.0         & 1.0         & 1.0         \\
                                         & \Minimaliterativepolicy & 1.00        & 0.98        & 0.98        & 0.96        \\
                                         & \Regression             & 1.0         & 1.0         & 0.98        & 0.94        \\
        \midrule
        \multirow{4}{*}{\sudeepdit}      & Flow                    & 1.0         & 1.0         & 1.0         & 0.98        \\
                                         & \Minimaliterativepolicy & 1.00        & 1.00        & 1.00        & 0.98        \\
                                         & \Regression             & 1.0         & 0.98        & 0.96        & 0.88        \\
        \bottomrule
    \end{tabular}
    \caption{\textbf{Performance comparison across different design choices on \textbf{kitchen} task.} Kitchen task has multiple stages and we report the success rate of finishing $n$ tasks in the table. For the performance reported in the main paper and previous tables, we report the success rate of finishing 4 tasks. }
    \label{tab:kitchen_performance_comparison}
\end{table}

\subsection{Different Method's Performance with Different Number of Function Evaluations}
\label{sec:method_nfe_comparison}

We also provide detailed evaluation on different method's scaling behavior given different amount of online computation budgets.
\Cref{tab:method_nfe_comparison} highlights that only \Minimaliterativepolicy and \Flow benefit from iterative computate.

\begin{table}[h]
    \centering
    \begin{tabular}{@{}lcccccccccccc@{}}
        \toprule
        \textbf{Method} & \multicolumn{1}{c}{\algname{Reg.}} & \multicolumn{3}{c}{\Straightflow} & \multicolumn{2}{c}{\Residualregression} & \multicolumn{2}{c}{\Minimaliterativepolicy} & \multicolumn{3}{c}{\Flow}                                           \\
        \cmidrule(lr){2-2} \cmidrule(lr){3-5} \cmidrule(lr){6-7} \cmidrule(lr){8-9} \cmidrule(lr){10-12}
        \textbf{NFEs}   & 1                                  & 1                                 & 3                                       & 9                                           & 1                         & 2    & 1    & 2    & 1    & 3    & 9    \\
        \midrule
        \textbf{S.R.}   & 0.46                               & 0.54                              & 0.55                                    & 0.52                                        & 0.31                      & 0.33 & 0.50 & 0.74 & 0.32 & 0.55 & 0.66 \\
        \bottomrule
    \end{tabular}
    \caption{\textbf{Comparison of methods and their corresponding number of function evaluations (NFEs).} Evaluated on state-based \toolhang task over \chiunet. Average success rate is reported across 3 seeds and 5 checkpoints.}
    \label{tab:method_nfe_comparison}
\end{table}

\subsection{Comparing \Minimaliterativepolicy with Consistency Models}
\label{sec:comp_consistency_models}

Given \Minimaliterativepolicy takes less integration steps compared to flow model, we compare it with consistency models which accelerate the sampling process of flow by distilling the learned flow into a shortcut model.
The major difference between \Minimaliterativepolicy and consistency models is that the latter do satisfy \componentref{comp:distr}, and require training over a continuum of noise levels.
On the other hand, \Minimaliterativepolicy is trained to predict the conditional mean of the interpolant, and thus, doesn't need extra distillation stage.
As shown in \Cref{tab:comp_consistency_models}, We benchmarks \Minimaliterativepolicy to common consistency model training methods including consistency trajectory model (CTM)~\citep{kim2023consistency} and Lagrangian map distillation (LMD)~\citep{boffiFlowMapMatching2025}, where LMD only works for \chiunet due to its dependency on jacobian matrix computation.
The benchmarking results indicates that, given best architecture, \Minimaliterativepolicy outperforms consistency models.
In terms of training time, \Minimaliterativepolicy only takes half of the time compared to CTM, where LMD training takes even longer due to jacobian matrix computation.

\begin{table}[h]
    \centering
    \scriptsize
    \begin{tabular}{ll|ccc}
        \toprule
        Architecture    & Method & \multicolumn{2}{c}{Transport} & Tool Hang                                        \\
        \midrule
                        &        & mh                            & ph                 &                             \\
        \midrule
        \sudeepdit      & Flow   & 0.40/0.27                     & 0.80/0.70          & 0.86/0.75                   \\
        \sudeepdit      & MIP    & 0.44/0.38                     & 0.76/0.68          & \textbf{0.92}/\textbf{0.88} \\
        \chitransformer & CTM    & 0.57/0.32                     & \textbf{0.90}/0.58 & 0.56/0.26                   \\
        \chitransformer & Flow   & 0.44/0.34                     & 0.88/0.64          & 0.68/0.54                   \\
        \chitransformer & MIP    & 0.42/0.37                     & 0.80/0.68          & 0.76/0.69                   \\
        \chiunet        & CTM    & 0.40/0.32                     & 0.72/0.63          & 0.46/0.37                   \\
        \chiunet        & Flow   & 0.52/0.40                     & 0.80/\textbf{0.73} & 0.84/0.70                   \\
        \chiunet        & LMD    & 0.44/0.32                     & 0.76/0.68          & 0.74/0.52                   \\
        \chiunet        & MIP    & \textbf{0.62}/\textbf{0.46}   & 0.80/0.69          & 0.80/0.64                   \\
        \bottomrule
    \end{tabular}
    \caption{Benchmark results across different architectures and methods on state-based tasks on consistency models and \Minimaliterativepolicy. Report average/best performance across 5 checkpoints with 3 random seeds. Both LMD and CTM integrate 2 steps, which is the same as \Minimaliterativepolicy.}
    \label{tab:comp_consistency_models}
\end{table}

\section{Lipschitz Constant Study Details}
\label{sec:lipchitz_study_details}

\subsection{Lipschitz Evluation Method}
\label{sec:lipchitz_evaluation_method}

We note that not all inputs $o$ are dynamically feasible, and our dataset lies only on a narrow manifold of the observation space.
Therefore, we must carefully evaluate the Lipschitz constant on the feasible observation space to avoid conflating model expressivity with errors arising from infeasible states.
To ensure feasibility, instead of directly perturbing the state, we perturb the action and then roll it out in the environment.
This guarantees that both the perturbed state and the resulting observation remain feasible.

In practice, we identify states that exhibit the highest ambiguity of actions in the dataset, referred to as \emph{critical states}.
For each critical state, we inject Gaussian noise $\eta \sim \mathcal{N}(0, \epsilon^2 I)$ into the normalized action, unnormalize it, and then roll it out.
We select $100$ critical states from the dataset. For each state, we perturb the corresponding expert action $a$ with $64$ independent Gaussian samples.

Let $o$ denote the next nominal observation after applying the nominal action $a$.
After rolling out the perturbed actions, we obtain perturbed observations $o^{(1)}, \dots, o^{(N_\text{perturb})}$.
The policy then predicts the perturbed actions $a^{(i)} = \pi(o^{(i)})$.
To ensure comparability across different states and tasks, we evaluate the Lipschitz constant with respect to normalized observations
$\bar{o} = \frac{o - \mu_o}{\sigma_o}$ and normalized actions $\bar{a} = \frac{a - \mu_a}{\sigma_a}$.
Finally, the Lipschitz constant is estimated using a zeroth-order approximation:
\begin{align}
    L \approx \max_i \frac{\|\bar{a}^{(i)} - \bar{a}\|_2}{\|\eta\|_2}.
\end{align}
Full version of above process is stated in~\cref{alg:lipschitz_alg}.

\begin{algorithm}[H]
    \caption{Lipschitz Constant Estimation via Action Perturbation}
    \label{alg:lipschitz_alg}
    \begin{algorithmic}[1]
        \Require Dataset $\mathcal{D}$, policy $\pi$, noise scale $\epsilon$, number of critical states $N_s{=}100$, number of perturbations $N_p{=}64$
        \Ensure Estimated Lipschitz constant $L$
        \State $S \gets$ identify $N_s$ critical states from $\mathcal{D}$ \Comment{Select states with highest action ambiguity}
        \ForAll{critical state $s \in S$}
        \State $(a, o) \gets$ expert action and nominal next observation for $s$ \Comment{Get ground truth action-observation pair}
        \State $(\bar{a}, \bar{o}) \gets$ normalize $(a, o)$ using dataset statistics \Comment{Ensure comparability across states/tasks}
        \For{$i = 1$ to $N_p$}
        \State $\eta \sim \mathcal{N}(0, \epsilon^2 I)$ \Comment{Sample Gaussian perturbation}
        \State $a_{\text{pert}} \gets$ unnormalize$(\bar{a} + \eta)$ \Comment{Create perturbed action in original scale}
        \State $o^{(i)} \gets$ rollout$(a_{\text{pert}})$ in environment \Comment{Execute perturbed action to get feasible state}
        \State $\bar{o}^{(i)} \gets$ normalize$(o^{(i)})$ \Comment{Normalize perturbed observation}
        \State $a^{(i)} \gets \pi(o^{(i)})$ \Comment{Get policy prediction on perturbed state}
        \State $\bar{a}^{(i)} \gets$ normalize$(a^{(i)})$ \Comment{Normalize predicted action}
        \State $r_i \gets \frac{\|\bar{a}^{(i)} - \bar{a}\|_2}{\|\eta\|_2}$ \Comment{Compute finite difference approximation}
        \EndFor
        \State $L_s \gets \max_i r_i$ \Comment{Local Lipschitz constant for state $s$}
        \EndFor
        \State $L \gets \frac{1}{N_s} \sum_{s=1}^{N_s} L_s$ \Comment{Average across all critical states}
        \State \Return $L$
    \end{algorithmic}
\end{algorithm}

\subsection{Full Lipschitz Evaluation Results}
\label{sec:full_lipschitz_evaluation_results}
In the main text, we only report the average Lipschitz constant on critical states across 3 architectures.
Here, we report the full Lipschitz constant evaluation reuslt in \cref{tab:detailed_lipschitz} with different architectures and tasks.

\begin{table}[h]
    \centering
    \footnotesize
    \begin{tabular}{@{}l l l c@{}}
        \toprule
        \textbf{Task} & \textbf{Architecture}                     & \textbf{Method}  & \textbf{Lipschitz Constant (Policy)} \\
        \midrule
        \multirow{6}{*}{\pusht (State)}
                      & \multirow{2}{*}{\texttt{\chiunet}}
                      & \Regression                               & $0.85 \pm 0.58$                                         \\
                      &                                           & \Flow            & $0.31 \pm 0.01$                      \\
                      & \multirow{2}{*}{\texttt{\sudeepdit}}
                      & \Regression                               & $0.52 \pm 0.11$                                         \\
                      &                                           & \Flow            & $0.22 \pm 0.02$                      \\
                      & \multirow{2}{*}{\texttt{\chitransformer}}
                      & \Regression                               & $1.33 \pm 1.14$                                         \\
                      &                                           & \Flow            & $0.82 \pm 0.26$                      \\
        \midrule
        \multirow{6}{*}{\kitchen (State)}
                      & \multirow{2}{*}{\texttt{\chiunet}}
                      & \Regression                               & $13.47 \pm 2.80$                                        \\
                      &                                           & \Flow            & $13.31 \pm 4.13$                     \\
                      & \multirow{2}{*}{\texttt{\sudeepdit}}
                      & \Regression                               & $15.37 \pm 3.69$                                        \\
                      &                                           & \Flow            & $12.54 \pm 5.09$                     \\
                      & \multirow{2}{*}{\texttt{\chitransformer}}
                      & \Regression                               & $13.37 \pm 4.00$                                        \\
                      &                                           & \Flow            & $11.44 \pm 4.10$                     \\
        \midrule
        \multirow{2}{*}{\toolhang (PH, State)}
                      & \multirow{2}{*}{\texttt{\chiunet}}
                      & \Regression                               & $1.63 \pm 0.79$                                         \\
                      &                                           & \Flow            & $1.53 \pm 1.01$                      \\
                      & \multirow{2}{*}{\texttt{\sudeepdit}}
                      & \Regression                               & $1.86 \pm 0.81$                                         \\
                      &                                           & \Flow            & $1.34 \pm 0.97$                      \\
                      & \multirow{2}{*}{\texttt{\chitransformer}}
                      & \Regression                               & $1.76 \pm 1.02$                                         \\
                      &                                           & \Flow            & $1.40 \pm 0.99$                      \\
        \bottomrule
    \end{tabular}
    \caption{\textbf{Detailed: Per-architecture policy Lipschitz.}}
    \label{tab:detailed_lipschitz}
\end{table}

\section{Multi-Modality Study Details}
\label{sec:multi_modality_study}

\subsection{Q Function Estimation}
\label{sec:q_function_estimation}

To rule out the possibility of hidden multi-modality, we also plot Q functions for each action to see if there is any clear clustering pattern of $Q$ w.r.t. different actions in t-SNE visualization.
Since we only have access to expert actions rather than their policy, we estimate the Q function by Monte Carlo sampling with the learned flow policy.
The detailed procedure is as follows:

Starting from one ``critical state'', we first sample $N$ actions
\[
    a^{(i)} = \Phi(o, z^{(i)}, s=0, t=1), \quad i=1, \dots, N, \quad z^{(i)} \sim \Normal(0,\eye).
\]
For each sampled action $a^{(i)}$, we execute one environment step to obtain the next observation $o'^{(i)}$ and immediate reward $r(o, a^{(i)})$.
Then, starting from $o'^{(i)}$, we rollout the learned policy for $N_\text{MC}$ episodes until termination (horizon $H$), and average the cumulative returns to obtain an estimate of the continuation value.
Thus, the Q-value for action $a^{(i)}$ is approximated as:
\begin{align}
    Q_\Phi(a^{(i)}, o)
     & = r(o, a^{(i)}) + \frac{1}{N_\text{MC}} \sum_{j=1}^{N_\text{MC}} \sum_{t=1}^H r\big(o_t^{(j)}, a_t^{(j)}\big).
\end{align}
We set the discount factor $\gamma = 1.0$ since rewards are sparse and triggered only at task completion.
The reward for \toolhang and \kitchen is defined by the \emph{final} success signal (with \kitchen’s success requiring all 4 subtasks to be completed).
The reward for \pusht is defined by \emph{final} coverage.

\begin{algorithm}[H]
    \caption{Q Function Estimation via Monte Carlo Sampling}
    \label{alg:q_function_estimation}
    \begin{algorithmic}[1]
        \Require Dataset $\mathcal{D}$, flow policy $\Phi$, reward function $r$, number of critical states $N_s{=}100$, number of action samples $N$, Monte Carlo samples $N_\text{MC}$
        \Ensure For each state $o$, pairs $\{(a^{(i)}, Q_\Phi(a^{(i)}, o))\}_{i=1}^N$
        \State $S \gets$ identify $N_s$ critical states from $\mathcal{D}$ \Comment{Select states with highest action ambiguity}
        \ForAll{critical state $s \in S$}
        \State $o \gets$ observation for state $s$
        \For{$i = 1$ to $N$} \Comment{Sample actions and compute Q estimates}
        \State $z^{(i)} \sim \Normal(0, \eye)$
        \State $a^{(i)} \gets \Phi(o, z^{(i)}, s{=}0, t{=}1)$
        \State Execute $(o, a^{(i)})$ in env $\to$ obtain $o'^{(i)}$, $r^{(i)} = r(o, a^{(i)})$
        \For{$j = 1$ to $N_\text{MC}$} \Comment{Monte Carlo rollouts from $o'^{(i)}$}
        \State Rollout $\Phi$ from $o'^{(i)}$ until horizon $H$ to get cumulative return $R_j^{(i)}$
        \EndFor
        \State $Q_\Phi(a^{(i)}, o) \gets r^{(i)} + \frac{1}{N_\text{MC}} \sum_{j=1}^{N_\text{MC}} R_j^{(i)}$
        \EndFor
        \State Store $\{(a^{(i)}, Q_\Phi(a^{(i)}, o))\}_{i=1}^N$ for state $s$
        \EndFor
    \end{algorithmic}
\end{algorithm}

The procedure above explicitly computes Q-values by rolling out trajectories separately for each sampled action.

\subsection{Deterministic Dataset Generation}
\label{sec:deterministic_dataset_generation}

To generate a deterministic dataset that completely eliminates any potential multi-modality, we follow a systematic process:

First, we train a flow expert policy $\Phi$ on the original dataset. Then, we collect a new dataset by rolling out this expert policy from different initial states (using different random seeds than those used during testing). Crucially, during rollout, we always evaluate the flow policy deterministically by setting the initial noise to zero: $z=0$. This ensures that the policy produces deterministic actions given any observation, completely removing any stochasticity from the action generation process.

During data collection, we discard all failed trajectories to maintain the same success rate as the original dataset. We continue collecting until we reach the target number of trajectories $N_{\text{traj}}$.

\begin{algorithm}[H]
    \caption{Deterministic Dataset Generation}
    \label{alg:deterministic_dataset_generation}
    \begin{algorithmic}[1]
        \Require Trained flow policy $\Phi$, target number of trajectories $N_{\text{traj}}$, maximum episode steps $T_{\max}$
        \Ensure Deterministic dataset $\mathcal{D}_{\text{det}}$
        \State $\mathcal{D}_{\text{det}} \gets \emptyset$
        \State $n_{\text{collected}} \gets 0$
        \While{$n_{\text{collected}} < N_{\text{traj}}$}
        \State Reset environment with new random seed
        \State $o_0 \gets$ initial observation
        \State $\tau \gets [(o_0, \cdot)]$ \Comment{Initialize trajectory}
        \For{$t = 0$ to $T_{\max} - 1$}
        \State $a_t \gets \Phi(z=0, o_t, s=0, t=1)$ \Comment{Deterministic action}
        \State $o_{t+1}, r_t, \text{done} \gets \text{env.step}(a_t)$
        \State $\tau \gets \tau \cup [(o_t, a_t)]$
        \If{done}
        \State \textbf{break}
        \EndIf
        \EndFor
        \If{trajectory $\tau$ is successful}
        \State $\mathcal{D}_{\text{det}} \gets \mathcal{D}_{\text{det}} \cup \{\tau\}$
        \State $n_{\text{collected}} \gets n_{\text{collected}} + 1$
        \EndIf
        \EndWhile
        \State \Return $\mathcal{D}_{\text{det}}$
    \end{algorithmic}
\end{algorithm}

\section{Manifold Adherence Study Details}
\label{sec:manifold_adherence_study_details}

\subsection{Validation Loss Is Not a Good Proxy for Policy Performance}
\label{sec:validation_loss_is_not_a_good_proxy_for_policy_performance}

To investigate whether validation loss serves as a reliable proxy for policy performance, we examine its relationship with success rates on \toolhang across different architectures given different training methods.
Evidence that validation loss is poorly correlated with success rate can be seen by comparing flow policies with varying numbers of function evaluations (NFEs) and their corresponding validation losses.
\Cref{tab:validation_loss_vs_success_rate} demonstrates that increasing NFEs does not reduce validation loss, yet policy performance consistently improves.
We hypothesize that higher NFEs introduce stronger inductive bias and regularization, which projects actions back onto the data manifold, thereby enhancing generalization.

\begin{table}[h]
    \centering
    \begin{tabular}{lllcc}
        \toprule
        Architecture                     & Method      & NFEs & Average Success Rate & $L_2$ Validation Loss \\
        \midrule
        \multirow{4}{*}{\chiunet}        & \Regression & 1    & 0.54                 & 0.063                 \\
                                         & \Flow       & 1    & 0.36                 & 0.053                 \\
                                         & \Flow       & 3    & 0.44                 & 0.052                 \\
                                         & \Flow       & 9    & 0.58                 & 0.053                 \\
        \midrule
        \multirow{4}{*}{\chitransformer} & \Regression & 1    & 0.18                 & 0.084                 \\
                                         & \Flow       & 1    & 0.06                 & 0.093                 \\
                                         & \Flow       & 3    & 0.72                 & 0.092                 \\
                                         & \Flow       & 9    & 0.68                 & 0.089                 \\
        \midrule
        \multirow{4}{*}{\sudeepdit}      & \Regression & 1    & 0.20                 & 0.063                 \\
                                         & \Flow       & 1    & 0.62                 & 0.082                 \\
                                         & \Flow       & 3    & 0.76                 & 0.080                 \\
                                         & \Flow       & 9    & 0.76                 & 0.080                 \\
        \bottomrule
    \end{tabular}
    \caption{Comparison of validation loss and success rate across different architectures and methods on state-based \toolhang. The results show that validation loss is not a reliable proxy for policy performance.}
    \label{tab:validation_loss_vs_success_rate}
\end{table}

\subsection{Manifold Adherence Evaluation Method}
\label{sec:manifold_adherence_evaluation_details}

To evaluate the manifold adherence, we compute the projection error of a predicted action $a$ onto the space spanned by expert actions at neighboring states.
Concretely, given a state, we compute its $\ell_2$ distance to all states in the training set.
Then, we pick $k$ nearest neighbor states and gather their corresponding actions $A = [a^{(0)},a^{(1)}, \dots, a^{(k)}]$.
Lastly, we compute projection error by projecting $a$ to the column space of $A$: $\|a -P_{A}(a)\|_2 = \min_c \|a -Ac\|_2$.

\section{Nearest Neighbor Hypothesis Study}
\label{sec:nearest_neighbor_hypothesis_study}
Another popular hypothesis is that GCPs are learning a lookup table of observation-to-action mappings~\citep{pari2021surprising,he2025demystifyingdiffusionpoliciesaction}.
This might be true for relatively simple tasks that do not require high precision and complex generalization, such as \can.
However, for tasks that require higher precision and more contact, such as \toolhang, the nearest-neighbor/lookup-table assumption is insufficient to explain the success of GCPs.
We evaluate the performance of a nearest-neighbor policy (VINN~\citep{pari2021surprising}) on state-based \toolhang and find that it achieves a success rate of only $12\%$ as shown in \cref{tab:nearest_neighbor_results}.
This is significantly lower than both flow and regression methods, indicating that the action manifold is not linearly spanned by the expert actions.
Nevertheless, nearest-neighbor can still serve as a proxy for the expert action manifold, as it captures the general trend of actions---even though linear combinations of actions in the dataset cannot directly produce the correct action, the expert action manifold should not be too distant.
Therefore, in this paper, we use nearest-neighbor as a proxy for the linearized expert action manifold rather than directly computing the distance between expert actions in the validation set and predicted actions.

\begin{table}[h]
    \centering
    \begin{tabular}{cc}
        \toprule
        \textbf{Action Chunk Size} & \textbf{Success Rate (\%)} \\
        \midrule
        1                          & $0$                        \\
        8                          & $4$                        \\
        16                         & $12$                       \\
        32                         & $2$                        \\
        \bottomrule
    \end{tabular}
    \caption{Performance of k-nearest neighbor policy on state-based \toolhang task. Using the same method as VINN with softmax over k=5 nearest neighbors. }
    \label{tab:nearest_neighbor_results}
\end{table}

\subsection{Action Chunk Size Study}
\label{sec:action_chunk_size_study}

Another equivalent important factor is the action chunk size.
\cref{fig:action_chunk_ablation} highlights the importance of action chunk size, where regression with larger action chunk can outperform flow with smaller action chunk size.
Our ablation also indicates that \Minimaliterativepolicy outperforms flow with smaller action chunk size and matches the performance of flow with larger action chunk size.

\begin{figure}[!ht]
    \centering
    \includegraphics[width=1.0\textwidth]{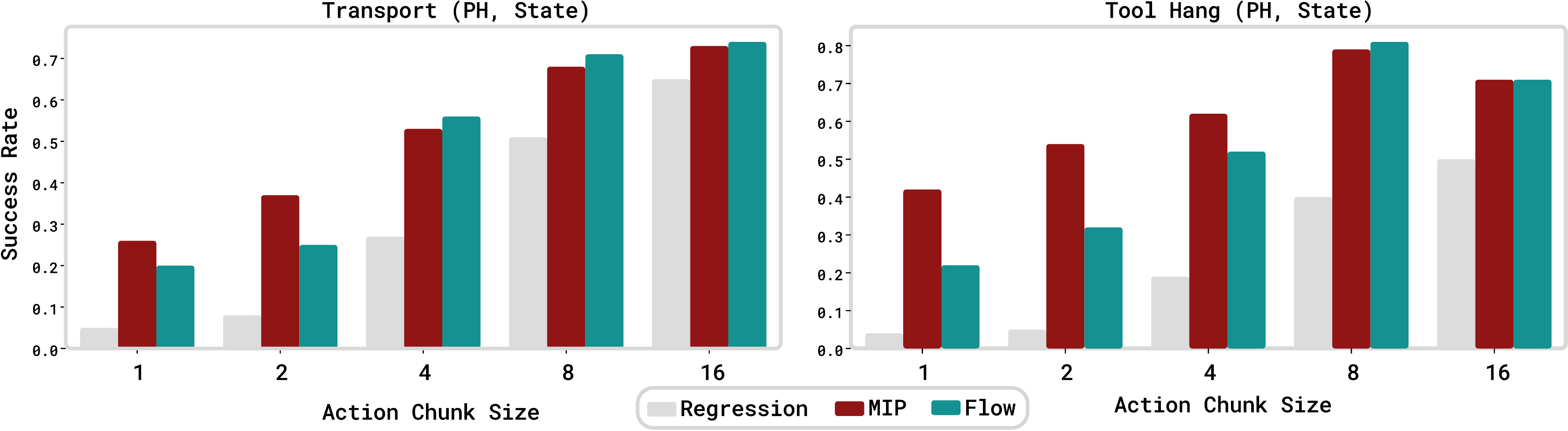}
    \caption{\footnotesize \textbf{Action chunk size ablation.} Success rate are averaged across 3 seeds, 3 architectures and 5 checkpoints on \toolhang and \transport tasks. Prediction horizon is set to powers of 2 to make sure it is compatible with \chiunet Architecture.}
    \label{fig:action_chunk_ablation}
\end{figure}

\subsection{Loss Norm Type Ablation Study}
\label{sec:loss_norm_type_ablation_study}

Previous work~\citep{kim2024openvla} shows that $\ell_1$ loss is superior to $\ell_2$ loss for regression policy.
To test this hypothesis, we ablate the loss norm type and compare the performance of different methods with $\ell_1$ and $\ell_2$ loss.
\Cref{tab:loss_norm_type_ablation} shows the performance comparison of different methods with $\ell_1$ and $\ell_2$ loss, where we find that $\ell_1$ loss generally outperforms $\ell_2$ loss, especially for regression policy.
However, even with $\ell_1$ loss, we still observe that Regression $<$ \Minimaliterativepolicy $\approx$ Flow, highlighting the importance of the stochasticity injection and iterative computation is independent of the loss function.

\begin{table}[h]
    \centering
    \small
    \begin{tabular}{ll|cc}
        \toprule
        Architecture & Method        & Transport (ph) & Tool Hang (ph) \\
        \midrule
        DiT          & Regression L1 & 0.72/0.64      & 0.76/0.65      \\
        DiT          & Regression L2 & 0.50/0.45      & 0.50/0.37      \\
        DiT          & MIP L1        & 0.81/0.73      & 0.91/0.84      \\
        DiT          & MIP L2        & 0.79/0.69      & 0.92/0.85      \\
        DiT          & Flow L1       & 0.83/0.76      & 0.93/0.84      \\
        DiT          & Flow L2       & 0.81/0.71      & 0.89/0.75      \\
        \midrule
        Transformer  & Regression L1 & 0.67/0.57      & 0.44/0.33      \\
        Transformer  & Regression L2 & 0.65/0.54      & 0.31/0.25      \\
        Transformer  & MIP L1        & 0.80/0.68      & 0.85/0.77      \\
        Transformer  & MIP L2        & 0.80/0.69      & 0.80/0.72      \\
        Transformer  & Flow L1       & 0.77/0.69      & 0.81/0.71      \\
        Transformer  & Flow L2       & 0.79/0.65      & 0.73/0.61      \\
        \midrule
        UNet         & Regression L1 & 0.84/0.71      & 0.71/0.55      \\
        UNet         & Regression L2 & 0.66/0.59      & 0.73/0.59      \\
        UNet         & MIP L1        & 0.85/0.68      & 0.76/0.67      \\
        UNet         & MIP L2        & 0.81/0.72      & 0.82/0.71      \\
        UNet         & Flow L1       & 0.85/0.69      & 0.87/0.71      \\
        UNet         & Flow L2       & 0.83/0.75      & 0.87/0.73      \\
        \bottomrule
    \end{tabular}
    \caption{Comparison of L1 vs L2 norm across different methods and architectures. Report average/best performance across 5 checkpoints with 3 random seeds.}
    \label{tab:loss_norm_type_ablation}
\end{table}

\section{Diversity of GCPs and RCPs}\label{sec:diversity}

A commonly believed hypothesis is that GCPs can express more diverse behaviors than RCPs by capturing the full distribution of expert actions~\citep{shafiullah2022behavior}.
We evalute different variants of GCPs and RCPs on \KitchenEnv, where the expert shows multiple task completion orders.
As demonstrated in \cref{fig:kitchen_methods_comparison}, GCPs with both stochastic and deterministic sampling show similar task completion order diversity.
Deterministic policies like regression and \Minimaliterativepolicy (to be introduced in \Cref{sec:parsing}) also demonstrate similar task completion order diversity.
This indicates that, given sparse expert demonstrations, both GCPs and RCP learns high-Lipschitz policies to switch between different modes given different observations (corresponding to (b.2) case in~\Cref{fig:multimodality_vs_lipschitz}).  RCPs and GCPs are equally good at learning such behaviors (\Cref{fig:diversity_of_gcp_and_rcp}) which explain why we see similar performance for both policy parametrizations, even on seemingly multi-modal tasks like \KitchenEnv.

\begin{figure}[ht]
    \centering
    \includegraphics[width=1.0\linewidth]{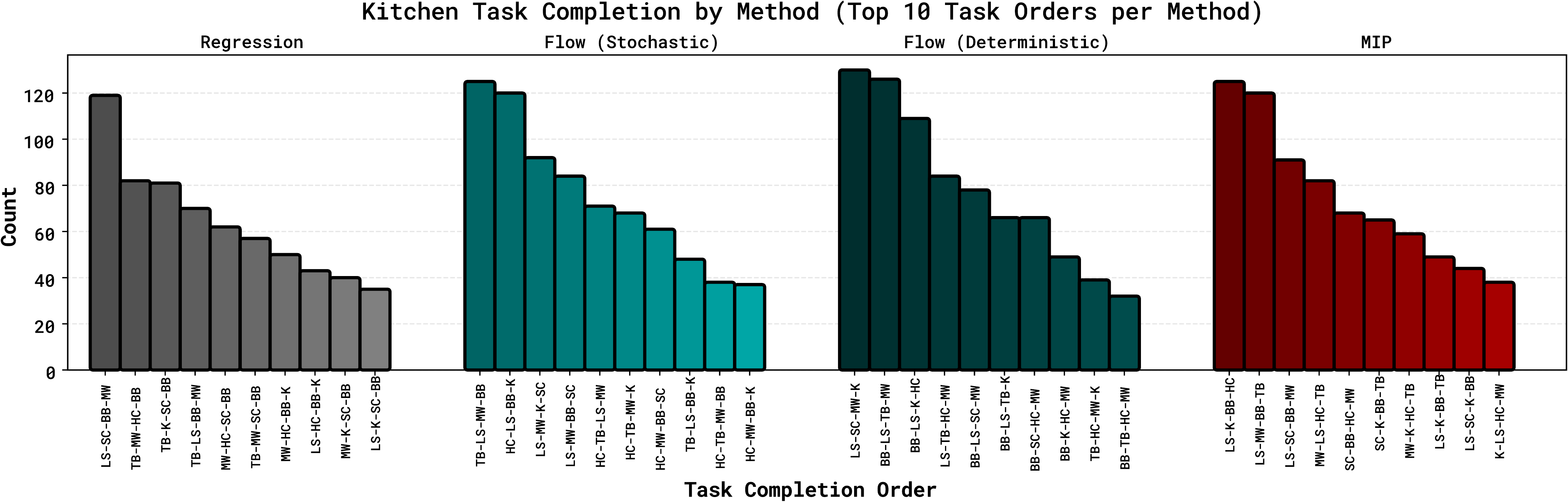}
    \captionof{figure}{\footnotesize \textbf{Task completion order in Kitchen environment with different methods.}
        We plot the count of different task completion orders for different methods to evaluate the diversity of the policies.
        The x-axis shows the task completion order, where each sub-task is represented by its initials.
        For each run, we collect 1000 trajectories with the same seed shared by all methods.
        For flow, we evaluate both stochastic and deterministic modes.
    }
    \label{fig:diversity_of_gcp_and_rcp}
\end{figure}

\section{Theoretical analysis of GCP's expressivity}
\label{app:theory}

\subsection{Formal statement of Theorem \ref{thm:informal}}

In this section, we introduce the notation and definition required for the subsequent proofs and provide the formal statement of \Cref{thm:informal} from the main text. Throughout, let $\|\cdot\|_{\circ}$ denote any matrix norm satisfying the property $\|X_1 X_2\|_{\circ}  \le \|X_1\|_{\op} \|X_2\|_{\circ}$. In contrast to the notation used in the main text, we define $\Phi_{s, t}(a, o)$ as the solution at time $t$ of the ODE:
\begin{align}
    \ddt a_t = b^{\star}_t(a_t \mid o), \quad \text{with initial condition } a_s = a.
\end{align}

Note that $\Phi_{0,1}(\gen_0 = z, \cond)$ coincides with the definition of $\pi^{\star}_{\theta}(z, \cond)$ in the main text. Next, we define the notion of \emph{$\kappa$-log-concavity}.

\begin{definition}[$\kappa$-log-concavity]
    A distribution with density $\rho = e^{-V(x)}$ is said to be \emph{$\kappa$-log-concave} if $V \in C^{2}(\R^{d})$ and its Hessian satisfies $\nabla^2 V(x) \succcurlyeq \kappa \eye$ for all $x\in\R^d$ and some $\kappa > 0$.
\end{definition}

With this notation in place, we now state the formal version of \Cref{thm:informal}.

\begin{theorem}
    \label{thm:lipschitz}
    Suppose that
    \begin{align}
        b^{\star}_t = \Exp[\dot I_t \mid I_t, \cond], \qquad \text{where } I_t = (1-t) a_0 + t a_1, \quad a_0 \sim \Normal(0,\eye), \quad a_1 \sim \rho_1,
    \end{align}
    where $\rho_1$ is $\kappa$-log-concave. Then, we have
    \begin{align}
        \|\nabla_\cond\Phi_{0,t}(\gen_0, \cond)\|_{\circ} \le \int_0^t \sqrt{\frac{\kappa (1-t)^2 + t^{2}}{\kappa (1-s)^2 + s^{2}}} \cdot \|\nabla_\cond b^{\star}_{s}(\gen_s \mid \cond)\|_{\circ} \rmd s.
    \end{align}
    In particular, for $t = 1$ we obtain
    \begin{align}
        \label{eq:holder}
        \|\nabla_\cond\Phi_{0,1}(\gen_0, \cond)\|_{\circ} \le \sqrt{1 + \kappa^{-1}} \int_0^1 \|\nabla_\cond b^{\star}_{s}(\gen_s \mid \cond)\|_{\circ} \rmd s.
    \end{align}
\end{theorem}
\begin{remark}
    \Cref{thm:informal} follows immediately from the fact that both the operator and the Frobenius norms satisfy $\|X_1 X_2\|_{\circ} \le \|X_1\|_{\op} \|X_2\|_{\circ}$ together with the inequality \Cref{eq:holder}.
\end{remark}

\subsection{Supporting lemmas}

We state the supporting lemmas for proving \Cref{thm:lipschitz} below and provide their proofs immediately for completeness. As a first step, we analyze the dynamical system satisfied by $\nabla_\cond \Phi_{s,t}(\gen, \cond)$.
\begin{lemma} Define $\gen_t := \Phi_{0,t}(\gen_0, \cond)$ where $\gen_0$ is the initial condition, and define the matrices
    \begin{align}
        M_t := \nabla_\cond\Phi_{0,t}(\gen_0, \cond), \quad A_t := (\nabla_\gen b^{\star}_t)(\gen_t \mid \cond), \quad E_t := (\nabla_\cond b^{\star}_t)(\gen_t \mid \cond)
    \end{align}
    Then,
    \begin{align}
        \ddt M_t = A_t M_t + E_t, \quad M_0 = 0
    \end{align}
\end{lemma}
\begin{proof} Since $\Phi_{0,0}(\gen_0, \cond) = \gen_0$, $M_0 = 0$. Moreover,
    \begin{align}
        \ddt \nabla_\cond\Phi_{0,t}(\gen_0, \cond) & = \nabla_\cond \ddt \Phi_{0,t}(\gen_0, \cond) = \nabla_\cond (b^{\star}_t(\Phi_{0,t}(\gen_0, \cond) \mid \cond))                            \\
                                                   & = (\nabla_\gen b^{\star}_t)(\gen_t \mid \cond) \cdot \nabla_\cond \Phi_{0,t}(\gen_0, \cond) + (\nabla_\cond b^{\star}_t)(\gen_t \mid \cond)
        \label{eq:matrices ode}
    \end{align}
\end{proof}

Note that, from the previous lemma, we may introduce $\Lambda_{s, t}$ as the solution to the matrix ODE
\begin{align}
    \ddt \Lambda_{s,t} = A_t \Lambda_{s,t}, \quad \Lambda_{s,s} = \eye.
    \label{eq:Lambda ode}
\end{align}
Moreover, it follows that
\begin{align}
    \Lambda_{s,t} = \nabla_{\gen}\Phi_{s,t}(\gen_s, \cond).
\end{align}

We are now ready to state the relation between $M_t$ and $\Lambda_{s, t}$.

\begin{lemma}
    \begin{align}
        M_t = \int_0^t \Lambda_{s,t} E_s \rmd s.
    \end{align}
\end{lemma}
\begin{proof}
    Using $\ddt \Lambda_{0, t}^{-1} = -\Lambda_{0, t}^{-1} A_t$ and we consider the time derivative of $\Lambda_{0, t}^{-1} M_t$:
    \begin{align}
        \ddt (\Lambda_{0, t}^{-1} M_t) & = (\ddt \Lambda_{0, t}^{-1}) M_t + \Lambda_{0, t}^{-1}(\ddt M_t)                       \\
                                       & = -\Lambda_{0, t}^{-1} A_t M_t + \Lambda_{0, t}^{-1} A_t M_t + \Lambda_{0, t}^{-1} E_t \\
                                       & = \Lambda_{0, t}^{-1} E_t.
    \end{align}
    Note that $\Lambda_{0, t}$ is invertible by uniqueness of the ODE solution in \Cref{eq:Lambda ode}. Integrating both sides with respect to $t$ gives
    \begin{align}
        \Lambda_{0, t}^{-1} M_t = \int_{0}^{t} \Lambda_{0, s}^{-1} E_s \rmd s.
    \end{align}
    Hence, we have
    \begin{align}
        M_t = \Lambda_{0, t} \int_{0}^{t} \Lambda_{0, s}^{-1} E_s \rmd s.
    \end{align}
    Note that $\Lambda_{0, s}^{-1} = \Lambda_{s, 0}$ and $\Lambda_{0, t} \cdot \Lambda_{s, 0} = \Lambda_{s, t}$, we obtain
    \begin{align}
        M_t = \int_0^t \Lambda_{s,t} E_s \rmd s.
    \end{align}
\end{proof}

An immediate application of the triangle inequality and the property of $\|\cdot \|_{\circ}$ yields
\begin{align}
    \label{eq:M_t bound}
    \|M_t\|_\circ \le \int_0^t\|\Lambda_{s,t}\|_\op \|E_s\|_{\circ}\rmd s.
\end{align}

Moreover, $\|\Lambda_{s,t}\|_\op$ admits the bound:
\begin{lemma}
    \label{lemma: Lambda bound}
    \begin{align}
        \|\Lambda_{s,t}\|_\op \le \exp\left(\int_{s}^t \|A_{s'}\|_\op \rmd s'\right).
    \end{align}
\end{lemma}
\begin{proof}
    Define $f_{\omega}(s, t) = \Lambda_{s, t} \omega$. We have
    \begin{align}
        \frac{\rmd}{\rmd t} \|f_{\omega}(s, t)\|_2 & = \frac{1}{\|f_{\omega}(s, t)\|_2} f_{\omega}(s, t)^{\top} \frac{\rmd}{\rmd t} f_{\omega}(s, t)  \\
                                                   & = \frac{1}{\|f_{\omega}(s, t)\|_2} \omega^{\top} \Lambda_{s, t}^{\top} A_t \Lambda_{s, t} \omega \\
                                                   & \le \|A_t\|_{\op} \|f_{\omega}(s, t)\|_2.
    \end{align}
    By Gronwall's theorem and $\|f_{\omega}(s, s)\|_2 = \|\omega\|_2$, we obtain
    \begin{align}
        \|f_{\omega}(s, t)\|_2 \le \|\omega \|_2 \exp(\int_{s}^t \|A_{s'}\|_\op \rmd s').
    \end{align}
\end{proof}

To bound $\exp\left(\int_{s}^t \|A_{s'}\|_\op \rmd s'\right)$, we use the following result from \citep{daniels2025contractivity}, included here for completeness.
\begin{theorem}[Restated; Theorem 6 in \citep{daniels2025contractivity}]
    \label{thm:Max Daniel}
    Suppose $\mu_0 \sim \Normal(0, \eye)$ and $\mu_1$ is a $\kappa$-log-concave distribution with $\kappa > 0$. Define
    \begin{align}
        I_t = \alpha_t X_0 + \beta_t X_1, \qquad X_0\sim\mu_0, \quad X_1\sim\mu_1,
    \end{align}
    and let $v_t(x)$ denote the corresponding flow field.
    Then,
    \begin{align}
        \nabla_x v_t(x) \preccurlyeq \frac{\kappa \alpha_{t} \dot{\alpha_{t}} + \beta_t \dot{\beta_t}}{\kappa \alpha_{t}^2 + \beta_{t}^{2}} \eye.
    \end{align}
\end{theorem}

With the result, we can bound $\exp\left(\int_{s}^t \|A_{s'}\|_\op \rmd s'\right)$ as follows.
\begin{lemma}
    \label{lemma:avg velocity bound}
    \begin{align}
        b^{\star}_t = \Exp[\dot I_t \mid I_t, \cond], \qquad \text{where } I_t = (1-t) a_0 + t a_1, \quad a_0 \sim \Normal(0,\eye), \quad a_1 \sim \rho_1,
    \end{align}
    where $\rho_1$ is $\kappa$-log-concave. Then, we have
    \begin{align}
        \int_{s}^t\|\nabla_x b^{\star}_{s'}(\gen_{s'} \mid \cond)\|_\op\rmd s' \le \log \sqrt{\frac{\kappa (1-t)^2 + t^{2}}{\kappa (1-s)^2 + s^{2}}}
    \end{align}
\end{lemma}
\begin{proof}
    By leveraging \Cref{thm:Max Daniel} for each condition $\cond$, we have
    \begin{align}
        \nabla_\gen b^{\star}_{s'}(\gen_{s'} \mid \cond) \preccurlyeq \frac{\kappa \alpha_{s'} \dot{\alpha_{s'}} + \beta_{s'} \dot{\beta_{s'}}}{\kappa \alpha_{s'}^2 + \beta_{s'}^{2}} \eye,
    \end{align}
    then we have
    \begin{align}
        \|\nabla_\gen b^{\star}_{s'}(\gen_{s'} \mid \cond)\|_\op \le \frac{\kappa \alpha_{s'} \dot{\alpha_{s'}} + \beta_{s'} \dot{\beta_{s'}}}{\kappa \alpha_{s'}^2 + \beta_{s'}^{2}}.
    \end{align}
    Integrating both sides, we obtain
    \begin{align}
        \int_{s}^t\|\nabla_\gen b^{\star}_{s'}(\gen_{s'} \mid \cond)\|_\op\rmd s' & \le \int_{s}^t \frac{\kappa \alpha_{s'} \dot{\alpha_{s'}} + \beta_{s'} \dot{\beta_s'}}{\kappa \alpha_{s'}^2 + \beta_{s'}^{2}} \rmd s' \\
                                                                                  & = \frac{1}{2} \log (\kappa \alpha_{s'}^2 + \beta_{s'}^{2})\Big|_{s}^t                                                                 \\
                                                                                  & = \log \sqrt{\frac{\kappa \alpha_{t}^2 + \beta_{t}^{2}}{\kappa \alpha_{s}^2 + \beta_{s}^{2}}}.
    \end{align}
    By substitute $\alpha_t = 1-t$ and $\beta_t = t$, we have
    \begin{align}
        \int_{s}^t\|\nabla_\gen b^{\star}_{s'}(\gen_{s'} \mid \cond)\|_\op\rmd s' \le \log \sqrt{\frac{\kappa (1-t)^2 + t^{2}}{\kappa (1-s)^2 + s^{2}}}.
    \end{align}
\end{proof}

With the preceding components in place, we now establish \Cref{thm:lipschitz}.

\subsection{Proof of Theorem \ref{thm:lipschitz}}
By combining \Cref{eq:M_t bound}, \Cref{lemma: Lambda bound}, and \Cref{lemma:avg velocity bound}, we have
\begin{align}
    \|\nabla_\cond\Phi_{0,t}(\gen_0, \cond)\|_{\circ} \le \int_0^t \sqrt{\frac{\kappa (1-t)^2 + t^{2}}{\kappa (1-s)^2 + s^{2}}} \cdot \|\nabla_\cond b^{\star}_{s}(\gen_s \mid \cond)\|_{\circ} \rmd s.
\end{align}
For $t = 1$, the function $s \mapsto \kappa(1-s)^2 + s^2$ attains its minimum at $s = \tfrac{\kappa}{\kappa+1}$. Applying Holder’s inequality then yields
\begin{align}
    \|\nabla_\cond\Phi_{0,1}(\gen_0, \cond)\|_{\circ} & \le \int_0^1 \sqrt{\frac{1}{\kappa (1-s)^2 + s^{2}}} \cdot \|\nabla_\cond b^{\star}_{s}(\gen_s \mid \cond)\|_{\circ} \rmd s                                 \\
                                                      & \le \max_{s\in [0, 1]} \left(\sqrt{\frac{1}{\kappa (1-s)^2 + s^{2}}}\right) \cdot \int_0^1 \|\nabla_\cond b^{\star}_{s}(\gen_s \mid \cond)\|_{\circ} \rmd s \\
                                                      & = \sqrt{1 + \kappa^{-1}} \int_0^1 \|\nabla_\cond b^{\star}_{s}(\gen_s \mid \cond)\|_{\circ} \rmd s.
\end{align}

\section{Regularization does not account for manifold adherence}
\label{app:reg}

In this section, we analyze a linear, population-level surrogate for \Minimaliterativepolicy to test whether \emph{implicit regularization}, instantiated via ridge regression and a two-step \Minimaliterativepolicy-like iteration, can explain the observed manifold adherence.
We mimic the two passes of \Minimaliterativepolicy with two ridge-regularized linear regressions: (i) a regression where the ridge penalty is applied to the observation-to-action map, and (ii) a regression where the ridge penalty is applied to the action-to-action map.
We then compose the two fitted maps to obtain the two-stage inference used by \Minimaliterativepolicy.
As shown below, it instead yields smooth spectral shrinkage and does not make the manifold absorbing.

Throughout we work in expectation (population covariances), so that conclusions reflect model structure rather than finite-sample effects. We assume independence of $o$, $z$, and the additive noise $\boldsymbol{\eta}$, and that the inverses we write exist (otherwise interpret as pseudoinverses on the relevant supports).

\noindent \textbf{Setup.} Observations $o\in\R^d$ and actions $a\in\R^d$ follow the linear model
\begin{align*}
    a = \Theta^* o + \boldsymbol{\eta}, \qquad \boldsymbol{\eta} \sim \Normal(0,\Sigma_\eta = \eta^2 \eye),
\end{align*}
with $\boldsymbol{\eta} \perp o$.
Let $z\sim \Normal(0,\Sigma_z)$ be an auxiliary signal, independent of $(o,a,\eta)$, and define $w := c_1 a + c_2 z$.
We consider linear predictors $\hat a = Bo + Cw$ with ridge regularization applied either to $B$ (\Cref{sec:ridge o to a}) or to $C$ (\Cref{sec:ridge a to a}). This reflects the use of both the observation $o$, and the action $a$, in prediction. 

\subsection{Ridge regression for observation-to-action mapping (penalty on \texorpdfstring{$B$}{B})}
\label{sec:ridge o to a}
We solve
\begin{align*}
    \min_{B, C} \Exp \big\|Bo + Cw - a\big\|^2 + \lambda \|B\|_F^2.
\end{align*}
Define
\begin{align*}
    X := \begin{bmatrix}
        o \\
        w
    \end{bmatrix}, \qquad
    \Psi := \begin{bmatrix}
        B &  C
    \end{bmatrix},
\end{align*}
and, at the population level, the second-moment blocks
\begin{align*}
    \Sigma_{11} &:= \Exp[o^{\otimes 2}] = \Sigma_o, \\
    \Sigma_{12} &:= \Exp[ow^{\top}] = \Exp[o(c_1 a + c_2 z)^{\top}] = c_1 \Sigma_o \Theta^{*\top}, \\
    \Sigma_{21} &:= \Sigma_{12}^{\top}, \\
    \Sigma_{22} &:= \Exp[w^{\otimes 2}] = c_1^2(\Theta^* \Sigma_o \Theta^{*\top} + \Sigma_\eta) + c_2^2 \Sigma_z. \\ 
    \Sigma_{a1} &:= \Exp[ao^{\top}] = \Theta^* \Sigma_o, \\
    \Sigma_{a2} &:= \Exp[aw^{\top}] = c_1(\Theta^*\Sigma_o \Theta^{*\top} + \Sigma_{\eta}),
\end{align*}
where, for any vector $x$, we write $x^{\otimes 2} := x x^\top$.

The objective can be written in trace form as
\begin{align*}
    \mathcal{L} &= \Exp[(a - \Psi X)^{\top}(a - \Psi X)] + \lambda \|B\|_F^2 \\
    &= \text{tr}(\Exp[(a - \Psi X)(a - \Psi X)^{\top}]) + \lambda \text{tr}(BB^{\top})
\end{align*}
Dropping terms that are constant in $(B, C)$, let
\begin{align*}
    \Sigma_{X} := \Exp[XX^{\top}] = \begin{bmatrix}
        \Sigma_{11} & \Sigma_{12} \\
        \Sigma_{21} & \Sigma_{22}
    \end{bmatrix},
    \qquad \Sigma_{aX} := \Exp[aX^{\top}] = \begin{bmatrix}
        \Sigma_{a1} & \Sigma_{a2}
    \end{bmatrix}.
\end{align*}
Then
\begin{align*}
    \mathcal{L} = -2 \text{ tr}(\Psi\Sigma_{aX}^{\top}) + \text{tr}(\Psi \Sigma_{X} \Psi^{\top}) + \lambda \text{ tr}(BB^{\top}).
\end{align*}
Differentiating gives
\begin{align*}
    \nabla_B \mathcal{L} &= -2 \Sigma_{a1} + 2B \Sigma_{11} + 2C \Sigma_{21} + \lambda 2B, \\
    \nabla_C \mathcal{L} &= -2 \Sigma_{a2} + 2B \Sigma_{12} + 2C \Sigma_{22}.
\end{align*}
Setting the gradients to zero yields the normal equations:
\begin{align*}
    B(\Sigma_{11} + \lambda \eye) &+ C \Sigma_{21} = \Sigma_{a1} \\
    B\Sigma_{12}  &+ C \Sigma_{22} = \Sigma_{a2}.
\end{align*}

Solving the linear system (e.g., by block elimination) yields
\begin{align*}
    B &= \underbrace{(\Sigma_{a1} \Sigma_{21}^{-1}\Sigma_{22} - \Sigma_{a2})}_{\text{(i)}} \underbrace{[(\Sigma_{11} + \lambda \eye)\Sigma_{21}^{-1}\Sigma_{22} - \Sigma_{12}]^{-1}}_{\text{(ii)}}, \\
    C &= [\Sigma_{a1} - B(\Sigma_{11} + \lambda \eye)] \Sigma_{21}^{-1}.
\end{align*}
Using $\Sigma_{21} = c_1 \Theta^* \Sigma_o$, we have $\Sigma_{21}^{-1} = \frac{1}{c_1}(\Theta^{*}\Sigma_o)^{-1}$. For (i),
\begin{align*}
    \text{(i)} &= \Sigma_{a1} \Sigma_{21}^{-1}\Sigma_{22} - \Sigma_{a2} \\
    &= \Theta^* \Sigma_o \frac{1}{c_1}(\Theta^{*}\Sigma_o)^{-1} [c_1^2(\Theta^* \Sigma_o \Theta^{*\top} + \Sigma_\eta) + c_2^2 \Sigma_z] - c_1(\Theta^*\Sigma_o \Theta^{*\top} + \Sigma_{\eta}) \\
    &= c_1 (\Theta^* \Sigma_o \Theta^{*\top} + \Sigma_\eta) + \frac{c_2^2}{c_1} \Sigma_z - c_1(\Theta^*\Sigma_o \Theta^{*\top} + \Sigma_{\eta}) \\
    &=  \frac{c_2^2}{c_1} \Sigma_z.
\end{align*}
For (ii),
\begin{align*}
    \text{(ii)} &= [(\Sigma_{11} + \lambda \eye)\Sigma_{21}^{-1}\Sigma_{22} - \Sigma_{12}]^{-1} \\
    &= [(\Sigma_{o} + \lambda \eye)\frac{1}{c_1}(\Theta^{*}\Sigma_o)^{-1}[c_1^2(\Theta^* \Sigma_o \Theta^{*\top} + \Sigma_\eta) + c_2^2 \Sigma_z] - c_1 \Sigma_o (\Theta^*)^{\top}]^{-1} \\
    &= [(\Sigma_{o} + \lambda \eye)[c_1 \Theta^{*\top} + c_1(\Theta^{*}\Sigma_o)^{-1}\Sigma_\eta + \frac{c_2^2}{c_1} (\Theta^{*}\Sigma_o)^{-1}\Sigma_z] - c_1 \Sigma_o (\Theta^*)^{\top}]^{-1} \\
    &= [\lambda c_1 (\Theta^*)^{\top} + (\Sigma_{o} + \lambda \eye)(c_1(\Theta^{*}\Sigma_o)^{-1}\Sigma_\eta + \frac{c_2^2}{c_1} (\Theta^{*}\Sigma_o)^{-1}\Sigma_z)]^{-1}
\end{align*}

\noindent \iclrpar{Isotropic specialization.} Take $\Sigma_o=\eye$, $\Sigma_\eta=\eta^2\eye$, $\Sigma_z=\eye$, and $\Theta^*=\text{diag}(s_i)$. Then
\begin{align*}
    \text{(i)} &= \frac{c_2^2}{c_1} \eye, \\
    \text{(ii)} &= [\lambda c_1 (\Theta^*)^{\top} + (1 + \lambda) c_1 \eta^2(\Theta^{*})^{-1} + (1 + \lambda) \frac{c_2^2}{c_1} (\Theta^{*})^{-1}]^{-1} \\
    &= \text{diag}\left([\lambda c_1 s_i + \frac{(1+\lambda)(c_1\eta^2 + \frac{c_2^2}{c_1})}{s_i}]^{-1}\right) \\
    &= \text{diag}\left(\frac{c_1 s_i}{\lambda c_1^2 s_i^2 + (1 + \lambda)(c_1^2 \eta^2 + c_2^2)}\right)
\end{align*}
Hence, we obtain
\begin{align*}
    B = \text{diag}\left(\frac{c_2^2 s_i}{\lambda c_1^2 s_i^2 + (1 + \lambda)(c_1^2 \eta^2 + c_2^2)}\right), \quad C = \text{diag}\left(\frac{1}{c_1}\left(1 - \frac{(1+\lambda)c_2^2}{\lambda c_1^2 s_i^2 + (1 + \lambda)(c_1^2 \eta^2 + c_2^2)}\right)\right).
\end{align*}
Define the shrinkage factor for the $i$-th singular direction by $\rho_i := B_{ii}/s_i$. Then,
\begin{align}
\label{eq:shrink factor}
    \rho_i = \frac{c_2^2}{\lambda c_1^2 s_i^2 + (1 + \lambda)(c_1^2 \eta^2 + c_2^2)}.
\end{align}

\noindent\iclrpar{Key implications:}
\begin{itemize}[leftmargin=*]
    \item \textbf{Ridge on $B$ ($\lambda > 0$) makes shrinkage $s_i$-dependent.} From \eqref{eq:shrink factor}, the factor decreases with $s_i$ (because the denominator has $\lambda c_1^2 s_i^2$). So \textbf{larger} singular directions are shrunk \textbf{more} when $\lambda > 0$.
    \item \textbf{If no ridge ($\lambda = 0$).} $B = \text{diag}(\frac{c_2^2 s_i}{c_1^2 \eta^2 + c_2^2})$: shrinkage is constant across $i$.
    \item \textbf{If no ridge and no noise ($\lambda = 0, \eta = 0$).} $B = \Theta^*$---no shrinkage (recovers the standard solution).
    \item \textbf{If no auxiliary $z$-signal ($c_2 = 0$).} $B = 0$.
\end{itemize}

\subsection{Ridge regression for action-to-action mapping (penalty on \texorpdfstring{$C$}{C})}
\label{sec:ridge a to a}

We now solve
\begin{align*}
\min_{B,C} ; \Exp\big\| Bo + Cw - a \big\|^2 + \lambda\|C\|_F^2.
\end{align*}
The normal equations become
\begin{align*}
    B\Sigma_{11} + C \Sigma_{21} &= \Sigma_{a1} \\
    B\Sigma_{12}  + C (\Sigma_{22} + \lambda \eye) &= \Sigma_{a2}.
\end{align*}

Solving gives the closed-form estimators:
\begin{align*}
    B &= (\Sigma_{a1} - C\Sigma_{21}) \Sigma_{11}^{-1}, \\
    C &= \underbrace{(\Sigma_{a2} - \Sigma_{a1}\Sigma_{11}^{-1}\Sigma_{12})}_{\text{(i)}}\underbrace{(\Sigma_{22} + \lambda \eye - \Sigma_{21}\Sigma_{11}^{-1}\Sigma_{12})^{-1}}_{\text{(ii)}}.
\end{align*}

In the isotropic specialization $\Sigma_o=\eye$, $\Sigma\eta=\eta^2\eye$, $\Sigma_z=\eye$, $\Theta^*=\text{diag}(s_i)$, one obtains
\begin{align*}
    B = \text{diag}\left(\frac{(c_2^2 + \lambda)s_i}{c_1^2 \eta^2 + c_2^2 + \lambda}\right), \qquad C = \frac{c_1\eta^2}{c_1^2 \eta^2 + c_2^2 + \lambda} \eye.
\end{align*}

\subsection{Two-pass linear surrogate of \Minimaliterativepolicy}
\label{sec:mip-composition}

Let $(B_1,C_1)$ denote the solution of \Cref{sec:ridge o to a} with ridge $\lambda_1$ on $B$, and $(B_2,C_2)$ denote the solution of \S\ref{sec:ridge a to a} with ridge $\lambda_2$ on $C$.
To mimic the \Minimaliterativepolicy two-pass inference rule in \eqref{eq:pimip_inference}, we consider
\begin{align*}
\hat a_0 \gets B_1 o,\qquad \hat a \gets B_2 o + c_1 C_2\hat a_0.
\end{align*}
We obtain
\begin{align*}
    \hat a \gets (\underbrace{B_2 + c_1 C_2B_1}_{=:\hat{\Phi}})o.
\end{align*}
Note that $c_1$ serves as the analogue of $t_{\star}$ from the main text. From \Cref{sec:ridge o to a,sec:ridge a to a} we then obtain
\begin{align}
    \hat{\Phi} &= (B_2 + c_1 C_2B_1) \nonumber \\
    &= \text{diag}\left(\frac{(c_2^2 + \lambda_2)s_i}{c_1^2 \eta^2 + c_2^2 + \lambda_2} + c_1 \frac{c_1\eta^2}{c_1^2 \eta^2 + c_2^2 + \lambda_2} \frac{c_2^2 s_i}{\lambda_1 c_1^2 s_i^2 + (1 + \lambda_1)(c_1^2 \eta^2 + c_2^2)}\right) \nonumber \\
    &= \text{diag}\left(\frac{s_i}{c_1^2 \eta^2 + c_2^2 + \lambda_2} \left[c_2^2 + \lambda_2 + \frac{c_1^2\eta^2c_2^2}{\lambda_1 c_1^2 s_i^2 + (1 + \lambda_1)(c_1^2 \eta^2 + c_2^2)}\right]\right). \label{eq:singular value of Phi}
\end{align}
Moreover, the shrink factor will be 
\begin{align}
\label{eq:phi shrink}
    \frac{\hat\Phi_{ii}}{s_i} = \frac{1}{c_1^2 \eta^2 + c_2^2 + \lambda_2} \left[c_2^2 + \lambda_2 + \frac{c_1^2\eta^2c_2^2}{\lambda_1 c_1^2 s_i^2 + (1 + \lambda_1)(c_1^2 \eta^2 + c_2^2)}\right].
\end{align}

\noindent \iclrpar{Key implications:}
\begin{itemize}[leftmargin=*]
    \item \textbf{If no ridge on B ($\lambda_1 = 0$).} No $s_i$-dependent shrinkage.
    \item \textbf{If no noise ($\eta = 0$).} Just same as the regular case: $\hat{\Phi} = \Theta^*$.
    \item \textbf{Signal-to-noise effect.} The quantity in \eqref{eq:singular value of Phi} rises with $s_i$ and falls with $\eta$, mildly favoring signal over noise by damping noisy directions.
\end{itemize}

\noindent \iclrpar{Why this composition is a plausible proxy.} The first stage applies a ridge penalty to the observation-to-action parameters $B$ and predicts an interpolant action from observations alone, as in the $t = 0, z = 0$ pass of \Minimaliterativepolicy.
We use ridge here as a canonical proxy for implicit regularization in the linear setting.
The first stage applies a ridge penalty to the action-to-action parameters $C$. It takes the interpolant action input (near $t = 1$) together with $o$ and produces the final output. Composing the two yields the operator $\hat\Phi$ in \eqref{eq:singular value of Phi}, which is the linear analogue of the two-pass prediction of \eqref{eq:pimip_inference}.

\noindent\iclrpar{Why shrinkage does not yield manifold adherence.}
The operator in \eqref{eq:singular value of Phi} acts as a spectral shrinker: because the factors in \eqref{eq:phi shrink} decrease with $s_i$ (for $\lambda_1 > 0$), it attenuates the dominant directions more than the weak ones---contrary to a projection onto a manifold, which would preserve principal directions and damp small and noisy modes.
Since these factors lie in $(0, 1]$ and vary smoothly with $s_i$, $\eta$, and $\lambda_1,\lambda_2$, the map lacks any projection-like behavior: once a point is off-manifold, it is neither returned to nor retained on any low-dimensional subspace.
Thus, implicit regularization alone, even with the two-pass composition of \Minimaliterativepolicy, cannot account for the observed manifold adherence.

\section{Toy experiments: Testing the function approximation capabilities of regression and flow models}
\label{app:toy}

\subsection{Overview}

This appendix summarizes an empirical comparison of training paradigms (regression, flow matching, straight flow, MIP) for function approximation with geometric constraints across three tasks: scalar reconstruction, high-dimensional projection with subspace constraints, and Lie algebra rotations. Experiments operate in low-data regimes (50 training samples) using concatenation and FiLM architectures, with results averaged across multiple random seeds.

\subsection{Evaluation Metrics}

\textbf{Reconstruction:} L1 and L2 errors measure point-wise approximation quality between predictions $\hat{f}(c)$ and targets $f(c)$.

\textbf{Projection:} Three metrics assess geometric constraints in piecewise-constant projection structure: \textit{subspace diagonal} quantifies predictions outside correct subspace $P_i$ for interval $i$, \textit{off-diagonal} tests cross-interval generalization with mismatched projections, and \textit{boundary} measures smoothness at interval transitions using combined adjacent subspace projections. All metrics use normalized form $\|(I - P)\hat{f}\| / \|\hat{f}\|$.

\textbf{Lie Algebra:} \textit{Cosine similarity} measures angular alignment between predicted and true rotation directions. \textit{Projection metric} quantifies normalized perpendicular error relative to the rotation axis span.

\subsection{Key Findings}

\textbf{Task-Dependent Performance:} Regression-based approaches achieve lowest L2 reconstruction error (0.003197 ± 0.000525 with L2 loss and FiLM), consistently outperforming flow-based methods on point-wise approximation tasks.

\textbf{Flow Methods Excel at Projections:} Flow-based training demonstrates superior geometric constraint satisfaction. Straight flow (flow matching without time conditioning) achieves best boundary projection (0.009769 ± 0.001630) and Lie algebra projection metrics (0.063612 ± 0.000952), indicating beneficial geometric biases from learning probability transport.

\textbf{MIP Competitive Performance:} MIP combines direct regression with denoising regularization, achieving near-optimal reconstruction while maintaining reasonable geometric constraint satisfaction across tasks.

\subsection{Training Loss Considerations}

Results focus on L2-trained models, providing mathematically grounded objectives for both regression and flow paradigms. While alternative loss functions were evaluated empirically, flow-based L1 training lacks principled derivation as conditional flow matching is naturally defined for squared error.

\subsection{Architectural Observations}

Both concatenation and FiLM architectures demonstrated competitive performance with no consistent dominance. FiLM showed marginal advantages on certain geometric metrics for flow-based methods, suggesting affine feature modulation may better capture conditional dependencies in probability transport.

\subsection{Implications for Method Selection}

\begin{itemize}[itemsep=0pt,parsep=2pt,topsep=2pt]
    \item Tasks prioritizing point-wise reconstruction: regression-based training with L2 loss offers superior accuracy and computational efficiency.
    \item Tasks requiring geometric constraint satisfaction: flow-based training provides significant advantages despite increased evaluation cost.
    \item Straight flow's success suggests time conditioning may be unnecessary, enabling simpler models with competitive performance.
\end{itemize}

\subsection{Experimental Details}

Study encompasses 540 runs: 5 modes (regression, flow, straight flow, MIP, MIP one-step) × 2 losses × 2 architectures × 3 tasks × 3 seeds. Configuration: 256 hidden dimensions, 3 layers, ReLU, batch size 32, 50k epochs, Adam with lr=0.001. Evaluation: 100k test samples; flow methods use Euler integration with 9 ODE steps.

\vspace{0.3em}
\noindent\textbf{Full Report:} \url{https://github.com/simchowitzlabpublic/much-ado-about-noising/blob/main/toyexp/toyexp/report/report.pdf} \quad \textbf{Code:} \url{https://github.com/simchowitzlabpublic/much-ado-about-noising/tree/main/toyexp/toyexp}

\section{Appendix for Section \ref{sec:prelim}}
\label{app:prelim}

\subsection{Markov Decision Processes Configuration}
\label{app:MDP}
We consider a Markov Decision Process $\mathcal{M} = (\mathcal{S}, \mathcal{A}, R, P, P_0)$\footnote{For simplicity, we consider the MDP case in this context by identifying the state with the observation defined in \ref{sec:prelim}. More generally, one may consider a Partially Observable Markov Decision Process (POMDP), where the agent receives observation $o$ emitted by an underlying latent state $s$.} with the state space $\mathcal{S}$, the action space $\mathcal{A}$, the reward $R(s, a)$\footnote{For ease of exposition, we use the same notation for rewards defined on random variables and their distributions.} obtained by taking action $a$ in state $s$, the transition dynamics $P:\mathcal{S} \times \mathcal{A} \to \Delta(\mathcal{S})$, and the initial-state distribution $P_0 \in \Delta(\mathcal{S})$ . To formulate the success rate (i.e., performance) in this setting, we define the reward function as:
\begin{align}
    R(s, a) = \begin{cases}
                  1, & \text{if the task is successful under $(s,a)$}, \\
                  0, & \text{otherwise}.
              \end{cases}
\end{align}
Under this definition of rewards, the expected return of a policy $\pi$ is $J(\pi) = \mathbb{E}[\sum_t R(s_t, a_t)]$, which reduces to $\Pr[\text{success under }\pi]$. Hence, $J(\pi)$ exactly equals the success rate of policy $\pi$.

\subsection{Integrated Flow Prediction}
\label{app:euler_integration}

For completeness, we provide the flow ODE as
\begin{align}
    \iftoggle{iclr}{\textstyle}{}   \ddt a_t = b_t(a_t \mid o) \qquad \text{starting from} \qquad a_0=z.
\end{align}
The associated integrated flow prediction is given by
\begin{align}
    \Phi_{\theta}(z \mid \cond) = z + \int_{0}^{1} b_t(a_t \mid o) \rmd t.
\end{align}
In practice, to approximate the ODE solution for sampling, we employ the following discretized Euler integration.
\begin{definition}[Discretized Euler Integration]
    We discretize the time interval $[0, 1]$ to $N$ steps with step size $h = 1 / N$. The iterates are then updated according to
    \begin{align}
        a_{k+1} = a_{k} + h\, b_{hk}(a_{k} \mid \cond), \quad k = 0, 1, \dots, N-1.
    \end{align}
    The final iterate $a_N$ serves as the Euler approximation $\Phieul(z \mid o)$. We also refer to $N$ as the \emph{Number of Function Evaluations (NFEs)}.
\end{definition}

\end{document}